\documentclass{article} 
\usepackage{iclr2027_conference,times}

\usepackage{amsmath,amsfonts,bm}

\def\eqref#1{equation~\ref{#1}}

\def\1{\bm{1}}

\DeclareMathAlphabet{\mathsfit}{\encodingdefault}{\sfdefault}{m}{sl}
\SetMathAlphabet{\mathsfit}{bold}{\encodingdefault}{\sfdefault}{bx}{n}

\usepackage{hyperref}
\usepackage{url}
\usepackage{graphicx}
\usepackage{placeins}
\usepackage{amssymb}
\usepackage{amsthm}
\usepackage{booktabs}
\usepackage{multirow,colortbl,pifont}

\title{Identify then Realize: Contrastive Learning of Latent Port-Hamiltonian Dynamics from Partial Observations}

\author{%
  \normalfont
  \begin{tabular}[t]{@{}l@{\hspace{2em}}l@{}}
    \textbf{Peilun Li} & \textbf{Kaiyuan Tan} \\
    Department of Computer Science & Department of Computer Science \\
    Vanderbilt University & Vanderbilt University \\
    Nashville, TN 37235 & Nashville, TN 37235 \\
    \texttt{peilun.li@vanderbilt.edu} &
    \texttt{kaiyuan.tan@vanderbilt.edu} \\[1.5em]
    \textbf{Daniel Moyer} & \textbf{Thomas Beckers} \\
    Department of Computer Science & Department of Computer Science \\
    Vanderbilt University & Vanderbilt University \\
    Nashville, TN 37235 & Nashville, TN 37235 \\
    \texttt{daniel.moyer@vanderbilt.edu} &
    \texttt{thomas.beckers@vanderbilt.edu}
  \end{tabular}%
}

\newtheorem{theorem}{Theorem}
\newtheorem{proposition}[theorem]{Proposition}
\newtheorem{corollary}{Corollary}

\newtheorem{lemma}{lemma}

\iclrfinalcopy 
\begin{document}

\maketitle

\begin{abstract}
Identifying latent state representations and dynamics is essential when direct modeling in observation space is infeasible, particularly under partial and high-dimensional observations. In such settings, representation learning and physics-aware modeling are inherently coupled. We propose CIPHER, a two-stage identify-then-realize framework for learning latent port-Hamiltonian models of conservative and dissipative systems. First, a contrastive teacher jointly learns an encoder and a neural ODE to obtain predictive state representations from observation histories. Second, a student learns an invertible nonlinear coordinate transformation and port-Hamiltonian dynamics by matching encoded observed futures. We establish sufficient conditions for recovering the physical state up to a diffeomorphism and realizing its dynamics in port-Hamiltonian form. Across ten clean and noisy datasets, CIPHER is competitive with or outperforms state-of-the-art baselines by improving long-horizon forecast on complex systems. Further studies show robustness to over-specified latent dimensions and additional predictive gains from pH parameterization.
\end{abstract}

\section{Introduction}
\label{sec:introduction why}

Data-driven system identification aims to learn dynamical models from observed trajectories when the governing equations or their parameters are unknown. Machine learning approaches have broadened its applicability by enabling flexible models of continuous-time dynamics \citep{rubanova2019latentode,rettberg2025data}. However, fitting observed trajectories alone does not ensure reliable prediction beyond the training regime or over long horizons. Physics-informed ML addresses this challenge by incorporating prior knowledge, such as conservation laws and structure, into model design or training \citep{karniadakis2021piml,raissi2019physics}. By restricting the class of admissible dynamics, these approaches provide an inductive bias that can improve learning and generalization.

 However, real-world physical systems are typically observed through incomplete, noisy, or high-dimensional measurements. These measurement data may reveal only a subset of the variables needed to model the system
dynamics. This is a common problem in system identification and scientific modeling, where one can only observe images or positions of system, without explicit velocity quantities \citep{toth2020hgn, ma2022gdoom, bhardwaj2026phast}. Hence, complete state representation on which physical structure is naturally expressed may not be available for modeling. This creates a core tension: the problem is no longer only \emph{how} to learn physics-informed dynamics, but also \emph{what} latent state representations on which such dynamics can be defined.

Contrastive learning (CL) is particularly well suited in this setting as it learns representations that capture predictive state information through InfoNCE loss~\citep{oord2018cpc, laiz2025dyncl, chen2020simpleframeworkcontrastivelearning}. Recent CL results show that latent dynamics can be recovered up to an affine transformation from nonlinear observation models \citep{laiz2025dyncl}. However, theoretical understanding of modeling dynamics under partial observation remains limited.

In this paper, we consider the typical setting where the governing equations are unknown and observations are partial and/or high-dimensional, so that the underlying state is not directly available in a physically complete representation. The goal is to identify and model underlying dynamics in an accurate and structurally informed way. We adopt the port-Hamiltonian (pH) formalism, an universal energy-based modeling framework, to enforce physical structure without requiring detailed prior knowledge and combine it with a CL teacher–student architecture to learn latent dynamics.

\textbf{Related work and gap:}
Three neighboring lines of work are most relevant. First, latent Neural ODE and methods like Hamiltonian Generative Network (HGN) learn continuous-time dynamics from high-dimensional data \citep{chen2019neuralordinarydifferentialequations,rubanova2019latentode,toth2020hgn}. However, these methods do not impose port-Hamiltonian prior, and HGN in particular assumes dissipation-free dynamics, which is impractical to model most real-world systems. Second, contrastive predictive coding and dynamics CL methods show that latent dynamical representations can be identified from nonlinear observation models \citep{oord2018cpc,laiz2025dyncl, chen2020simpleframeworkcontrastivelearning}. In particular, Dynamics Contrastive Learning (DCL) provides strong affine-identifiability, but it does not impose any physical prior or continuous modeling in its latent space \citep{laiz2025dyncl}. Third, port-Hamiltonian Neural Network (pHNN) and latent pH methods exploit port-Hamiltonian structure\citep{desai2021phnn,rettberg2025data,bhardwaj2026phast}, but typically assume access to physically meaningful coordinates, \emph{full state} trajectories, or position-only measurements rather than visual observations, where full state information cannot be computed explicitly. Hence, there is a clear gap in the literature, which motivates our work that addresses physics-informed learning for system identification where not only governing equations or system parameters are unknown, but also observations are partial, high-dimensional, and do not explicitly contain the entire state. 

\textbf{Contributions.}
We introduce CIPHER (\textbf{C}ontrastive \textbf{I}dentification with \textbf{PH} \textbf{E}mbedded \textbf{R}ealization), a two-stage framework that separates predictive state identification from learning latent dynamics with physical structure, as depicted in Fig.~\ref{fig:method_schematic}. Our contributions are threefold:

\begin{figure*}[t]
    \centering
    \includegraphics[width=0.98\textwidth]{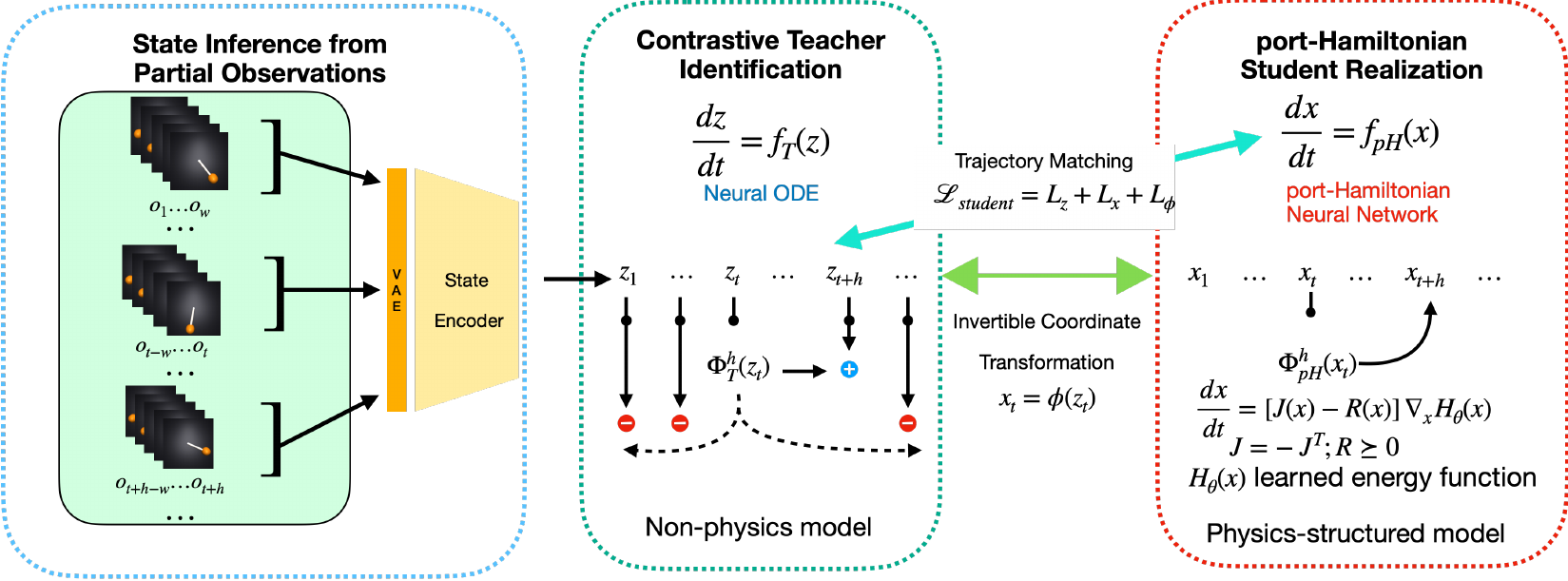}
    \caption{\textbf{CIPHER: identification followed by structured realization.}
    The \textbf{first stage} jointly trains a causal windowed encoder and a contrastive neural ODE teacher \(f_T\) to learn predictive state representations \(z_t\) from partial observations \(\{o_t\}\), which can be vectors or images. The \textbf{second stage} assumes fixed state encodings and trains port-Hamiltonian student dynamics \(f_{pH}\) in coordinates \(x_t=\phi(z_t)\) through a learned invertible coordinate transformation \(\phi\). The separation is deliberate: first stage learns representations \(z_t\) retaining full predictive state quantities, while second stage establishes physics-structured model on transformed coordinates \(x_t\) for predictive benefits.}
    \label{fig:method_schematic}
\end{figure*}

\noindent \textbf{(1) Method.} We propose a two-stage framework combining contrastive continuous-time identification with port-Hamiltonian realization. A windowed encoder and neural ODE jointly learn predictive representations, followed by a pH student network learning encoded observed futures through a learned invertible nonlinear coordinate transformation .\\
\noindent \textbf{(2) Theory.} We establish sufficient conditions for full-state recovery up to a diffeomorphism and show how such representations support a port-Hamiltonian realization when the underlying dynamics admit one.\\
\noindent \textbf{(3) Empirical evaluation.} Across clean and noisy datasets, CIPHER outperforms or stays competitive with state-of-the-art baselines, and improves long-horizon forecasting on complex systems. Comprehensive studies are consistent with our theory and evaluate design choices. 

\section{Method}
\label{sec:background}

\subsection{Preliminaries}
\textbf{Problem setup:}
Let \(s_t\in\mathbb R^{d_s}\) denote the physical state of a continuous-time dynamical system
\begin{equation}
\dot s_t=f_\star(s_t),\qquad o_t=g(s_t),
\end{equation}
where \(f_\star:\mathbb R^{d_s}\to\mathbb R^{d_s}\) is the physical vector field and \(g:\mathbb R^{d_s}\to\mathbb R^{d_g}\) is a possibly nonlinear and non-injective observation map.
Here, \(o_t\) is either a generalized position vector (e.g. angle for a pendulum) or a high-dimensional visual observation (e.g. image of the pendulum).
The goal is to infer latent states and continuous-time dynamics with pH structure from trajectories of partial observations \(\{o_t\}_{t=1}^N\) at fixed sampling rate, without knowledge of \(f_\star\) or \(g\), or physical-state labels.\newline
\textbf{Notation:} We distinguish three coordinate systems: the hidden physical state \(s_t\), the teacher latent state \(z_t\in\mathbb R^{d_z}\), and the student state \(x_t\in\mathbb R^d\).
The teacher latent state evolves under a neural ODE model, while the student state follows a pH dynamics model.

\textbf{Contrastive identification for dynamics:}
Contrastive Predictive Coding \citep{oord2018cpc} is a learning framework that produces context variables that are functions of current states and maximally informative about future states. This is estimated by contrasting the true future against negative samples.
This shares an intuitive connection to modeling dynamics. An encoder maps an observation history to a representation, \(z_t=E_\psi(o_{t-w:t})\), and a predictor produces a future latent estimate \(\hat z_{t+h}\).
The InfoNCE objective scores this prediction against the encoded observed future and negative samples; see Appendix~\ref{app:bg_infonce} for details.
For continuous-time latent dynamics, predictions can be obtained by integrating a neural ODE from the encoded state over the prediction horizon \citep{chen2019neuralordinarydifferentialequations}.
Contrastive objectives also admit identifiability results: under assumptions, DCL recovers latent states and dynamics up to an affine coordinate transformation \citep{laiz2025dyncl}. We refer to this idea and formulate identifiability of latent dynamics up to diffeomorphisms.

\textbf{Port-Hamiltonian systems:}
Port-Hamiltonian (pH) systems provide an energy-based description of conservative and dissipative dynamics \citep{van2014port}.
An autonomous pH system without external inputs, with state \(x\in\mathbb R^d\), has the form
\begin{equation}
\dot x=(J(x)-R(x))\nabla H(x),
\label{eq:bg_phs}
\end{equation}
where \(H:\mathbb R^d\to\mathbb R\) is the Hamiltonian, \(J(x)=-J(x)^\top\) is the skew-symmetric interconnection matrix, and \(R(x)=R(x)^\top\succeq0\) is the positive semidefinite (PSD) dissipation matrix. In practice, \(J\text{ and } R\) can be constant or state-dependent.  Constant matrices offer analytic simplicity, while state-dependent formulations capture configuration-dependent couplings and nonlinear dissipation.
Along a continuous-time solution, the energy follows
\begin{equation}
\frac{d}{dt}H(x(t))
=\nabla H(x)^\top(J(x)-R(x))\nabla H(x)
=-\nabla H(x)^\top R(x)\nabla H(x)
\le0.
\label{eq:energy_decay}
\end{equation}
Thus, \(H\) is conserved when \(R=0\), while dissipative pH systems satisfy \(R \succeq0\). This explicit separation of energy-conserving and energy-dissipating dynamics makes the pH formulation particularly powerful, and thus have been applied to a variety of real-world modeling problems \citep{van2014port}. 
 Port-Hamiltonian realizations are not unique and can be defined up to diffeomorphisms. Our view builds on this key fact that pH systems are closed under diffeomorphisms.

\subsection{Proposed Method}
\label{sec:method}

Our method treats latent pH learning from partial observations as \emph{identification followed by realization}. A contrastive teacher jointly learns a state encoder and a neural ODE, shaping the representation through temporal prediction on a unconstained manifold. With the learned latent states, a student transforms the representations into coordinates in which structured pH dynamics are defined and learned from the observed trajectories. We refer to this stage as \emph{realization}: an invertible coordinate transformation learned jointly with pH dynamics on the transformed state manifold.

\textbf{State representation from partial observations.}
Let partial observation at time \(t\) be defined as $o_t$, which can be either a numerical representation of the system $y_t$ (i.e., a vector of positions), or a raw image frame $I_t$ from which a visual embedding $C_\omega$\footnote{In principle, the visual encoder \(C_\omega\) can be any architecture. In our setting, the visual encoder is a variational autoencoder pretrained on training-split frames \citep{kingma2022autoencodingvariationalbayes}.} can be computed. The encoder \(E_\psi\) is a windowed gated recurrent unit (GRU) network \citep{chung2014empiricalevaluationgatedrecurrent} that maps a causal observation history to a latent state,
\(z_t=E_\psi(o_{t-w:t})\).
By construction, the window supplies temporal context for inferring components absent from a single observation, such as momentum. Section~\ref{sec:theory} provides theoretical support for this design and states the observability and information-retention conditions under which the encoding suffices for full-state recovery.

\textbf{Teacher: contrastive identification.}
The teacher models latent evolution by a neural ODE, \(\dot z=f_T(z)\) \citep{chen2019neuralordinarydifferentialequations}. Let \(\Phi_T^h\) denote its numerical rollout over \(h\) sampling intervals. From an encoded anchor \(z_t\), it predicts \(\hat z_{t+h}=\Phi_T^h(z_t)\). The positive target \(\bar z_{t+h}\) is obtained by encoding the observed future window with an exponential moving average (EMA) of the encoder parameters \citep{tarvainen2018meanteachersbetterrole}. 
The encoder and vector field are trained jointly with
\begin{equation}
\mathcal L_T
=
\mathbb E_{h\in\mathcal H_T,\,t}
\big[
\mathcal L_{\mathrm{NCE}}^{(h)}
(\hat z_{t+h},\bar z_{t+h},\{\bar z_j^-\})
\big],
\label{eq:teacher_loss}
\end{equation}
where \(\mathcal H_T\) is the set of prediction horizons and \(\{\bar z_j^-\}\) are negative target encodings.
We use the negative-MSE based score
\(\ell(u,v)=-\|u-v\|_2^2/(\tau\,\operatorname{sg}(\sigma^2))\),
where \(\tau>0\) is the temperature, \(\operatorname{sg}\) denotes stop-gradient, and \(\sigma^2>0\) is the mean per-coordinate variance of batch target latents. A common rescaling of queries and keys, with variance recomputed from the rescaled batch, leaves the scores unchanged.
Negatives are drawn from the batch and a memory bank of past EMA encodings. 

For positive and negative set \(\{z_j^-\}\), \citet{shamba2025contrasttimelearningtime}
treats temporal neighbors to \(\hat{z}_{t+h}\) as positive pairs, but this is not accurate for dynamics since it treats multiple states as predictive targets. For a positive at time \(t+h\), a same-trajectory candidate at time \(j\) is retained only if \(l<|j-(t+h)|\le r\), with offsets measured in sample indices. The lower cutoff \(l\) excludes immediate temporal neighbors from selection; the upper cutoff \(r\) restricts distant candidates that may revisit similar states.
The InfoNCE definition is given in Appendix~\ref{app:bg_infonce}. Sampling and normalization details appear in Appendices~\ref{app:contrastive_collisions} and~\ref{app:impl}. The resulting teacher is the \emph{identification model}: recovering state representations and a predictive latent flow from partial observations, without any physical prior.

\textbf{Student: port-Hamiltonian realization.}
An invertible map \(\phi:\mathbb R^{d_z}\to\mathbb R^{d}\), (nonlinear) transformation parameterized by affine coupling layers, maps encoder states to student coordinates \(x_t=\phi(z_t)\) \citep{dinh2017densityestimationusingreal}.
In these coordinates, the student evolves under
\begin{equation}
\dot x=f_{\mathrm{pH}}(x)
=(J(x)-R(x))\nabla_x H_\theta(x),
\label{eq:student_phs}
\end{equation}
where \(J(x)^\top=-J(x)\), \(R(x)=R(x)^\top\succeq0\), and \(H_\theta\) is a learned scalar energy function \citep{desai2021phnn}.
Parameterizations and regularization of \(H_\theta,J,R\), including constant and state-dependent matrix choices, are given in Appendix~\ref{app:impl}.
This structure enforces the continuous-time energy balance in \eqref{eq:energy_decay} for \(H_\theta\).
The resulting student is a \emph{physics-structured model} that learns pH dynamics on a latent state manifold obtained through the invertible coordinate transformation.

\subsection{Student realization by trajectory matching}
\label{subsec:trajectory_matching}

After the training of the teacher, we freeze \(E_\psi\) and cache its encodings of the observation trajectories.
Writing \(\Phi_{pH}^h\) for the numerical student rollout over \(h\) sampling intervals, a forecast initialized at \(z_t\) is
\(\hat x_{t+h}=\Phi_{pH}^h(\phi(z_t))\),
with prediction
\(\hat z^{\mathrm{pH}}_{t+h}=\phi^{-1}(\hat x_{t+h})\)
in encoder coordinates.
Its target is the encoded observed future
\(z_{t+h}=E_\psi(o_{t+h-w:t+h})\).

\textbf{Observed-future matching.}
We fit trajectories in both encoder and student states:
\begin{equation}
\mathcal L_z
=\sum_{h\in\mathcal H_S}\mathbb E_t
\big\|\hat z^{\mathrm{pH}}_{t+h}-\operatorname{sg}(z_{t+h})\big\|_2^2, \quad
\mathcal L_x
=\sum_{h\in\mathcal H_S}\mathbb E_t
\big\|\hat x_{t+h}-\phi(\operatorname{sg}(z_{t+h}))\big\|_2^2.
\label{eq:Ltraj}
\end{equation}
The primary term \(\mathcal L_z\) measures error in the fixed encoder coordinates, so changing the scale of student coordinates alone cannot reduce it without changing the mapped-back forecast.
The auxiliary term \(\mathcal L_x\) measures error in the student pH submanifold.
For a fixed invertible \(\phi\), the two terms vanish at the same trajectory matches, although their nonzero values depend on different coordinate metrics.
Encoder outputs remain fixed; gradients pass through \(\phi\) at both anchor and target states.

\textbf{Chart regularization and optimization.}
Teacher and student need not share the same coordinate system. In learning $\phi$ the mapping between these systems, we add an small identity penalty,
\(\mathcal L_\phi=\mathbb E_t\|\phi(z_t)-z_t\|_2^2\),
to favor only necessary coordinate changes on observed states. The transformation \(\phi\) and student model are trained jointly with
\begin{equation}
\mathcal L_S
=\mathcal L_z+\lambda_x\mathcal L_x+\lambda_\phi\mathcal L_\phi.
\label{eq:student_loss}
\end{equation}
Details on horizon sets, loss weights, and optimization settings are provided in Appendix~\ref{app:impl}.

\subsection{State recovery and structured realization}
\label{sec:theory}

Under the assumptions below, the theory follows the two stages of CIPHER: observation histories provide state information, predictive representations retain it, and a coordinate change permits a pH realization. A pendulum observed only through its angle provides a running example. Write \(D=d_z=d\), \(F_\Delta=\varphi_\star^\Delta\) for the exact sampled physical flow, \(G_w(s)=(g(F_\Delta^{-w}(s)),\ldots,g(s))\), and \(\eta=E_\psi\circ G_w\). Full statements and proofs are in Appendix~\ref{app:representation_realization}.

\textbf{Assumptions.}
The main result assumes latent state \(z_t\) and physical state \(s_t\) shares the same dimension \(D=d_s\), and the following conditions on the modeled state region and the trajectories needed to form its histories.
\begin{enumerate}
    \item[\textbf{(A1)}] \textbf{Smooth dynamics and observable histories.} The autonomous dynamics, noiseless observation map, and fixed encoding are \(C^2\) on open domains. The required flows exist, and \(G_w\) is injective with Jacobian rank \(d_s\): histories distinguish states and retain local direction.
    \item[\textbf{(A2)}] \textbf{Predictive and observation sufficiency.} There are \(C^2\) maps \(T_\Delta,r\) satisfying \(T_\Delta\circ\eta=\eta\circ F_\Delta\) and \(r\circ\eta=g\) throughout the regions needed for \(w\) forward steps. The readout \(r\) expresses recoverability of the current observation; it is an existence condition, not an additional trained decoder.
    \item[\textbf{(A3)}] \textbf{Compatible pH realization.} The physical dynamics admit a \(C^1\) pH field \(f_{\mathrm{pH}}^\star\) in the chosen field class under a \(C^2\) diffeomorphism \(\kappa\), so \(D\kappa\,f_\star=f_{\mathrm{pH}}^\star\circ\kappa\).
\end{enumerate}

\textbf{A1: From observation histories to momentum information.}
Only observing the pendulum angle \((q)\) does not distinguish \((q,p)\) from \((q,-p)\) for \(p\ne0\): the same position can be traversed in either direction. Thus, at sufficiently small sampling intervals, window observation locally determines both \(q_t\) and \(p_t\).  Assumption (A1) specifies the required observability for a general observation and modeled region (see Appendix~\ref{app:mechanical_observability} for more details).

\textbf{A2: Predictive contrast favors information retention.}
If the encoder merges two histories into one code, an autonomous predictor must assign them the same latent future. When their positive future keys are retained as negatives for one another, this collision incurs a contrastive penalty (Proposition~\ref{prop:contrastive_collision}). This motivates the information-retention conditions in (A2). Together with (A1), these conditions prevent distinct states from sharing a latent code and yield smooth full-state recovery.

\textbf{A3: From recovered states to structured dynamics.}
Once \(\eta\) is a diffeomorphic state chart, \(\phi^\star=\kappa\circ\eta^{-1}\) connects encoder states to the pH realization in (A3). This construction carries the true trajectories expressed in encoder coordinates to their pH realization.

\begin{theorem}[State recovery and nonlinear pH realization]
\label{thm:nonlinear_realization}
\textbf{(i) State recovery.} Under (A1)--(A2), with \(D=d_s\), the encoding \(\eta\) is a \(C^2\) diffeomorphism onto its image on the initial state region.

\textbf{(ii) Structured realization.} If (A3) also holds, then \(\phi^\star=\kappa\circ\eta^{-1}\) is a \(C^2\) diffeomorphism satisfying
\(
(\phi^\star)^{-1}\circ\varphi_{\mathrm{pH}}^u\circ\phi^\star
=\eta\circ\varphi_\star^u\circ\eta^{-1}\),
for every time \(u\) for which the corresponding solutions exist and remain in the regions where \(\eta\) and \(\kappa\) are defined.
\end{theorem}

The first result establishes a state chart, and the second transfers a compatible pH realization to those coordinates. Thus, the recovered representation admits a pH realization
whose trajectories correspond to the physical trajectories up to an invertible coordinate change. 

\textbf{Dimensions and scope.}
For \(D>d_s\), retaining (A2) and strengthening the history condition in (A1) to an embedding gives recovery on a \(d_s\)-dimensional latent manifold; \(D<d_s\) is incompatible with full-rank smooth recovery on an open physical domain (Appendix~\ref{app:dimension_extension}). The assumptions above describe exact population-level identities, not zero InfoNCE loss or guaranteed training outcomes. Appendix~\ref{app:population_interpretation} states their relation to data coverage; Appendix~\ref{app:realization_approximation} treats realization error. The following numerical study evaluates noise robustness and dimensional scaling in addition to the theoretical results of Theorem~\ref{thm:nonlinear_realization}.

\section{Numerical Evaluation}
\label{sec:experiments}
\label{sec:results}
\vspace{-10pt}

We investigate three questions: \textbf{Q1} how accurately does CIPHER predict dynamics? \textbf{Q2} what is gained by separating identification from structured fitting? \textbf{Q3} what do physical structure, coordinate learning, and latent dimension contribute?

These questions distinguish state-information recovery from prediction accuracy, assess the benefit of staged training, and examine the contribution of physical structure. Together, they connect the conditional theory in Section~\ref{sec:theory} to empirical performance under partial and noisy observations.

\textbf{Systems and datasets.}
We evaluate six noiseless-observation datasets and four noisy variants spanning conservative and dissipative dynamics, numerical and visual observations, and physical state vectors of dimension \(2\) to \(12\). A detailed descriptions of the dataset are provided in Appendix~\ref{app:systems_datasets}.

\textbf{State dimensions:}
We set \(D=d_z=d\) for the common encoder and student dimension and \(d_s\) for the physical state-vector dimension. The main comparisons set \(D=d_s\) using simulator metadata; this dimension is supplied, not inferred from observations. In applications with unknown state dimension, \(D\) is a modeling choice, and we evaluate CIPHER by varying latent dimensions in \textbf{Q3}.

For the noisy variants, i.i.d. Gaussian noise is added to observed body positions in all splits, leaving the reference states unchanged. The higher noise level is comparable to a planet's displacement between consecutive frames and to the star's reflex excursion, substantially degrading adjacent-frame velocity estimates. These tasks test whether the observation history supports state inference under noise, beyond the noiseless assumptions of Section~\ref{sec:theory}. Signal scales and finite-difference noise are detailed in Appendix~\ref{app:observation_noise}.

\textbf{Experimental Comparisons methods:} We compare multiple variants of three classes of contemporary baselines, DCL \cite{laiz2025dyncl}, Mamba-3 \cite{lahoti2026mamba3}, and PHAST \cite{bhardwaj2026phast}. There are three DCL \cite{laiz2025dyncl} variants possible for the above experimental conditions, each with different input data: instantaneous observations (DCL), consecutive position observations (DCL pair), and the full physical state (DCL\(^{*}\)) \citep{laiz2025dyncl}. Mamba-3 variants pivot on their prediction window, either the default configuration or horizon-matched to the teacher \citep{lahoti2026mamba3}. \textbf{PHAST} has only one variant (no physical knowledge given, similar to CIPHER).
Methods using video observations receive the same frozen VAE features. Implementation details appear in Appendices~\ref{app:compared_methods} and~\ref{app:phast}.

\textbf{Proposed Method Variants:}
Our method has two stages, each of which can be evaluated for accuracy: the \emph{Teacher} (a windowed encoder and neural ODE), and \emph{student} (CIPHER) (a learned invertible chart and state-dependent pH). CIPHER can additionally be compared with constant \(J,R\).

\textbf{One-stage comparison:}
We can instantiate a direct \textsc{CL+pHNN} with the same causal encoder arch. and contrastive objective as the teacher, but replacing the unconstrained vector field with a pH field, trained jointly using its own EMA target encoder. This comparison assesses the effect of imposing pH structure during representation learning, as the single-stage alternative to our proposed method. This is a pipeline-level comparison between direct structured identification and identify-then-realize approach. Results shown in Appendix Table~\ref{tab:staged_constant_comparison}) show this to be sub-optimal except in the most simple of experimental cases.

\textbf{Training and reporting.}
The main comparisons use five randomly selected seeds per dataset, with one student per configuration trained for each teacher using the same seed. Teacher and CIPHER use final checkpoints; external methods follow the training and evaluation policies in Appendix~\ref{app:compared_methods} and Appendix~\ref{app:phast}. We report mean \(\pm\) sample standard deviation over all runs, including poor-performing ones; teacher to student pairing measures complete pipeline variability. The dimension sweep uses three seeds per configuration. Physical state trajectories are used only for evaluation, except by DCL\(^{*}\). Architecture and optimization details appear in Appendices~\ref{app:impl} and~\ref{app:exp_metrics}.

\textbf{Autonomous prediction.}
A nonlinear alignment is fitted from each evaluated representation to physical states on training trajectories and frozen before evaluation. From every eligible test anchor, each model forecasts \(50\) steps without outside observations; student forecasts are mapped through \(\phi^{-1}\) and decoded by the same fitted probe as their matched teacher. We evaluate performance using dynamical coefficient of determination at step \(h\) denoted as \(R^2_{\mathrm{Dyn}}(h)\), whose value ranges from \((-\infty, 1]\). We report \(R^2_{\mathrm{Dyn}}(30)\), \(R^2_{\mathrm{Dyn}}(50)\), and \(\mathrm{AUC}_{1:50}=50^{-1}\sum_{h=1}^{50}R^2_{\mathrm{Dyn}}(h)\). Negative scores are retained in aggregation. Probe fitting and anchor selection are specified in Appendix~\ref{app:exp_metrics}.

\newcommand{\unc}[2]{%
  $#1{\scriptscriptstyle\,\pm\,#2}$%
} 
\newcommand{\uncb}[2]{%
  \ensuremath{\mathbf{#1}{\scriptscriptstyle\,\color{gray}{\pm\,#2}}}%
}
\newcommand{\uncu}[2]{%
  \underline{\ensuremath{#1{\scriptscriptstyle\,\color{gray}{\pm\,#2}}}}%
}

\begingroup
\newcommand{\resultscore}[2]{\ensuremath{#1{\scriptscriptstyle\,\pm\,#2}}}
\newcommand{\resultbest}[2]{\cellcolor[gray]{0.90}\ensuremath{\mathbf{#1}{\scriptscriptstyle\,\pm\,#2}}}
\newcommand{\resultsecond}[2]{\ensuremath{\underline{#1}{\scriptscriptstyle\,\pm\,#2}}}
\begin{table*}[t]
\centering
\caption{\textbf{Comparison with SOTA.} CIPHER (student model) is consistently strongest or closely competitive across all tasks. CIPHER uses state-dependent matrices; CIPHER (constant \(J,R\)) uses constant matrices. Entries are mean $\pm$ sample standard deviation over five runs. Bold with grey shading marks the highest reported mean; underline marks the second-highest distinct reported mean within each task and metric. Ties receive the same highlighting. $<0$ abbreviates a negative mean for readability. Table~\ref{tab:ablation_deltas} summarizes its ablations.}
\label{tab:results}
\scriptsize
\setlength{\tabcolsep}{2pt}
\renewcommand{\arraystretch}{1.14}
\resizebox{\textwidth}{!}{%
\begin{tabular}{@{}llccccccccc@{}}
\toprule
\multirow{2}{*}{\textbf{Task}} & \multirow{2}{*}{\shortstack{\textbf{Metric}\\($R^2_{\mathrm{Dyn}}$)}} & \multicolumn{6}{c}{\textbf{External baselines}} & \multicolumn{3}{c}{\textbf{Our framework (CIPHER) variants}} \\
\cmidrule(lr){3-8}\cmidrule(lr){9-11}
 & & \textbf{DCL} & \shortstack{\textbf{DCL}\\\textbf{pair}} & \textbf{DCL$^*$} & \shortstack{\textbf{PHAST}\\(self-cond.)} & \textbf{Mamba-3} & \shortstack{\textbf{Mamba-3}\\($h\le32$)} & \textbf{Teacher} & \shortstack{\textbf{CIPHER}\\(constant $J,R$)} & \textbf{CIPHER} \\
\midrule
\multirow{3}{*}{\shortstack[l]{Pendulum\\(numeric)}} & @30 & $<0$ & \resultscore{0.661}{0.015} & \resultscore{0.685}{0.005} & \resultscore{0.972}{0.002} & \resultsecond{0.998}{0.001} & \resultscore{0.974}{0.021} & \resultbest{0.999}{0.000} & \resultscore{0.977}{0.012} & \resultscore{0.977}{0.012} \\
 & @50 & $<0$ & \resultscore{0.095}{0.054} & \resultscore{0.196}{0.011} & \resultscore{0.970}{0.004} & \resultsecond{0.996}{0.002} & \resultscore{0.947}{0.034} & \resultbest{0.999}{0.000} & \resultscore{0.974}{0.002} & \resultscore{0.974}{0.002} \\
 & AUC & $<0$ & \resultscore{0.663}{0.016} & \resultscore{0.701}{0.003} & \resultscore{0.972}{0.002} & \resultbest{0.999}{0.001} & \resultscore{0.978}{0.016} & \resultbest{0.999}{0.000} & \resultsecond{0.984}{0.007} & \resultsecond{0.984}{0.006} \\
\midrule
\multirow{3}{*}{\shortstack[l]{Pendulum\\(video)}} & @30 & \resultscore{0.657}{0.025} & \resultscore{0.657}{0.025} & \resultscore{0.685}{0.005} & \resultscore{0.972}{0.001} & \resultsecond{0.994}{0.004} & \resultscore{0.982}{0.008} & \resultbest{0.997}{0.002} & \resultscore{0.979}{0.012} & \resultscore{0.981}{0.011} \\
 & @50 & \resultscore{0.074}{0.044} & \resultscore{0.074}{0.044} & \resultscore{0.196}{0.011} & \resultscore{0.972}{0.001} & \resultsecond{0.985}{0.010} & \resultscore{0.957}{0.019} & \resultbest{0.996}{0.004} & \resultscore{0.977}{0.004} & \resultscore{0.977}{0.003} \\
 & AUC & \resultscore{0.657}{0.016} & \resultscore{0.657}{0.016} & \resultscore{0.701}{0.003} & \resultscore{0.972}{0.001} & \resultsecond{0.995}{0.003} & \resultscore{0.984}{0.007} & \resultbest{0.998}{0.002} & \resultscore{0.985}{0.006} & \resultscore{0.986}{0.006} \\
\midrule
\multirow{3}{*}{\shortstack[l]{Duffing\\(video)}} & @30 & $<0$ & $<0$ & \resultscore{0.497}{0.088} & \resultscore{0.856}{0.031} & \resultsecond{0.995}{0.002} & \resultscore{0.991}{0.004} & \resultbest{0.997}{0.001} & \resultscore{0.972}{0.009} & \resultscore{0.971}{0.010} \\
 & @50 & $<0$ & $<0$ & $<0$ & \resultscore{0.826}{0.044} & \resultsecond{0.979}{0.010} & \resultscore{0.963}{0.015} & \resultbest{0.996}{0.001} & \resultscore{0.943}{0.029} & \resultscore{0.939}{0.031} \\
 & AUC & $<0$ & $<0$ & \resultscore{0.534}{0.038} & \resultscore{0.864}{0.030} & \resultsecond{0.994}{0.002} & \resultscore{0.990}{0.004} & \resultbest{0.998}{0.000} & \resultscore{0.976}{0.008} & \resultscore{0.975}{0.009} \\
\midrule
\multirow{3}{*}{\shortstack[l]{Three-mass\\chain}} & @30 & $<0$ & $<0$ & $<0$ & \resultsecond{0.994}{0.000} & \resultscore{0.990}{0.002} & \resultscore{0.983}{0.003} & \resultscore{0.986}{0.002} & \resultscore{0.992}{0.001} & \resultbest{0.995}{0.001} \\
 & @50 & $<0$ & $<0$ & $<0$ & \resultsecond{0.994}{0.000} & \resultscore{0.982}{0.004} & \resultscore{0.966}{0.008} & \resultscore{0.974}{0.004} & \resultscore{0.991}{0.002} & \resultbest{0.995}{0.001} \\
 & AUC & $<0$ & $<0$ & $<0$ & \resultsecond{0.994}{0.000} & \resultscore{0.990}{0.001} & \resultscore{0.985}{0.003} & \resultscore{0.988}{0.002} & \resultsecond{0.994}{0.001} & \resultbest{0.996}{0.000} \\
\midrule
\multirow{3}{*}{\shortstack[l]{Three-body\\(planar)}} & @30 & $<0$ & $<0$ & $<0$ & $<0$ & $<0$ & $<0$ & \resultscore{0.878}{0.088} & \resultsecond{0.899}{0.069} & \resultbest{0.920}{0.067} \\
 & @50 & $<0$ & $<0$ & $<0$ & $<0$ & $<0$ & $<0$ & \resultscore{0.810}{0.119} & \resultsecond{0.857}{0.094} & \resultbest{0.895}{0.092} \\
 & AUC & $<0$ & $<0$ & $<0$ & $<0$ & \resultscore{0.164}{0.133} & $<0$ & \resultscore{0.888}{0.077} & \resultsecond{0.908}{0.063} & \resultbest{0.926}{0.061} \\
\midrule
\multirow{3}{*}{\shortstack[l]{Three-body\\(3-D)}} & @30 & $<0$ & $<0$ & $<0$ & $<0$ & $<0$ & $<0$ & \resultscore{0.829}{0.160} & \resultsecond{0.869}{0.108} & \resultbest{0.900}{0.075} \\
 & @50 & $<0$ & $<0$ & $<0$ & $<0$ & $<0$ & $<0$ & \resultscore{0.736}{0.268} & \resultsecond{0.819}{0.163} & \resultbest{0.871}{0.109} \\
 & AUC & $<0$ & $<0$ & $<0$ & $<0$ & $<0$ & $<0$ & \resultscore{0.847}{0.138} & \resultsecond{0.881}{0.095} & \resultbest{0.907}{0.067} \\
\midrule
\multirow{3}{*}{\shortstack[l]{Three-body\\(planar)\\$\sigma=0.01$}} & @30 & $<0$ & $<0$ & $<0$ & \resultbest{0.943}{0.002} & \resultscore{0.707}{0.053} & $<0$ & \resultscore{0.901}{0.061} & \resultscore{0.912}{0.070} & \resultsecond{0.927}{0.063} \\
 & @50 & $<0$ & $<0$ & $<0$ & \resultsecond{0.892}{0.003} & \resultscore{0.195}{0.110} & $<0$ & \resultscore{0.846}{0.080} & \resultscore{0.871}{0.094} & \resultbest{0.899}{0.086} \\
 & AUC & $<0$ & $<0$ & $<0$ & \resultbest{0.953}{0.001} & \resultscore{0.722}{0.046} & \resultscore{0.121}{0.098} & \resultscore{0.910}{0.057} & \resultscore{0.920}{0.063} & \resultsecond{0.934}{0.058} \\
\midrule
\multirow{3}{*}{\shortstack[l]{Three-body\\(planar)\\$\sigma=0.05$}} & @30 & $<0$ & $<0$ & $<0$ & \resultscore{0.808}{0.006} & \resultsecond{0.864}{0.009} & \resultscore{0.839}{0.008} & \resultscore{0.846}{0.049} & \resultscore{0.845}{0.060} & \resultbest{0.879}{0.033} \\
 & @50 & $<0$ & $<0$ & $<0$ & \resultscore{0.605}{0.021} & \resultscore{0.729}{0.029} & \resultscore{0.664}{0.027} & \resultsecond{0.755}{0.075} & \resultscore{0.752}{0.089} & \resultbest{0.813}{0.049} \\
 & AUC & $<0$ & $<0$ & $<0$ & \resultscore{0.814}{0.006} & \resultsecond{0.885}{0.009} & \resultscore{0.861}{0.008} & \resultscore{0.868}{0.044} & \resultscore{0.866}{0.053} & \resultbest{0.895}{0.031} \\
\midrule
\multirow{3}{*}{\shortstack[l]{Three-body\\(3-D)\\$\sigma=0.01$}} & @30 & $<0$ & $<0$ & $<0$ & \resultscore{0.575}{0.449} & \resultscore{0.272}{0.057} & $<0$ & \resultscore{0.850}{0.112} & \resultsecond{0.881}{0.072} & \resultbest{0.904}{0.060} \\
 & @50 & $<0$ & $<0$ & $<0$ & $<0$ & $<0$ & $<0$ & \resultscore{0.779}{0.180} & \resultsecond{0.832}{0.108} & \resultbest{0.873}{0.090} \\
 & AUC & $<0$ & $<0$ & $<0$ & $<0$ & \resultscore{0.383}{0.045} & $<0$ & \resultscore{0.866}{0.096} & \resultsecond{0.891}{0.063} & \resultbest{0.911}{0.054} \\
\midrule
\multirow{3}{*}{\shortstack[l]{Three-body\\(3-D)\\$\sigma=0.05$}} & @30 & $<0$ & $<0$ & $<0$ & \resultscore{0.737}{0.007} & \resultscore{0.793}{0.005} & \resultscore{0.767}{0.016} & \resultscore{0.801}{0.073} & \resultsecond{0.831}{0.047} & \resultbest{0.865}{0.026} \\
 & @50 & $<0$ & $<0$ & $<0$ & \resultscore{0.460}{0.019} & \resultscore{0.592}{0.010} & \resultscore{0.546}{0.038} & \resultscore{0.690}{0.123} & \resultsecond{0.742}{0.075} & \resultbest{0.801}{0.039} \\
 & AUC & $<0$ & $<0$ & $<0$ & \resultscore{0.750}{0.007} & \resultscore{0.822}{0.004} & \resultscore{0.800}{0.014} & \resultscore{0.826}{0.062} & \resultsecond{0.850}{0.041} & \resultbest{0.878}{0.024} \\
\bottomrule
\end{tabular}%
}
\end{table*}
\endgroup
\textbf{Q1: How accurately do the models predict dynamics?}
Table~\ref{tab:results} shows complementary behavior across 10 tasks. For simple dynamics where performance are saturated, the teacher predicts nearly perfectly, while CIPHER remains competitive with external baselines while falling slightly below its teacher. However, on higher dimensional, complex, and noisy systems, CIPHER is superior to the teacher only learning and achieves the highest reported mean \(R^2_{\mathrm{Dyn}}(50)\) on the three-mass chain and every clean and noisy three-body condition, outperforming SOTA methods. This result highlights CIPHER's high performance, reliability, and generality across tasks, especially for more complex high-dimensional systems. 

\textbf{PHAST comparison in Table \ref{tab:results}:}
Self-conditioned PHAST is competitive on the pendulum and three-mass, and leads at the shorter reported horizon and in AUC on the lower-noise planar task. Its repeated observer updates differ from integrating a single latent pH trajectory (see Appendix~\ref{app:phast}). 

\textbf{Clean versus noisy observations:}
Some methods have small improvements on noisy observations despite sharing the clean tasks' underlying trajectories. Appendix Figures~\ref{fig:horizon_planar} and~\ref{fig:horizon_3d} show delayed degradation of some Mamba-3 and PHAST forecasts. Clean-test cross-evaluations (Appendix~\ref{app:noise_horizons}) support a training effect consistent with reduced sensitivity to errors, which matches regularization literature \cite{bishop1995training}, but this effect is task and noise level dependent and thus hard to calibrate.

\begin{figure}[!h]
    \centering
    \includegraphics[width=\linewidth]{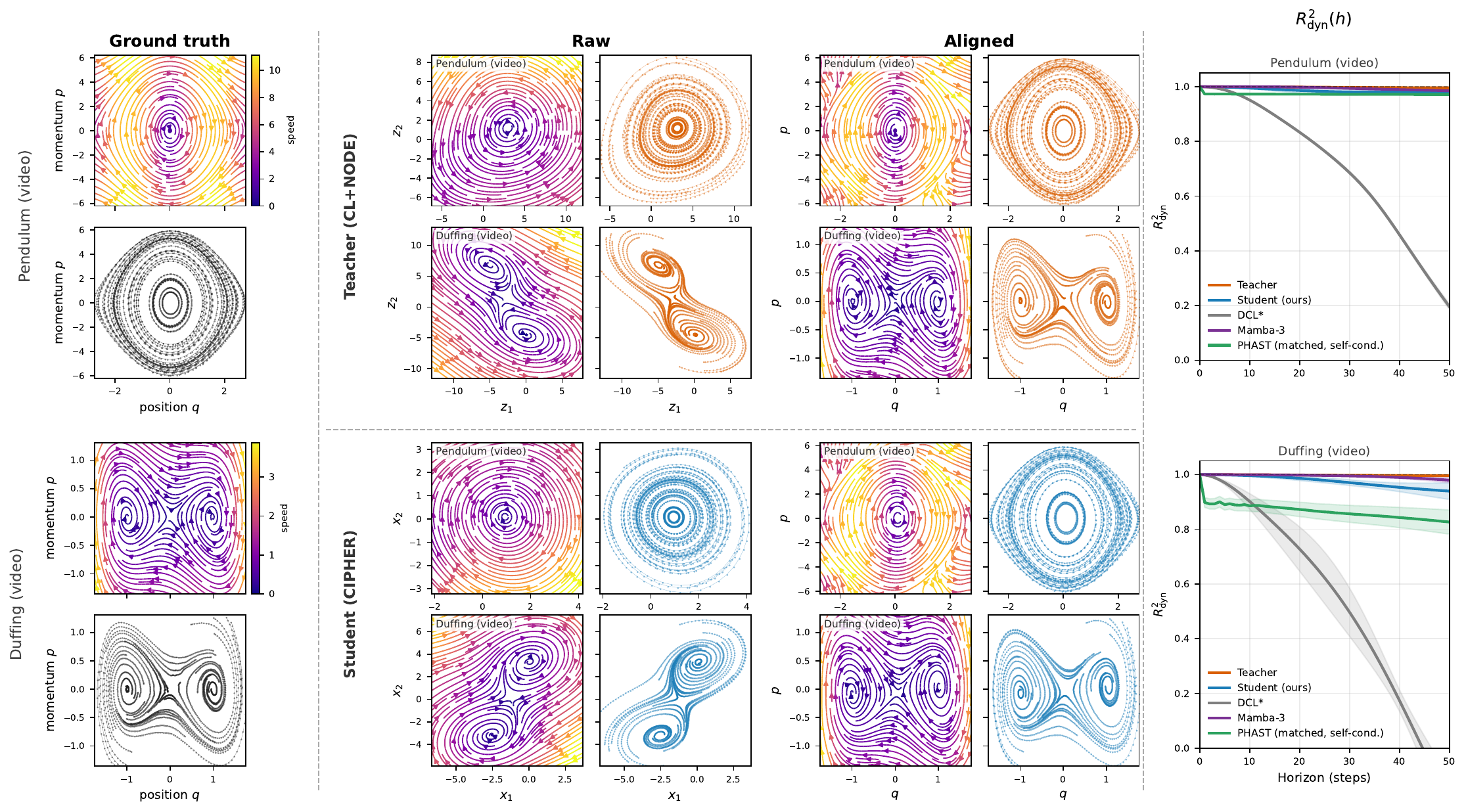}
    \caption{\textbf{Q1: how well does CIPHER predict.}
    We visualize the Teacher and CIPHER learned representation portrait with the ground truth for the pendulum and Duffing video tasks. We visualize methods' learned representations in raw (\(z,x\)) coordinates and aligned ones (after nonlinear alignment with ground truth). We also show how baselines perform over prediction horizon that match Table \ref{tab:results}. Predictive representations need not be canonical coordinates: even though raw latent axes need not represent physical position or momentum, the correlation to ground truth remain high. Appendix~\ref{app:representation_geometry} describes the plotting procedure and Figure~\ref{fig:representation_clouds_extended} includes three-body views.}
    \label{fig:representation_clouds}
    \vspace{-12pt}
\end{figure}
\textbf{Noncanonical state representations.}
\label{subsec:results_projection}
Figure~\ref{fig:representation_clouds} illustrates state recovery without canonical coordinate recovery. The teacher and student express oscillator dynamics in (nonlinearly) entangled latent space that become comparable to ground-truth physical trajectories through a nonlinear probe, see Appendix ~\ref{app:state_retention}. This is compatible with Section~\ref{sec:theory}, which identifies states up to a diffeomorphism. While pH dynamics not expressed in canonical state space, CIPHER's learned dynamics and prediction stay highly correlated with ground truth, in visual and quantitative manners.

\textbf{Q2: What is gained by staged training?} Appendix Table~\ref{tab:staged_constant_comparison} compares CIPHER (constant \(J,R\)) with one-stage CL+pHNN, both using constant interconnection and dissipation matrices. Matching the pH field parameterization controls for state-dependent matrix flexibility. The two-stage model achieves higher mean forecasting scores on nine of ten datasets across all three metrics; numerical pendulum is the exception. The clearest failure is pendulum video: one-stage forecasting means are negative despite retaining substantial state information (Figure~\ref{fig:state_retention}). One-stage is stronger on numerical pendulum and gives positive forecasts on Duffing video, so the results do not indicate a general collapse on visual inputs.

A plausible difficulty in joint training is that changes to
the encoder alter the latent trajectories that the structured
field must fit. Freezing the encoder removes this source of
nonstationarity during CIPHER student training. Concretely,
the same observation window may receive different latent codes
as the one-stage encoder is updated, requiring the pH field to
adapt to a changing representation. CIPHER instead optimizes
its coordinate map and pH dynamics against fixed encoded
observed futures, with the primary trajectory loss evaluated
in the frozen encoder's coordinates. The forecasting gains
across datasets support the identify-then-realize pipeline
for learning structured dynamics from partial observations.

\textbf{Q3: What do structure, coordinate learning, and latent dimension contribute?}
\label{subsec:results_ablations}
Table~\ref{tab:ablation_deltas} summarizes \(R^2_{\mathrm{Dyn}}(50)\) \textbf{ablations on the second stage}, full table in Appendix Table~\ref{tab:student_ablations}. CIPHER uses \textbf{pH structure}, \textbf{learned coordinate transformation}, and \textbf{state-dependent \(J(x), R(x)\)}. Fixed-coordinates indicates identity coordinate transformation to student model. Unconstrained Neural ODE means the student model is parameterized as Neural ODE, instead of pH as CIPHER. We highlight several observations from the ablation table.

\begin{table}[t]
\centering
\caption{\textbf{Component ablations.} \(R^2_{\mathrm{Dyn}}(50)\) on high-dimensional systems. Column 1 reports mean \(\pm\) standard deviation over five runs, while other entries are signed differences from its mean (variant minus CIPHER). Column 5 shows that the pH structure of CIPHER has a consistent advantage over Neural ODE. Columns 3 and 4 show coordinates have little effect for state-dependent pH, but improves constant-matrix pH, consistent with Corollary~\ref{cor:chart_absorption}. Additional details in Appendix Table~\ref{tab:student_ablations}.}
\label{tab:ablation_deltas}
\footnotesize
\setlength{\tabcolsep}{3pt}
\renewcommand{\arraystretch}{1.12}
\resizebox{\linewidth}{!}{%
\begin{tabular}{@{}lcccccc@{}}
\toprule
 & \multicolumn{2}{c}{\textbf{State-dependent} \(J(x),R(x)\)} & \multicolumn{2}{c}{\textbf{Constant-matrix} \(J,R\)} & & \\
\cmidrule(lr){2-3}\cmidrule(lr){4-5}
\textbf{Task} & \textbf{CIPHER} & \shortstack{\textbf{Fixed-coordinate}\\\textbf{CIPHER}} & \shortstack{\textbf{Learned}\\\textbf{chart}} & \shortstack{\textbf{Fixed}\\\textbf{coordinates}} & \shortstack{\textbf{Unconstrained}\\\textbf{Neural ODE}} & \shortstack{\textbf{Teacher-rollout}\\\textbf{supervision}} \\
\midrule
Three-body (3-D) & \ensuremath{0.871{\scriptscriptstyle\,\pm\,0.109}} & $-0.003$ & $-0.052$ & $-0.184$ & $-0.019$ & $-0.131$ \\
3-D, $\sigma=0.01$ & \ensuremath{0.873{\scriptscriptstyle\,\pm\,0.090}} & $-0.001$ & $-0.041$ & $-0.164$ & $-0.019$ & $-0.094$ \\
3-D, $\sigma=0.05$ & \ensuremath{0.801{\scriptscriptstyle\,\pm\,0.039}} & $-0.001$ & $-0.059$ & $-0.187$ & $-0.023$ & $-0.107$ \\
Planar, $\sigma=0.01$ & \ensuremath{0.899{\scriptscriptstyle\,\pm\,0.086}} & $-0.001$ & $-0.028$ & $-0.079$ & $-0.003$ & $-0.048$ \\
Planar, $\sigma=0.05$ & \ensuremath{0.813{\scriptscriptstyle\,\pm\,0.049}} & $+0.000$ & $-0.061$ & $-0.172$ & $-0.036$ & $-0.058$ \\
\bottomrule
\end{tabular}%
}
\end{table}

\textbf{Structured versus unconstrained dynamics.}
Both students use our realization stage. CIPHER exceeds the unconstrained Neural ODE in mean long-horizon accuracy on complex systems, with larger gains on three-body tasks. Table \ref{tab:ablation_deltas} therefore supports the additional predictive benefits of pH structure over unconstrained ODE, even without knowledge of governing equations. 

\textbf{Coordinate learning and state-dependent structure.}
When pH is parameterized with constant matrices, learning a chart improves performance on every dataset. State-dependent models perform similarly with fixed and learned coordinates, except on Duffing video. This is consistent with Corollary~\ref{cor:chart_absorption}: unrestricted state-dependent \(H,J,R\) can absorb an invertible coordinate change, whereas the constant-matrix class is not closed under arbitrary nonlinear coordinate changes.

\textbf{Latent-dimension sensitivity.}
Table \ref{tab:results} assumes true state dimension for comparison, so a logical question asks ``what if state dimension is unknown?" To answer this question, we vary the shared teacher and student latent dimensions on the three-mass chain and three-body systems (Appendix Table~\ref{tab:dimsweep}). At each dimension and seed, both pH and Neural ODE students learn from the same teacher. Figure~\ref{fig:dimsweep_relative} shows that 1) when the latent dimension is less than the true dimension, no methods can recover the full-state dynamics and 2) both realization-stage students remain accurate across larger latent dimensions; CIPHER additionally achieves positive AUC-error reductions relative to the size-matched Neural ODE at every tested \(D>d_s\).
Appendix~\ref{app:dimension_extension} gives conditional recovery and pH realization on an embedded manifold for \(D>d_s\).

\begin{figure}[t]
\centering
\includegraphics[width=\linewidth]{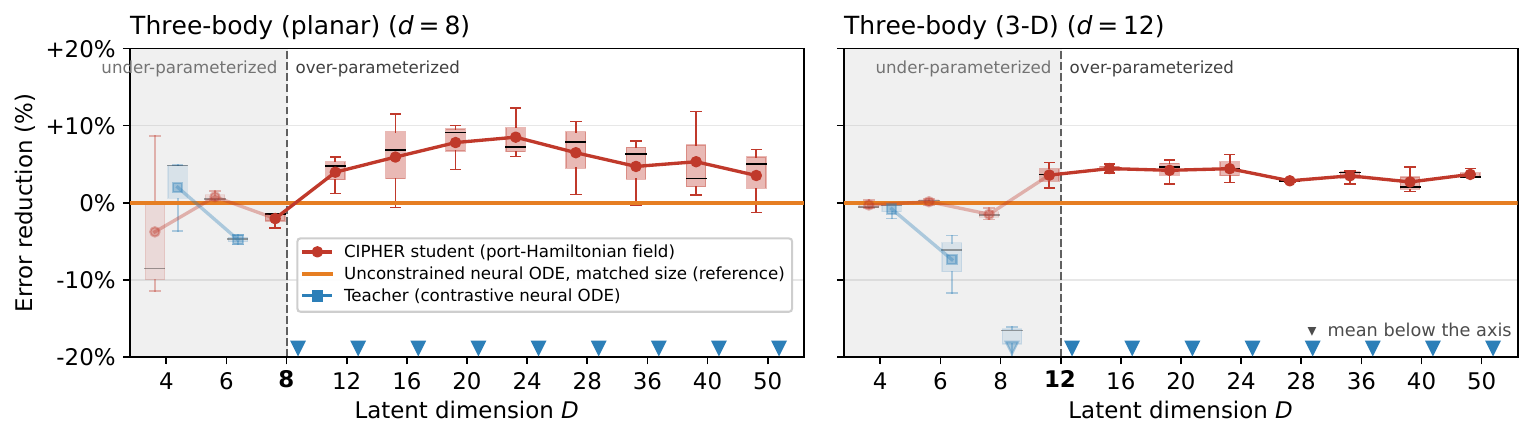}
\caption{\textbf{Forecast-error reduction across latent dimensions.} Relative reduction in normalized forecast error for \(\mathrm{AUC}_{1:50}\). The reference (zero line) is the parameter-matched, second-stage Neural ODE student. Positive values indicate lower error. Dashed vertical lines mark the physical state-vector size (labeled \(d\) in figure). Under-parameterized state cannot recover full dynamics. CIPHER is consistently more accurate than Neural ODE student across over-parameterized latent dimensions, indicating the robustness provided by pH structure. Additional details in Appendix Table~\ref{tab:dimsweep}.}
    \vspace{-12pt}
\label{fig:dimsweep_relative}
\end{figure}

\vspace{-10pt}
\section{Discussion and Conclusion}
\label{sec:discussion_conclusion}
\vspace{-10pt}

The paper introduces the two-stage framework CIPHER for learning port-Hamiltonian latent dynamics from partial observations. A contrastive teacher jointly learns a state encoder and a neural ODE, shaping the representation through temporal prediction on a unconstrained manifold. With learned latent states, a student transforms the representations into coordinates in which pH dynamics are defined and learned from the observed trajectories. We provide theoretical conditions and results for physical state recovery from partial observation under contrastive learning and pH modeling in latent space. The experiments show that CIPHER performs consistently well across systems and input modalities, even when the true latent state dimension is over-specified, while showing particularly strong advantages over state-of-the-art methods in high-dimensional systems and under noisy state observations.



\newpage
\bibliography{iclr2027_conference}
\bibliographystyle{iclr2027_conference}

\appendix
\newpage
\section{Appendix}
\subsection{Background details}
\subsubsection{CPC and InfoNCE equation}
\label{app:bg_infonce}
For query \(\hat z_{t+h}\), positive target \(\bar z_{t+h}\), and negative target encodings \(\{\bar z_j^-\}_{j=1}^K\), the unweighted InfoNCE loss is
\begin{equation}
\mathcal L_{\mathrm{NCE}}^{(h)}
=-\log\frac{\exp\ell(\hat z_{t+h},\bar z_{t+h})}
{\exp\ell(\hat z_{t+h},\bar z_{t+h})+
\sum_{j=1}^K\exp\ell(\hat z_{t+h},\bar z_j^-)}.
\end{equation}
The logit \(\ell\) includes the temperature: in Section~\ref{sec:method},
\(\ell(u,v)=-\|u-v\|^2/(\tau\operatorname{sg}(\sigma^2))\).
Target encodings and the variance in the logit are treated as constants when differentiating this loss. A per-coordinate standardized variant changes the distance metric rather than adding a second temperature; its use is described in Appendix~\ref{app:impl}.

\newpage
\subsection{Observable histories and nonlinear pH realization}
\label{app:representation_realization}

This appendix provides the full statements and proofs for Section~\ref{sec:theory}.
We first examine how deterministic prediction and contrastive discrimination restrict loss of dynamical information. We then give sufficient conditions for smooth state recovery and transfer an existing pH realization to the learned coordinates.
Identification means that the representation is a diffeomorphic coordinate system for the physical state on the modeled region, without requiring individual coordinates to have canonical or physical positions or momenta.
The results concern deterministic autonomous systems. We separate the state-sufficiency assumptions from the contrastive mechanisms that motivate them, then explain their population interpretation and extension to other latent dimensions. They are conditional statements, not convergence guarantees for the training procedure.

\subsubsection{Physical states, causal windows, and regularity}
\label{app:theory_setting}

Let \(\dot s=f_\star(s)\) on an open set \(\Omega\subset\mathbb R^d\), with \(f_\star\in C^2\), and write \(\varphi_\star^u\) for its local flow.
The observation map is \(g:\Omega\to\mathbb R^\ell\), \(g\in C^2\); the observation dimension \(\ell\) is unrelated to the momentum variable \(p\).
For sampling interval \(\Delta>0\), define \(F_\Delta=\varphi_\star^\Delta\).
All inverse-flow expressions below are restricted to states whose corresponding backward trajectories exist.

For a history length \(w+1\), with integer \(w\ge1\), define
\begin{equation}
G_w(s)=\big(g(F_\Delta^{-w}(s)),\ldots,g(s)\big),
\qquad G_w(s_t)=o_{t-w:t}.
\label{eq:causal_window_observability}
\end{equation}
Let \(V\subset\Omega\) be open and admit these histories, and let
\(\eta=E_\psi\circ G_w\in C^2(V,\mathbb R^d)\), with \(d_z=d\).
The encoder includes the observation front end and is fixed for the state-recovery arguments.
For an open initial-state region \(U\subset V\), assume the forward iterates up to \(w\) exist, and set
\[
U_j=F_\Delta^j(U),\qquad
P_w^+=\bigcup_{j=0}^w U_j\subset V.
\]
All windows needed to evaluate these encodings are therefore available.
Starting at time \(t\), the construction uses the encoder's usual past window and observations through \(t+w\); it does not change the inference input.

The window-observability conditions on a region \(\mathcal A\subset V\) are
\begin{equation}
G_w(s)=G_w(s')\ \Longrightarrow\ s=s',
\qquad
\operatorname{rank}DG_w(s)=d,
\qquad s,s'\in\mathcal A.
\label{eq:appendix_window_assumptions}
\end{equation}
For the recovery proposition below it suffices to impose them on \(\mathcal A=U_w\); observability on the larger modeled region \(V\) implies this condition.
These conditions concern a physical region, not merely distinct points of a finite training set.
If observations cover only a lower-dimensional state manifold, corresponding manifold assumptions are needed; distinguishability on a single trajectory does not certify recovery of all ambient phase-space directions.

\noindent\textbf{Injectivity and differential rank.}
The Jacobian \(DG_w(s)\) maps a physical-state perturbation \(v\) to its first-order effect on the observation window.
Rank \(d\) means \(DG_w(s)v=0\) only when \(v=0\).
This excludes loss of an infinitesimal state direction.
Injectivity excludes collisions between distinct states, whereas the rank condition excludes a singular local coordinate change.
For example, \(a\mapsto a^3\) is injective, but its derivative vanishes at zero and its inverse is not differentiable there.
Both properties are required for a diffeomorphic state chart.

\subsubsection{Mechanical observability and the position-only obstruction}
\label{app:mechanical_observability}

For \(s=(q,p)\in\mathbb R^{2n}\), consider
\(\dot q=M^{-1}p\), \(\dot p=a(q,p)\), with constant invertible \(M\) and \(a\in C^2\).
Suppose \(g(q,p)=q\).
At a state admitting sufficiently short backward trajectories,
\begin{equation}
D_pq(-\Delta;q,p)=-\Delta M^{-1}+O(\Delta^2)
\quad\text{as }\Delta\downarrow0.
\label{eq:past_momentum_observability}
\end{equation}
Consequently, \((q,p)\mapsto(q(-\Delta;q,p),q)\) has nonsingular Jacobian for all sufficiently small positive \(\Delta\), and is locally invertible.
A longer window containing this pair also has full differential rank locally.

\begin{proof}
Let \(Q_p(u)=D_pq(u;q,p)\) and \(P_p(u)=D_pp(u;q,p)\).
The variational equations give \(Q_p(0)=0\), \(P_p(0)=I\), and
\(\dot Q_p(u)=M^{-1}P_p(u)\).
Smoothness implies \(P_p(u)=I+O(|u|)\) near zero.
Integration to \(-\Delta\) gives \eqref{eq:past_momentum_observability}.
The Jacobian of the two-position map is
\[
\begin{pmatrix}
D_qq(-\Delta;q,p)&D_pq(-\Delta;q,p)\\
I&0
\end{pmatrix},
\]
which is nonsingular because its upper-right block is nonsingular.
Apply the inverse function theorem.
\end{proof}

For a free particle with \(M=I\), the pair is exactly \((q-\Delta p,q)\), so \(p=(q-q(-\Delta))/\Delta\).
This illustrates the information supplied by temporal observations.
The nonlinear result above is local: global injectivity on \(U\) is an additional condition.
It does not certify the experimental sampling interval or arbitrary partial-position measurements.
For example, position samples of a harmonic oscillator at a half-period satisfy \(q_{t+1}=-q_t\) and can omit momentum even in an arbitrarily long window.
Thus choosing a window is a design response to non-injective measurements; window observability remains a property of the system, sensor, and sampling.

\begin{proposition}[Sampled prediction with position-only encodings]
\label{prop:position_only_prediction}
Let \(U\subset\mathbb R^{2n}\) be connected and open, and let
\(Q_\Delta(q,p)=\pi_q F_\Delta(q,p)\) be the future configuration.
Suppose \(D_pQ_\Delta(q,p)\) is nonsingular throughout \(U\).
Let \(e:C\to\mathbb R^{d_z}\) be \(C^1\), where \(C\subset\mathbb R^n\) is open and contains both \(\pi_q(U)\) and \(Q_\Delta(U)\).
If a deterministic predictor \(T_\Delta\) satisfies
\begin{equation}
T_\Delta(e(q))=e(Q_\Delta(q,p))
\qquad ((q,p)\in U),
\label{eq:position_only_closure}
\end{equation}
then \(e\) is constant on \(Q_\Delta(U)\).
Thus exact sampled prediction with nonconstant future encodings is incompatible with a position-only representation on these states.
\end{proposition}

\begin{proof}
Hold \(q\) fixed and vary \(p\) in a neighborhood within \(U\).
The left side of \eqref{eq:position_only_closure} is constant in \(p\), so
\[
0=De(Q_\Delta(q,p))D_pQ_\Delta(q,p).
\]
Nonsingularity gives \(De=0\) on \(Q_\Delta(U)\).
This image is open because \(Q_\Delta\) has full differential rank, and connected because \(U\) is connected.
A \(C^1\) map with zero derivative on a connected open set is constant there.
The proof requires neither differentiability nor invertibility of \(T_\Delta\).
\end{proof}

The conclusion concerns the future configuration range.
It extends to the current configuration range if \(\pi_q(U)\subset Q_\Delta(U)\), or if \(T_\Delta\) is injective on \(e(\pi_q(U))\).
Without connectedness, the conclusion applies separately to connected components.
For the mechanical dynamics above, the same variational argument gives
\(D_pQ_\Delta=\Delta M^{-1}+O(\Delta^2)\), so nonsingularity holds locally for sufficiently small positive \(\Delta\).
Uniformity over a compact region requires a common short-time flow neighborhood and uniform derivative bounds.
This does not certify the experimental step size automatically.

The continuous-time counterpart is immediate.
If \(De(q)M^{-1}p=f_T(e(q))\) throughout an open phase-space region, differentiation with respect to \(p\) gives \(De(q)M^{-1}=0\).
Hence \(e\) is locally constant on the current configuration projection.
This stronger identity concerns tangent consistency and is not inferred from matching finitely many sampled horizons.
Likewise, small sampled prediction error alone does not justify differentiating its residual as if it vanished.

The proposition establishes momentum dependence as a necessary condition for nonconstant future encodings under exact sampled prediction in this setting; it does not establish recovery of every state component.
It concerns a fixed encoder used on both sides of the predictive identity.
During EMA training, the online and target encoders differ, so exact predictive closure of a common encoder is an idealized condition.

\subsubsection{Contrastive penalties for predictive collisions}
\label{app:contrastive_collisions}

Fix a realized batch and a common horizon.
Write \(u_i=\Phi_T^h(z_{t_i})\) for query \(i\), \(v_i=\bar z_{t_i+h}\) for its EMA positive key, and \(\mathcal N_i\) for its retained negative keys.
The raw per-anchor loss, before density weighting of anchors or averaging over horizons, is
\begin{equation}
\mathcal L_i=
\log\!\left(
e^{\ell(u_i,v_i)}
+\sum_{n\in\mathcal N_i}\omega_{i,n}e^{\ell(u_i,n)}
\right)-\ell(u_i,v_i).
\label{eq:raw_contrastive_loss}
\end{equation}
All scores are finite, both queries use the same score function and batch metric, and the negative weights are nonnegative.
The following statement is conditional on the actual keys present in the candidate sets.

\begin{proposition}[Penalty for coincident predictive queries]
\label{prop:contrastive_collision}
Suppose \(v_j\) is a retained negative for \(i\) with weight \(\omega_{ij}>0\), and \(v_i\) is a retained negative for \(j\) with weight \(\omega_{ji}>0\).
If \(u_i=u_j\), then
\begin{equation}
\frac{\mathcal L_i+\mathcal L_j}{2}
\ge \log\!\left(1+\sqrt{\omega_{ij}\omega_{ji}}\right).
\label{eq:contrastive_collision_bound}
\end{equation}
In particular, reciprocal unit-weight negatives give the lower bound \(\log2\).
\end{proposition}

\begin{proof}
Let \(a=\omega_{ij}\), \(b=\omega_{ji}\), and
\(c=\ell(u_i,v_j)-\ell(u_i,v_i)\).
Equal queries and a common score give the opposite contrast \(-c\) for anchor \(j\).
Dropping the other nonnegative denominator terms yields
\[
\mathcal L_i\ge\log(1+ae^c),\qquad
\mathcal L_j\ge\log(1+be^{-c}).
\]
By the arithmetic--geometric mean inequality,
\[
(1+ae^c)(1+be^{-c})
=1+ab+ae^c+be^{-c}
\ge(1+\sqrt{ab})^2.
\]
Taking logarithms proves the claim.
\end{proof}

This bound applies to identical anchor encodings because the shared deterministic rollout gives identical queries.
It also applies to query collisions produced by the predictor itself.
It does not require exact predictive closure, invertibility of RK4, or equality of online and EMA encoders.
If instead the two positive keys coincide, \(v_i=v_j\), direct inspection gives
\(\mathcal L_i\ge\log(1+\omega_{ij})\) and
\(\mathcal L_j\ge\log(1+\omega_{ji})\), regardless of the queries.
Thus constant future codes from Proposition~\ref{prop:position_only_prediction} obstruct reciprocal discrimination whenever the same frozen encoding supplies the keys and the scores are well-defined.

\noindent\textbf{Near-collisions.}
Define
\[
c_i=\ell(u_i,v_j)-\ell(u_i,v_i),\qquad
c_j=\ell(u_j,v_i)-\ell(u_j,v_j).
\]
The preceding proof, without equal queries, gives
\begin{equation}
\frac{\mathcal L_i+\mathcal L_j}{2}
\ge
\log\!\left(1+\sqrt{\omega_{ij}\omega_{ji}}\,
e^{(c_i+c_j)/2}\right).
\label{eq:near_collision_bound}
\end{equation}
Indeed, apply the same inequality to the nonnegative terms
\(\omega_{ij}e^{c_i}\) and \(\omega_{ji}e^{c_j}\).
For a common quadratic score \(\ell(u,v)=-(u-v)^\top A_{\mathrm b}(u-v)\), where \(A_{\mathrm b}\succeq0\),
\[
c_i+c_j=2(u_i-u_j)^\top A_{\mathrm b}(v_j-v_i)
\ge-2\|u_i-u_j\|\,\|A_{\mathrm b}(v_j-v_i)\|.
\]
Substitution into \eqref{eq:near_collision_bound} quantifies the penalty for nearby queries.
If \(\Phi_T^h\) is \(L_h\)-Lipschitz on the anchor region, one may further use
\(\|u_i-u_j\|\le L_h\|z_{t_i}-z_{t_j}\|\).
The scalar-normalized loss has \(A_{\mathrm b}=(\tau\sigma^2)^{-1}I\); a shared per-coordinate standardization gives a diagonal metric instead.
Detaching this metric affects gradients but not the stated loss inequalities.

\noindent\textbf{Weighting and sampling qualifications.}
Under the hypotheses of Proposition~\ref{prop:contrastive_collision}, for positive anchor weights \(w_i,w_j\), nonnegativity of the raw losses gives
\[
w_i\mathcal L_i+w_j\mathcal L_j
\ge2\min(w_i,w_j)
\log(1+\sqrt{\omega_{ij}\omega_{ji}}).
\]
An unequally weighted pair average therefore need not satisfy the unweighted \(\log2\) bound.
The reciprocal-negative premise must hold after temporal masking.
A stale bank key from the same physical timestamp need not equal the current EMA positive and cannot be substituted without checking that equality.
All statements assume a well-defined metric; the scalar-normalized loss as written is undefined at \(\sigma^2=0\).
These inequalities do not assign a finite value to that singular case.

The result is a penalty for exposed collisions, not a conclusion of global encoder injectivity from average training loss.
An average can remain small despite a poorly discriminated subset.
Moreover, in a continuous state distribution, a finite independent sample can contain no exactly colliding pair even when the encoding folds globally.
The near-collision bound quantifies a local obstruction but does not remove the need for coverage and state-sufficiency conditions.

\subsubsection{Predictive histories and sufficient state information}
\label{app:predictive_sufficiency}

\begin{proposition}[Recovery from predictive histories]
\label{prop:predictive_sufficiency}
Under the setting above, assume \eqref{eq:appendix_window_assumptions} on \(U_w=F_\Delta^w(U)\).
Let \(Z_0\) be an open latent region containing \(\eta(P_w^+)\).
Suppose \(T_\Delta:Z_0\to Z_0\) is \(C^2\), and
\(r:Z_0\to\mathbb R^\ell\) is \(C^2\), satisfying
\begin{equation}
\begin{aligned}
r(\eta(s))&=g(s) &&(s\in P_w^+),\\
T_\Delta(\eta(s))&=\eta(F_\Delta(s))
&&\left(s\in\bigcup_{j=0}^{w-1} U_j\right).
\end{aligned}
\label{eq:forward_predictive_closure}
\end{equation}
Then \(\eta|_U:U\to\eta(U)\) is a \(C^2\) diffeomorphism onto an open set.
It suffices that the predictor and readout be defined smoothly on neighborhoods of the finitely many required iterates.
The predictor need not be invertible.
\end{proposition}

\begin{proof}
Predictive closure gives, by induction,
\[
T_\Delta^j(\eta(s))=\eta(F_\Delta^j(s)),
\qquad s\in U,\quad 0\le j\le w.
\]
Define the future observation map and its latent realization by
\[
O_w(s)=\big(g(s),\ldots,g(F_\Delta^w(s))\big),\qquad
B_w(z)=\big(r(z),\ldots,r(T_\Delta^w(z))\big).
\]
The observation-retention identity in \eqref{eq:forward_predictive_closure} yields
\begin{equation}
B_w\circ\eta=O_w=G_w\circ F_\Delta^w
\quad\text{on }U.
\label{eq:future_window_factorization}
\end{equation}
The exact physical flow \(F_\Delta^w:U\to U_w\) is a \(C^2\) diffeomorphism.
This follows from smooth dependence on initial conditions and uniqueness, with inverse the physical backward flow on \(U_w\).
Consequently \(O_w\) is injective, and
\(DO_w(s)=DG_w(F_\Delta^w(s))DF_\Delta^w(s)\) has rank \(d\).
If \(\eta(s)=\eta(s')\), then \(O_w(s)=O_w(s')\), hence \(s=s'\).
Differentiating the other factorization gives
\[
DO_w(s)=DB_w(\eta(s))D\eta(s).
\]
The left side has rank \(d\), so the square matrix \(D\eta(s)\) has rank \(d\).
The inverse function theorem supplies \(C^2\) local inverses and shows that \(\eta(U)\) is open.
Injectivity makes those local inverses agree, giving a \(C^2\) inverse on \(\eta(U)\).
The physical state is therefore a smooth function \(\eta^{-1}(z)\) of the code.
\end{proof}

The condition \(r\circ\eta=g\) requires retention of the current partial observation.
Predictive closure extends this retention to a distinguishing sequence of future observations.
The readout \(r\) is a witness to information retention; it introduces no additional trained network or reconstruction loss.
The conclusion is joint recoverability of the physical state by a smooth inverse, without assigning physical meaning to individual coordinates.
The training mechanisms above motivate, but do not establish, the exact identities throughout the regions required by the proposition.

\noindent\textbf{Predictive closure and global non-injectivity.}
Consider the physical oscillator \(\dot q=p,\ \dot p=-q\), and the encoding
\[
\eta(q,p)=(q^2-p^2,\,2qp).
\]
Its evolution satisfies \(\dot z=2J_0z\), where
\(J_0=\left(\begin{smallmatrix}0&1\\-1&0\end{smallmatrix}\right)\).
The teacher flow is smooth and invertible, and predictive closure is exact.
Nevertheless, \(\eta(q,p)=\eta(-q,-p)\).
For \(g(q,p)=q\) and \(\sin\Delta\ne0\), the two-position causal window is injective, but this encoder discards the sign information it contains.
No global readout \(r\) can recover \(q\) from this code.
The observation-retention premise excludes this ambiguity.
Away from the origin, \(D\eta\) is nonsingular, so full local rank alone would not exclude the global folding either.

\noindent\textbf{Scope of the contrastive argument.}
DCL's affine-identifiability theorem uses an injective observation map together with a specified stochastic transition, score correction, and population optimum \citep{laiz2025dyncl}.
Its discussion also considers time-lag embeddings for partial observations.
An injective window addresses the observation condition, but does not establish the remaining hypotheses for our deterministic, multi-horizon teacher with EMA targets and a memory bank.
We therefore use a conditional smooth-state result, rather than conclude an affine gauge \(\eta(s)=Ls+\beta\) for the implemented training procedure.

More generally, even predictive contrast need not exclude every global folding.
On an annulus, the smooth map
\(\eta_k(\rho\cos\theta,\rho\sin\theta)=(\rho\cos(k\theta),\rho\sin(k\theta))\), for an integer \(k\ge2\), is \(k\)-to-one and follows the invertible oscillator flow \(\dot z=kJ_0z\).
For a rotation-invariant state distribution, a fixed score metric, and independent stationary negatives, it preserves the latent distribution and gives the same per-query exact-positive distance-InfoNCE loss distribution as \(k=1\).
Indeed, the angular map preserves the uniform angle distribution at each radius, so the query and independent negative keys have the same joint law; exact prediction sets the positive equal to the query.
This illustrative comparison is not a claim about our temporally filtered sampler.
It shows why global encoder injectivity requires conditions beyond ODE uniqueness and the use of a contrastive objective.

\noindent\textbf{Exact versus numerical prediction.}
Proposition~\ref{prop:predictive_sufficiency} can use \(T_\Delta=\Phi_T^1\), provided its smoothness and the stated identities hold.
It does not need numerical-map invertibility.
The exact teacher flow is invertible where its trajectories exist, but its vector field need not be an injective function.
The mechanism used here is deterministic prediction.
Furthermore, matching sampled latent futures does not by itself establish equality of the underlying continuous-time generators.

\subsubsection{Population assumptions and the meaning of nonzero loss}
\label{app:population_interpretation}

Assumptions (A1)--(A2) in Section~\ref{sec:theory} specialize the preceding setup to observable histories and exact predictive/observation sufficiency on the required physical regions. They are pointwise conditions on the learned functions, not statements about a particular number of training trajectories. A population formulation makes the role of coverage explicit.

\begin{lemma}[From zero population residual to state sufficiency]
\label{lem:population_sufficiency}
Retain the smoothness and domain conditions of Proposition~\ref{prop:predictive_sufficiency}. Let \(\mu\) be a probability measure with full support on \(U\), meaning that every nonempty relatively open subset of \(U\) has positive probability. For fixed \(\eta,T_\Delta,r\), define
\begin{equation}
\begin{aligned}
\mathcal R_{\mathrm{pop}}
=\mathbb E_{s\sim\mu}\Big[&
\sum_{j=0}^{w-1}\|T_\Delta(\eta(F_\Delta^j(s)))-\eta(F_\Delta^{j+1}(s))\|^2\\
&+\sum_{j=0}^{w}\|r(\eta(F_\Delta^j(s)))-g(F_\Delta^j(s))\|^2\Big].
\end{aligned}
\label{eq:population_sufficiency_residual}
\end{equation}
If \(\mathcal R_{\mathrm{pop}}=0\), the identities in \eqref{eq:forward_predictive_closure} hold throughout the required regions. With window observability and \(d_z=d\), the encoding is therefore a \(C^2\) diffeomorphism onto its image on \(U\).
\end{lemma}

\begin{proof}
The integrand is a continuous nonnegative function of \(s\). If it were positive at any point of \(U\), continuity would bound it below by a positive constant on a nonempty open neighborhood. Full support would give that neighborhood positive measure, contradicting zero expectation. Every summand consequently vanishes at every \(s\in U\), yielding the identities on each \(U_j=F_\Delta^j(U)\). Proposition~\ref{prop:predictive_sufficiency} applies.
\end{proof}

This lemma formalizes a realizable population idealization. The residual \(\mathcal R_{\mathrm{pop}}\) is an analytical quantity, not the implemented teacher objective, and \(r\) remains a witness to observation retention. Infinite sampling alone neither supplies full support on an open physical region nor forces this residual to vanish. In particular, sampling one orbit does not by itself establish full support on an open phase-space region.

\paragraph{Why zero InfoNCE is not the assumption.}
With finite logits and at least one positive-weight negative, the raw loss in \eqref{eq:raw_contrastive_loss} is strictly positive:
\(\mathcal L_i=\log(1+\sum_n\omega_{i,n}e^{\ell(u_i,n)-\ell(u_i,v_i)})>0\).
The contrastive objective rewards relative discrimination \citep{oord2018cpc}; its value is not a latent prediction residual. Even a sequence of losses approaching zero need not give accurate predictions. For example, use positive keys \(v_1=0,v_2=1\), reciprocal unit-weight negatives, queries \(u_1=-a,u_2=1+a\), and score \(\ell(u,v)=-(u-v)^2/c\), with fixed \(c>0\). Both losses equal \(\log(1+e^{-(2a+1)/c})\), which tends to zero as \(a\to\infty\), while both squared prediction errors equal \(a^2\). Fixed target-variance normalization leaves this example unchanged. The queries can also be produced from anchors \(0,1\) by the autonomous field \(\dot z=\lambda(z-\tfrac12)\), with \(e^{\lambda\Delta}=2a+1\).

Thus population-optimal contrastive discrimination and exact predictive sufficiency are distinct conditions. The collision bound explains a pressure against merging distinguishable training cases, while the state-recovery proposition states the additional conditions that suffice for an invertible state representation. For nonzero residuals, infinite data does not turn an approximate identity into an exact one. Quantitative recovery statements would additionally need error bounds and conditioning assumptions on the observation inverse; the realization bound in Appendix~\ref{app:realization_approximation} instead starts from an existing state chart.

\subsubsection{State recovery with under- and over-specified latent dimension}
\label{app:dimension_extension}

For this extension write \(n=d_s\) for the physical-domain dimension and \(D=d_z=d\) for the encoder/student dimension. Let \(U\subset\mathbb R^n\) be open and use the same physical flow, history regions, and predictive identities as above, with \(\eta:V\to\mathbb R^D\). A \(C^2\) embedding means an injective \(C^2\) map of rank \(n\) whose inverse onto its image is continuous. It identifies the physical domain with a smooth submanifold of the ambient latent space.

\begin{proposition}[Recovery as an embedded state manifold]
\label{prop:embedded_state_recovery}
Assume the smoothness and domain conditions of Proposition~\ref{prop:predictive_sufficiency}, allowing arbitrary \(D\), and its predictive-closure and observation-retention identities. Suppose \(G_w|_{U_w}\) is a \(C^2\) embedding. Then \(D\ge n\) and \(\eta|_U\) is a \(C^2\) embedding. For \(D=n\), its image is open and it is an open-domain diffeomorphism. For \(D>n\), it is a diffeomorphism onto an embedded \(n\)-dimensional manifold \(\mathcal M=\eta(U)\subset\mathbb R^D\).
\end{proposition}

\begin{proof}
As in \eqref{eq:future_window_factorization}, define \(B_w(z)=(r(z),\ldots,r(T_\Delta^w(z)))\) and \(O_w=G_w\circ F_\Delta^w\), so \(B_w\circ\eta=O_w\). Since the physical flow is a diffeomorphism, \(O_w\) is an embedding. Its injectivity forces \(\eta\) to be injective, and
\[
n=\operatorname{rank}DO_w(s)
\le\operatorname{rank}D\eta(s)\le\min(n,D).
\]
Thus \(D\ge n\) and \(\eta\) is an immersion. Moreover, \(\eta^{-1}=O_w^{-1}\circ B_w|_{\mathcal M}\) is continuous. Hence \(\eta\) is an embedding and has a smooth inverse as a map between manifolds. When \(D=n\), the inverse function theorem also makes \(\mathcal M\) open.
\end{proof}

The embedding premise strengthens the injectivity/rank condition to ensure a globally continuous inverse in the overcomplete case. It can be verified locally by a suitable observation chart, but does not follow globally from an injective immersion alone. When \(D<n\), the rank inequality shows that the exact sufficiency assumptions cannot all hold on an open \(n\)-dimensional physical domain. This obstruction does not apply to a dataset supported on a lower-dimensional manifold merely because its states are stored using \(n\) coordinates.

\begin{proposition}[pH dynamics on an embedded state manifold]
\label{prop:embedded_ph_realization}
Let \(\eta:U\to\mathbb R^D\) be a \(C^2\) embedding, with \(U\subset\mathbb R^n\) open, and suppose \(f=(J-R)\nabla H\) on \(U\), with \(H\in C^2\), \(J,R\in C^1\), \(J^\top=-J\), and \(R^\top=R\succeq0\). Set \(\mathcal M=\eta(U)\) and \(A(s)=D\eta(s)\). At \(z=\eta(s)\), define
\begin{equation}
H_{\mathcal M}(z)=H(s),\qquad
J_{\mathcal M}(z)=A(s)J(s)A(s)^\top,\qquad
R_{\mathcal M}(z)=A(s)R(s)A(s)^\top.
\label{eq:embedded_ph_tensors}
\end{equation}
These matrices preserve skew symmetry and positive semidefiniteness. With the metric induced by \(\mathbb R^D\), the transported dynamics satisfy
\((J_{\mathcal M}-R_{\mathcal M})\nabla_{\mathcal M}H_{\mathcal M}=A f\).
They have a local ambient pH extension around each point of \(\mathcal M\), leaving \(\mathcal M\) locally invariant under its exact flow.
\end{proposition}

\begin{proof}
The rank condition makes \(A^\top A\) invertible. The tangent gradient is
\(\nabla_{\mathcal M}H_{\mathcal M}=A(A^\top A)^{-1}\nabla H\), since its inner product with \(Av\) is \(\nabla H^\top v\) for every physical tangent vector \(v\). Substituting this gradient in \eqref{eq:embedded_ph_tensors} yields \(A(J-R)\nabla H=A f\). Skew symmetry and positive semidefiniteness follow by transposition and congruence. The resulting field is tangent to \(\mathcal M\), and the chain rule conjugates its exact trajectories to those of \(f\).

Locally, a submanifold chart supplies a \(C^2\) map \(\rho\) from an ambient neighborhood into \(U\) with \(\rho\circ\eta=\mathrm{id}\). Define \(\widetilde H=H\circ\rho\), \(\widetilde J=(AJA^\top)\circ\rho\), and \(\widetilde R=(ARA^\top)\circ\rho\). On \(\mathcal M\), differentiating \(\rho\circ\eta=\mathrm{id}\) gives \(D\rho\,A=I\), so
\[
(\widetilde J-\widetilde R)\nabla\widetilde H
=A(J-R)A^\top D\rho^\top\nabla H=A f.
\]
This is a \(C^1\) ambient pH field with the required matrix properties. Uniqueness makes its trajectories initialized on \(\mathcal M\) agree locally with the embedded physical trajectories.
\end{proof}

If the physical field admits pH structure only in coordinates \(\kappa(s)\), apply the proposition to \(\eta\circ\kappa^{-1}\) and the pH field in those coordinates. This establishes smooth pH realizability on the embedded image. It does not assert that an arbitrary overcomplete embedding can be straightened by one globally invertible ambient coupling network, or that a finite neural field exactly represents the extension. Those require additional extension and approximation assumptions. Dynamics normal to the data manifold are also not identified by trajectories lying on it. Consequently, larger \(D\) permits full-state recovery under the stated assumptions, but neither guarantees successful optimization nor predicts monotone improvement with dimension.

\subsubsection{Full statement and proof of nonlinear pH covariance}
\label{app:nonlinear_ph_covariance}

\begin{proposition}[Nonlinear pH covariance]
\label{prop:nonlinear_ph_covariance}
Let \(U,X\subset\mathbb R^d\) be open, \(H\in C^2(U)\), and
\(J,R\in C^1(U,\mathbb R^{d\times d})\), with
\(J(s)^\top=-J(s)\) and \(R(s)^\top=R(s)\succeq0\).
Suppose \(f(s)=(J(s)-R(s))\nabla H(s)\).
For a \(C^2\) diffeomorphism \(\chi:U\to X\), define
\begin{equation}
\widetilde H(x)=H(s),\qquad
\widetilde J(x)=B J(s)B^\top,\qquad
\widetilde R(x)=B R(s)B^\top,
\label{eq:nonlinear_ph_covariance}
\end{equation}
where \(s=\chi^{-1}(x)\) and \(B=D\chi(s)\).
Then the transported field is
\[
\widetilde f(x)=D\chi(\chi^{-1}(x))f(\chi^{-1}(x))
=(\widetilde J(x)-\widetilde R(x))\nabla\widetilde H(x),
\]
with the stated skew and positive-semidefinite properties.
Its exact flow satisfies
\(\varphi_{\widetilde f}^u=\chi\circ\varphi_f^u\circ\chi^{-1}\)
on trajectories contained in the coordinate domains.
\end{proposition}

\begin{proof}
At \(s=\chi^{-1}(x)\), write \(B=D\chi(s)\).
The chain rule gives
\(\nabla_x\widetilde H(x)=B^{-\top}\nabla_sH(s)\).
Therefore
\[
D\chi(s)f(s)
=B(J(s)-R(s))B^\top\nabla_x\widetilde H(x)
=(\widetilde J-\widetilde R)\nabla_x\widetilde H(x).
\]
Taking transposes proves \(\widetilde J^\top=-\widetilde J\).
For every vector \(v\),
\(v^\top\widetilde Rv=(B^\top v)^\top R(B^\top v)\ge0\).
For an exact physical trajectory \(s(u)\), the chain rule shows that
\(\chi(s(u))\) solves the transported ODE.
Both fields are \(C^1\), so uniqueness establishes flow conjugacy.
In particular, along the corresponding solutions,
\[
\frac{\mathrm d}{\mathrm d u}\widetilde H(x(u))
=-\nabla\widetilde H^\top\widetilde R\nabla\widetilde H
=-\nabla H^\top R\nabla H.
\]
\end{proof}

This is covariance of the unforced pH energy-balance form: the represented dynamics are unchanged, while their coordinate-dependent tensors change.
The argument is consistent with the coordinate formulation of pH systems \citep{vanderschaft2014phs}.
It does not require the map to preserve canonical position and momentum coordinates.
If \(J\) additionally defines a Poisson bracket, its Jacobi identity is preserved by this pushforward: for test functions \(a,b\),
\(\{a,b\}_{\widetilde J}\circ\chi=\{a\circ\chi,b\circ\chi\}_J\), so the transformed Jacobi expression is the original one after composition.
An arbitrary learned state-dependent skew matrix does not, in general dimension, enforce that additional identity.

\subsubsection{Absorbing the student chart into state-dependent structure}
\label{app:chart_absorption}

The following consequence of Proposition~\ref{prop:nonlinear_ph_covariance} compares a charted student with a pH field fitted directly in the fixed encoder coordinates. It is a statement about smooth field classes and does not require the learned encoder to identify the physical state.

\begin{corollary}[Absorption of an invertible chart]
\label{cor:chart_absorption}
Let \(Z,X\subset\mathbb R^d\) be open, and let \(\phi:Z\to X\) be a \(C^2\) diffeomorphism. Suppose \(H\in C^2(X)\) and \(J,R\in C^1(X,\mathbb R^{d\times d})\), with \(J^\top=-J\) and \(R^\top=R\succeq0\). For \(f_x=(J-R)\nabla_xH\), the field represented by the charted student in encoder coordinates is \(f_z(z)=D\phi(z)^{-1}f_x(\phi(z))\). Writing \(A(z)=D\phi(z)\), define
\begin{equation}
\begin{aligned}
H_z(z)&=H(\phi(z)),\\
J_z(z)&=A(z)^{-1}J(\phi(z))A(z)^{-\top},\\
R_z(z)&=A(z)^{-1}R(\phi(z))A(z)^{-\top}.
\end{aligned}
\label{eq:chart_absorption}
\end{equation}
Then \(H_z\in C^2(Z)\), \(J_z,R_z\in C^1(Z,\mathbb R^{d\times d})\), \(J_z^\top=-J_z\), \(R_z^\top=R_z\succeq0\), and \(f_z=(J_z-R_z)\nabla_zH_z\). The exact flows satisfy \(\varphi_{f_z}^u=\phi^{-1}\circ\varphi_{f_x}^u\circ\phi\) wherever the corresponding trajectories remain in these domains. Consequently, allowing arbitrary functions with these regularities and matrix properties, charted pH fields and fixed-coordinate state-dependent pH fields represent the same class of continuous vector fields on \(Z\).
\end{corollary}

\begin{proof}
The chain rule gives \(\nabla_zH_z=A^\top\nabla_xH\circ\phi\). Hence
\[
(J_z-R_z)\nabla_zH_z
=A^{-1}\bigl(J(\phi(z))-R(\phi(z))\bigr)\nabla_xH(\phi(z))
=A^{-1}f_x(\phi(z))=f_z(z).
\]
Skew symmetry follows by transposition, and \(v^\top R_zv=(A^{-\top}v)^\top R(\phi(z))(A^{-\top}v)\ge0\) for every \(v\in\mathbb R^d\). The stated regularities follow from the assumptions and invertibility of \(A\). Uniqueness of solutions for the \(C^1\) fields gives exact-flow conjugacy. This proves that every charted field is in the fixed-coordinate state-dependent class; the converse follows by choosing \(X=Z\) and \(\phi=\mathrm{id}\).
\end{proof}

\noindent\textbf{Implications for the coordinate ablation.}
The chart is absorbed into \(H\), \(J\), and \(R\) together. For the unrestricted state-dependent class it provides an alternative parameterization, rather than additional continuous-time expressiveness. A finite neural parameterization need not contain the transformed functions in \eqref{eq:chart_absorption}, however, and a chart can still affect approximation, optimization, and regularization. In particular, the student-coordinate loss \(\mathcal L_x\) depends on the chart; with \(\phi=\mathrm{id}\) it coincides with \(\mathcal L_z\), whereas a learned chart changes its error geometry. The corollary therefore does not assert equality of training objectives or learned solutions. It also concerns exact flows: a fixed-step RK4 update generally does not commute with a nonlinear coordinate transformation. Section~\ref{subsec:results_ablations} compares coordinate learning under constant and state-dependent matrix parameterizations to assess this dependence empirically.

\noindent\textbf{Constant-matrix realizations.}
For an affine map, \(B\) is constant, recovering the previous affine covariance result.
For a nonlinear map, even constant \(J,R\) can become state-dependent.
For example,
\[
\chi(q,p)=(Q,P)=(q,e^q p),\qquad
D\chi=
\begin{pmatrix}1&0\\P&e^Q\end{pmatrix}
\]
is a global smooth diffeomorphism with inverse \((q,p)=(Q,e^{-Q}P)\).
It transports \(J_0\) to \(e^QJ_0\), and
\(R=\operatorname{diag}(0,\gamma)\) to
\(\widetilde R=\operatorname{diag}(0,\gamma e^{2Q})\).

The constant-matrix class itself is not closed under arbitrary nonlinear coordinate changes.
To see this, take the conservative oscillator \(H=(q^2+p^2)/2\), \(R=0\), and the same chart.
The transformed field is
\[
\widetilde f(Q,P)=
\begin{pmatrix}
e^{-Q}P\\
e^{-Q}P^2-e^Q Q
\end{pmatrix},
\qquad
\operatorname{div}\widetilde f=e^{-Q}P.
\]
On the image of a complete open oscillator annulus, every trajectory is periodic and the divergence is not identically zero.
If this field admitted any representation
\(\widetilde f=(J_c-R_c)\nabla K\) there, with constant skew \(J_c\), constant \(R_c\succeq0\), and \(K\in C^2\), integrating its energy balance around each periodic orbit would give \(R_c\nabla K=0\) everywhere.
Hence \(\widetilde f=J_c\nabla K\), which has zero divergence because
\(\operatorname{tr}(J_c\nabla^2K)=0\), a contradiction.
Thus general pH covariance does not prove that a prescribed nonlinear chart retains a constant-matrix realization.
Learning the chart permits a return to coordinates that do admit the chosen class.

\subsubsection{Full nonlinear identification--pH bridge}
\label{app:nonlinear_realization}

\noindent\textbf{Full statement of Theorem~\ref{thm:nonlinear_realization}.}
Under the smoothness, domain, observability, and predictive-sufficiency assumptions of Proposition~\ref{prop:predictive_sufficiency}, with \(d_s=d_z=d\), the encoder \(\eta|_U:U\to Z=\eta(U)\) is a \(C^2\) diffeomorphism onto an open domain. These are the full conditions summarized by (A1)--(A2) in Section~\ref{sec:theory}. More generally, the following realization conclusion only needs this coordinate property of \(\eta\), however it is established. Let \(\kappa:U\to X\) be a \(C^2\) diffeomorphism onto an open coordinate domain.
Suppose there exists a \(C^1\) pH field \(f_{\mathrm{pH}}^\star\) on \(X\) such that
\[
D\kappa(s)f_\star(s)=f_{\mathrm{pH}}^\star(\kappa(s)).
\]
Compatibility requires \(f_{\mathrm{pH}}^\star\) to belong to the chosen student's field class, including any restrictions on \(H,J,R\).
For the constant-matrix configuration, require
\(f_{\mathrm{pH}}^\star=(J_\star-R_\star)\nabla H_\star\), where \(H_\star\in C^2\), \(J_\star,R_\star\) are constant, \(R_\star=R_\star^\top\succeq0\), and \(J_\star\) satisfies the model's dimension-specific skew-matrix restriction.
For a state-dependent configuration, use its corresponding field class instead.
Then \(\phi^\star=\kappa\circ\eta^{-1}:Z\to X\) is a \(C^2\) diffeomorphism satisfying
\begin{equation}
(\phi^\star)^{-1}\!\left(\varphi_{\mathrm{pH}}^u(\phi^\star(z))\right)
=\eta\!\left(\varphi_\star^u(\eta^{-1}(z))\right),
\label{eq:nonlinear_realization}
\end{equation}
where \(\varphi_{\mathrm{pH}}^u\) denotes the exact flow of \(f_{\mathrm{pH}}^\star\). This identity holds for every \(z\in Z\) and time \(u\) for which the physical trajectory from \(\eta^{-1}(z)\) exists and remains in \(U\) throughout the interval between \(0\) and \(u\).
If the student classes contain the required restrictions of \(\phi^\star\) and \(f_{\mathrm{pH}}^\star\), exact integration gives \(\mathcal L_z=\mathcal L_x=0\) on noiseless encoded trajectories contained in these domains.

\begin{proof}
Proposition~\ref{prop:predictive_sufficiency} gives the state-recovery conclusion under (A1)--(A2). Given that coordinate property, the inverse is \((\phi^\star)^{-1}=\eta\circ\kappa^{-1}\), so both the chart and its inverse are \(C^2\).
For \(s(u)=\varphi_\star^u(s_0)\), the chain rule gives
\[
\frac{\mathrm d}{\mathrm d u}\kappa(s(u))
=f_{\mathrm{pH}}^\star(\kappa(s(u))).
\]
Uniqueness therefore yields
\(\varphi_{\mathrm{pH}}^u(\kappa(s_0))=\kappa(\varphi_\star^u(s_0))\).
Substituting \(s_0=\eta^{-1}(z)\) and composing with \(\eta\circ\kappa^{-1}\) proves \eqref{eq:nonlinear_realization}.
At \(u=h\Delta\), exact student prediction consequently gives
\(\hat x_{t+h}=\phi^\star(z_{t+h})\) and
\(\hat z^{\mathrm{pH}}_{t+h}=z_{t+h}\).
Both residuals in \eqref{eq:Ltraj} vanish.
\end{proof}

\noindent\textbf{Relation to the learned teacher field.}
The true field in the encoder coordinates is
\(f_\eta(z)=D\eta(\eta^{-1}(z))f_\star(\eta^{-1}(z))\).
The theorem gives \(D\phi^\star f_\eta=f_{\mathrm{pH}}^\star\circ\phi^\star\); it does not replace \(f_\eta\) by the learned teacher field \(f_T\).
The theorem therefore transfers an existing pH realization to the learned coordinates without requiring exact teacher forecasts or teacher tangents.
Proposition~\ref{prop:predictive_sufficiency} provides one sufficient route to the state-chart property; its predictive hypotheses are not additional requirements of the bridge theorem.

\noindent\textbf{Observed-future realization.}
Under the theorem's assumptions, encoded observed futures lie on trajectories of \(f_\eta\).
Fitting these targets asks the student to realize the dynamics in the learned representation, without imposing the learned teacher's forecast errors as targets.
The bridge supplies an existence justification for this second stage; whether a fitted pH student improves on the teacher or an unconstrained refit is assessed empirically.

The open domain \(X=\kappa(U)\) has the same dimension as \(U\) and \(Z\).
The pH constraint applies to the vector field on this domain and its energy, interconnection, and dissipation representation; it does not select a lower-dimensional subset of state points.
Existence of \(\kappa\) means that the physical vector field admits the required realization throughout \(U\), rather than that a single observed trajectory can be fitted by some pH system.

The theorem also separates smooth existence from finite-network representability.
The prescribed coupling architecture, chart orientation, and Hamiltonian network must represent or approximate the required maps on the modeled region.
For a globally defined coupling flow, exact representability includes an admissible extension of \(\phi^\star\) beyond \(Z\).
Neither global Euclidean coordinates for every physical phase space nor recovery of uniquely defined physical energy parameters is implied.
An affine identification \(\eta(s)=Ls+\beta\), when independently justified, reduces this construction to the previous affine bridge if \(\kappa\) is affine.

\subsubsection{Trajectory fitting, approximation, and numerical integration}
\label{app:realization_approximation}

The exact bridge is an existence result conditional on state recovery and a compatible pH realization.
The following bound relates chart, field, and integration errors to the positional losses, without assuming that finite-data optimization attains the ideal realization.

\begin{proposition}[Finite-horizon realization error]
\label{prop:realization_approximation}
Fix \(0\le u\le U_{\max}\) and an encoded initial-state set \(K\subset Z\).
Let \(z^\star(u;z)=\eta(\varphi_\star^u(\eta^{-1}(z)))\), with the physical trajectories contained in \(U\).
Let \(\phi\) be a candidate diffeomorphism with
\(\|\phi-\phi^\star\|\le\varepsilon_\phi\) along all these encoded trajectories.
Assume exact solutions of the candidate field \(f_{\mathrm{pH}}\) initialized at \(\phi(K)\) exist throughout this time interval.
Suppose \(f_{\mathrm{pH}}\) is \(L\)-Lipschitz on a region containing these solutions and the ideal trajectories
\(\phi^\star(z^\star(u;z))\), and
\(\|f_{\mathrm{pH}}-f_{\mathrm{pH}}^\star\|\le\varepsilon_f\) along the latter.
Let a numerical prediction \(\hat x(u;z)\) have uniform error at most \(\varepsilon_{\mathrm{num}}(u)\) relative to the exact candidate trajectory.
Assume \(\phi^{-1}\) is \(L_{\phi^{-1}}\)-Lipschitz on a comparison region containing \(\hat x(u;z)\) and \(\phi(z^\star(u;z))\).
Then
\begin{equation}
\begin{aligned}
\|\hat x(u;z)-\phi(z^\star(u;z))\|&\le B(u),\\
\|\phi^{-1}(\hat x(u;z))-z^\star(u;z)\|
&\le L_{\phi^{-1}}B(u),\\
B(u)&=(e^{Lu}+1)\varepsilon_\phi
      +b_L(u)\varepsilon_f+\varepsilon_{\mathrm{num}}(u),
\end{aligned}
\label{eq:realization_error_bound}
\end{equation}
where \(b_L(u)=(e^{Lu}-1)/L\) for \(L>0\), and \(b_0(u)=u\).
\end{proposition}

\begin{proof}
The candidate and ideal initial states differ by at most \(\varepsilon_\phi\).
Their exact trajectories \(x(u)\) and \(x^\star(u)\) satisfy
\[
\|x(u)-x^\star(u)\|
\le e^{Lu}\varepsilon_\phi+b_L(u)\varepsilon_f
\]
by adding and subtracting \(f_{\mathrm{pH}}(x^\star(u))\) and applying Gronwall's inequality.
The ideal trajectory differs from \(\phi(z^\star(u;z))\) by at most \(\varepsilon_\phi\).
Adding this terminal chart error and the numerical error gives the first bound.
Applying the Lipschitz inequality for \(\phi^{-1}\) to the two comparison points gives the second.
\end{proof}

For \(u=h\Delta\), \(\hat x(u;z)=\Phi_{pH}^h(\phi(z))\), using the numerical student-rollout notation of Section~\ref{subsec:trajectory_matching}.
If the bounds are uniform on the anchor distribution, then for \(\lambda_x\ge0\),
\[
\mathcal L_z+\lambda_x\mathcal L_x
\le (L_{\phi^{-1}}^2+\lambda_x)
\sum_{h\in\mathcal H_S} B(h\Delta)^2.
\]
This bound controls the positional terms.
The chart penalty \(\mathcal L_\phi\) and the implemented matrix-residual penalty may remain positive at an exact realization, so zero positional loss need not minimize the full regularized objective.

For RK4 with internal step \(\delta\), assume sufficient smoothness and a common compact neighborhood containing all relevant exact trajectories, numerical iterates, and intermediate stages.
A uniform local error \(C_{\mathrm{loc}}\delta^5\) and step stability factor \(1+C_{\mathrm{stab}}\delta\) give
\(e_{k+1}\le(1+C_{\mathrm{stab}}\delta)e_k+C_{\mathrm{loc}}\delta^5\).
Summing for \(k\delta\le U_{\max}\) yields
\(\varepsilon_{\mathrm{num}}(k\delta)\le C_{U_{\max}}\delta^4\) for sufficiently small \(\delta\).
The implemented step is \(\delta=\Delta\).
The argument compares exact and numerical solutions in student coordinates and then applies the inverse chart.
It neither assumes nonlinear equivariance of RK4 nor guarantees preservation of the continuous learned-energy inequality by the numerical rollout.
Likewise, sampled trajectory agreement alone does not identify a unique continuous-time generator.

\newpage
\section{Experimental Details}
\label{app:exp_details}

This appendix describes the datasets, model configurations, and evaluation protocol for Section~\ref{sec:experiments}. The main section combines the experimental setup with results organized around Q1--Q3.

\subsection{Systems and datasets}
\label{app:systems_datasets}

All trajectories are sampled at \(\Delta=0.05\). The six noiseless-observation tasks and trajectory lengths are listed in Table~\ref{app:datasets}; four noisy three-body variants complete the benchmark. Datasets are generated once with a fixed data-generation seed and partitioned by trajectory; the training seeds do not regenerate or repartition the data. Each task has \(1000\) training, \(125\) validation, and \(250\) test trajectories, except the three-mass chain, which has \(2000/250/500\). The video and numerical pendulum datasets use the same underlying trajectories and split assignments. Figure \ref{fig:dataset} shows example of the trajectories.

Reference trajectories are generated with adaptive solvers and then sampled at the observation times, rather than with the models' single fixed RK4 step per sampling interval.

\begin{table}[!htbp]
\centering
\small
\setlength{\tabcolsep}{5pt}
\begin{tabular}{llrrl}
\toprule
System & Observation & Dimension \(d_s\) & Timesteps \(T\) & Property \\
\midrule
Pendulum, numeric & Angle & 2 & 150 & Conservative \\
Pendulum, video & Image & 2 & 150 & Conservative \\
Duffing, video & Image & 2 & 150 & Dissipative \\
Three-mass chain & Body positions & 6 & 140 & Dissipative \\
Three-body, planar & Body positions & 8 & 150 & Conservative \\
Three-body, 3-D & Body positions & 12 & 300 & Conservative \\
\bottomrule
\end{tabular}
\caption{The six noiseless-observation datasets. Here \(d_s\) is the physical state-vector dimension; \(T\) is the trajectory length. Each three-body dataset also has variants with position-based Gaussian noise with standard deviation \(\sigma_{\mathrm{obs}}=0.01\) and \(0.05\), giving ten datasets in total. Samples are separated by \(\Delta=0.05\); reference physical states remain unchanged.}
\label{app:datasets}
\end{table}

\paragraph{Pendulum.}
The conservative pendulum has unit mass and length, with \(g=9.81\), state \(s=(q,p)\), and equations \(\dot q=p\), \(\dot p=-g\sin q\). Its energy is \(E^\star(q,p)=p^2/2+g(1-\cos q)\). Initial energies are stratified into ten equal-width bins over \([0.2,19.12]\), below the separatrix energy \(19.62\), and sampled uniformly within each bin. An initial phase is drawn uniformly in time over the corresponding libration period, obtained from the elliptic-integral expression; the initial state is reached by integrating from the turning point. All trajectories are librational. RK45 uses relative tolerance \(10^{-9}\) and absolute tolerance \(10^{-11}\), with \(150\) sampled states. Numerical observations contain only \(q\). For video, the renderer is set to each solved state and produces \(84\times84\) color frames without advancing the dynamics itself.

\paragraph{Duffing oscillator.}
The double-well oscillator has \(\dot q=p\) and \(\dot p=-\alpha q-\beta q^3-\delta p\), with \(\alpha=-1\), \(\beta=1\), and \(\delta=0.3\). Its energy is \(E^\star=p^2/2+\alpha q^2/2+\beta q^4/4\), so \(\dot E^\star=-\delta p^2\). Initial conditions comprise \(20\%\) in each well and \(60\%\) in a cross-well population combining near-saddle and high-momentum states. Trajectories have \(150\) samples, using the same integration tolerances as the pendulum. The observed bead is rendered at four times the final resolution and downsampled with a Lanczos filter to \(84\times84\) color frames.

\paragraph{Three-mass chain.}
The wall-anchored chain has mass matrix \(M=\operatorname{diag}(1.0,1.2,0.8)\). Its four springs, ordered from wall to wall, have displacements \(r=(q_1,q_2-q_1,q_3-q_2,-q_3)\), linear coefficients \((1.0,0.8,1.1,0.9)\), and quartic coefficients \((0.08,0.05,0.07,0.06)\). With system Hamiltonian \(E^\star\), the equations are \(\dot q=M^{-1}p\) and \(\dot p=-\nabla_q E^\star-\Gamma M^{-1}p\), where \(\Gamma=\operatorname{diag}(0.08,0.10,0.06)\). Thus the pH dissipation matrix has a zero position block and momentum block \(\Gamma\), acting on \(\nabla_p E^\star=M^{-1}p\). Only positions are observed.

Candidate initial positions and velocities are drawn independently, componentwise, from \(q_i\sim\mathcal U(-1.1,1.1)\) and \(v_i\sim\mathcal U(-1,1)\), with \(p=Mv\). Rejection sampling balances counts across eight initial-energy bins spanning \([0.4,3.0]\). Accepted trajectories satisfy \(|q_i|\le1.2\) during the first \(0.5\) time units and have non-increasing energy at every sampled step to tolerance \(10^{-6}\). Trajectories have \(140\) samples and split sizes \(2000/250/500\).

\paragraph{Three-body systems.}
Both tasks simulate a hierarchical star and two planets with gravitational constant \(G=1\), no physical softening, and a \(10^{-12}\) denominator guard. Reference trajectories use DOP853 with relative tolerance \(10^{-10}\) and absolute tolerance \(10^{-12}\). Initial semi-major axes are sampled uniformly from \([0.5,0.65]\) and \([1.0,1.3]\); eccentricities are at most \(0.15\), with uniform true anomaly and argument of periapsis and an equal-probability choice of orbital sense. States are shifted to the centre-of-mass frame, and trajectories with a minimum pairwise separation below \(0.15\) are resampled.

The planar masses are \((1,0.03,0.02)\), with \(150\) samples and an eight-coordinate reduced state for the two planets. The three-dimensional masses are \((1,0.06,0.04)\), with \(300\) samples and a twelve-coordinate reduced state. The latter uses mutual inclinations sampled uniformly from \(10\)--\(40^\circ\), a random global rotation, and a fixed rotation of each complete trajectory into an invariable frame in which the conserved total angular momentum points along the vertical axis. Such a fixed rotation preserves the Newtonian equations. Observations contain all three body positions; the star moves according to the centre-of-mass constraint and is not held stationary.

\begin{figure}
    \centering
    \includegraphics[width=0.9\linewidth]{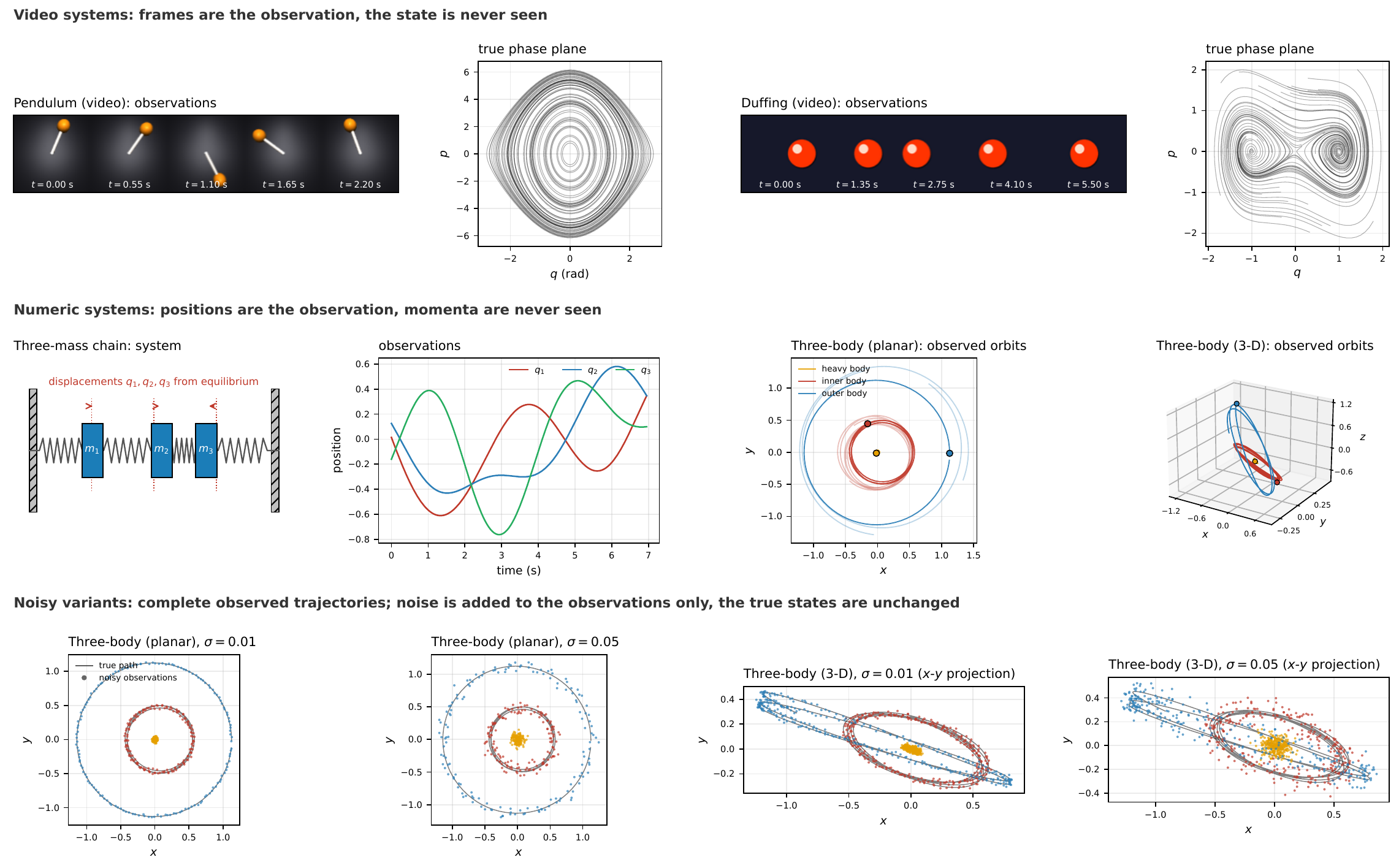}
    \caption{4 systems example states and trajectories (pendulum, duffing, three-mass-chain, and two three-body variants).}
    \label{fig:dataset}
\end{figure}

The evaluated state is \(s=(q_1,q_2,p_1,p_2)\), using the two planets' barycentric positions and physical momenta \(p_i=m_i v_i\). The eliminated stellar state is reconstructed from \(q_0=-(m_1q_1+m_2q_2)/m_0\) and \(p_0=-(p_1+p_2)\). These are reduced evaluation coordinates; they are not assumed to be canonical coordinates.

The listed dimensions count state coordinates after eliminating the centre-of-mass variables. They need not equal the intrinsic dimension of the sampled support: in particular, aligning angular momentum with the vertical axis constrains its two transverse components. The dimension sweep therefore uses the state-vector size as a reference, rather than establishing a minimal embedding dimension.

\paragraph{Observation noise and signal scales.}
\label{app:observation_noise}
Each three-body task has two noisy variants. Noise is i.i.d. Gaussian across frames, bodies, and observed position coordinates, and is applied to the training, validation, and test observations: \(o_t=g(s_t)+\epsilon_t\), with \(\epsilon_t\sim\mathcal N(0,\sigma_{\mathrm{obs}}^2 I)\) and \(\sigma_{\mathrm{obs}}\in\{0.01,0.05\}\). The underlying trajectories and reference states are unchanged. Table~\ref{tab:noise_scales} records the training-split scales in the position units used for the noise.

\begin{table}[!htbp]
\centering
\small
\setlength{\tabcolsep}{7pt}
\begin{tabular}{lcc}
\toprule
Training-split quantity & Planar & 3-D \\
\midrule
Observation standard deviation & 0.52 & 0.45 \\
Observation 5th--95th percentile span & 1.92 & 1.56 \\
Median orbital radius, inner / outer planet & 0.56 / 1.11 & 0.54 / 1.11 \\
Median stellar reflex radius & 0.027 & 0.053 \\
Median per-frame displacement, inner / outer planet & 0.065 / 0.047 & 0.064 / 0.047 \\
5th percentile planet--planet separation & 0.55 & 0.57 \\
\bottomrule
\end{tabular}
\caption{Geometric and temporal scales of the three-body training data at \(\Delta=0.05\). Radius, displacement, and separation describe spatial magnitudes, whereas \(\sigma_{\mathrm{obs}}\) is a per-coordinate noise standard deviation.}
\label{tab:noise_scales}
\end{table}

Relative to these scales, \(\sigma_{\mathrm{obs}}=0.01\) is approximately \(1.9\)--\(2.2\%\) of the observation standard deviation, \(0.9\)--\(1.9\%\) of the median planetary radii, and \(15\)--\(21\%\) of the median planetary step displacement. At \(0.05\), the corresponding ranges are \(9.6\)--\(11.1\%\), \(4.5\)--\(9.3\%\), and \(77\)--\(106\%\). Using \(20\log_{10}(\sigma_{\mathrm{signal}}/\sigma_{\mathrm{obs}})\) with the observation standard deviations in the table gives approximately \(33\)--\(34\) dB and \(19\)--\(20\) dB, respectively. The lower noise level is below the median stellar reflex radius in both systems; the higher level is comparable to or larger than that radius. Typical orbital radii and interplanetary separations remain substantially larger than either per-coordinate noise scale.

For a position observation, adjacent-frame differencing gives \(\hat v_t=(o_t-o_{t-1})/\Delta\). Its noise contribution is \((\epsilon_t-\epsilon_{t-1})/\Delta\), with per-coordinate variance \(2\sigma_{\mathrm{obs}}^2/\Delta^2\). Thus its noise standard deviation is \(\sqrt{2}\sigma_{\mathrm{obs}}/\Delta\), approximately \(0.283\) or \(1.414\) at the two levels. These figures concern measurement noise and exclude finite-difference approximation error. The position-noise-to-displacement ratios above are therefore not relative velocity-error estimates. The larger noise level makes two-frame velocity estimates substantially noisier and motivates aggregation over the observation window; it does not eliminate all velocity information from an observation pair.

\subsection{Compared methods}
\label{app:compared_methods}

All comparisons use the same pre-split files. Physical states are withheld from model training except for the explicitly labeled full-state DCL\(^{*}\) reference. In the main comparisons, the known physical dimension \(d_s\) determines the teacher and student ODE state dimensions; it should not be confused with the internal hidden width of a recurrent sequence model.

\paragraph{Teacher and one-stage model.}
The teacher is the contrastive encoder--Neural ODE model of Section~\ref{sec:method}. Direct \textsc{CL+pHNN} uses the same encoder architecture, contrastive objective, and state dimension, replacing the unconstrained ODE with a pH field and learning both components jointly for \(30{,}000\) steps. Its interconnection matrix is canonical in two dimensions and a learned constant skew matrix in higher dimensions. It has approximately \(1.8\)M trainable parameters. Thus it differs from the primary state-dependent student in field parameterization as well as objective and training schedule.

\paragraph{Student ablations.}
CIPHER is the primary student, with a learned invertible chart and state-dependent interconnection and dissipation matrices, fitted to encoded observed futures. Table~\ref{tab:student_ablations} uses descriptive names for five variants:
\begin{itemize}
    \item \textbf{Fixed-coordinate CIPHER:} state-dependent pH dynamics with the encoder-to-student chart fixed to \(\phi=\mathrm{id}\).
    \item \textbf{CIPHER (constant \(J,R\)):} pH dynamics with constant \(J,R\), a learned chart, and encoded observed-future targets.
    \item \textbf{Fixed-coordinate CIPHER (constant \(J,R\)):} constant-\(J,R\) pH dynamics fitted to encoded observed futures with \(\phi=\mathrm{id}\), completing the coordinate-by-matrix-parameterization comparison.
    \item \textbf{Unconstrained Neural ODE:} an unconstrained neural vector field fitted to encoded observed futures in place of the pH field.
    \item \textbf{Teacher-rollout supervision:} fitting targets are the teacher's rollout states \(\Phi_T^h(z_t)\), instead of encodings of observed future windows.
\end{itemize}
The learned-chart field comparisons share the chart, step size, six supervised horizons, three trajectory/chart loss weights, anchor count, optimizer, schedule, training budget, and parent teacher checkpoint for each seed. Their configurations differ in field parameterization and the associated matrix-residual regularization. The coordinate and target-source ablations make the additional changes specified above. Trajectory losses use raw squared error without division by latent variance; implementation details are given in Appendix~\ref{app:impl}.

\paragraph{DCL variants.}
We evaluate DCL \citep{laiz2025dyncl} with instantaneous observations on numerical datasets and consecutive frozen-VAE codes on video datasets. DCL-pair uses consecutive observation pairs \([o_{t-1},o_t]\) in both settings, while DCL\(^{*}\) receives the full physical state. The video variants use the same frozen VAE features as our teacher. Consequently, DCL and DCL-pair coincide on both pendulum video and Duffing video; their entries in Table~\ref{tab:results} reuse the same five trained checkpoints and evaluation results.

The pair provides temporal-change information without simulator velocities or momenta: for fixed \(\Delta>0\), \([o_{t-1},o_t]\) is equivalent to \([o_t,(o_t-o_{t-1})/\Delta]\). For position measurements, the latter component approximates a velocity; for nonlinear observation features it is an observation-space change. DCL\(^{*}\) alone receives the full physical state and is a privileged-input reference, not an upper bound on achievable performance.

\paragraph{DCL transition and training.}
All three DCL input variants in this benchmark use a \emph{single learned linear latent transition}. The encoder is a three-layer, width-\(64\) GELU MLP, with latent dimension \(D=d_s\). The transition matrix \(A\in\mathbb R^{D\times D}\) has no bias and is initialized to the identity. The joint objective contrasts \(Az_t\) against the next encoded state using negative squared distance at temperature \(1\). The denominator contains the positive and \(8192\) negatives sampled uniformly in time; anchor times are also uniform. Both encoder and transition are optimized with Adam at learning rate \(10^{-3}\), for \(15{,}000\) iterations with batch size \(512\), without a learning-rate schedule or gradient clipping. These configurations have approximately \(9\)--\(10\)k trainable parameters.

At evaluation, the autonomous forecast is \(\hat z_{t+h}=A^h z_t\): the matrix is applied repeatedly, without re-encoding generated observations or fitting an additional transition. DCL's switching extension selects among several matrices by comparing each candidate prediction with the next encoded observation \citep[Eq.~(6)]{laiz2025dyncl}. That observation is unavailable during autonomous generation, so applying the switching extension would require an additional causal mode-selection rule or another rollout adaptation. We therefore report the single-mode configuration, whose fifty-step latent rollout is a linear system. This limitation concerns the configuration evaluated here, not every possible autonomous switching model.

\paragraph{Mamba-3.}
Mamba-3 \citep{lahoti2026mamba3} uses four blocks, model width \(128\), state size \(32\), expansion factor \(2\), and head dimension \(32\), with approximately \(450\)k parameters. It processes independent \(10\)-frame windows and exports a \(128\)-dimensional representation of the last window position. A separate linear head predicts the observation at each supervised horizon. Training averages the corresponding squared errors, with horizons \(\{1,3,5,10\}\) in the default configuration and \(\{1,2,4,8,16,32\}\) in the matched-horizon configuration. Horizon matching does not match the encoder context length.

Optimization uses AdamW at learning rate \(3\times10^{-4}\), weight decay \(10^{-4}\), a cosine schedule, batch size \(32\), gradient clipping at \(1\), and \(200\) epochs. A checkpoint is selected on validation data. Autonomous evaluation appends the next predicted observation to the window, removes its oldest entry, and re-encodes the resulting synthetic window. This is repeated for \(50\) steps; no observed future is supplied. The exported representation at each horizon is obtained from that generated window. Video inputs use the shared frozen VAE features.

\subsection{PHAST configuration and rollout modes}
\label{app:phast}

We evaluate PHAST in its UNKNOWN, configuration-only regime \citep{bhardwaj2026phast}. The main table uses the enlarged configuration: potential and observer width \(256\), potential depth \(4\), observer receptive field \(31\) frames, dissipation-factor rank \(4\), and approximately \(0.8\)M parameters. For comparison, the default configuration uses potential width \(64\) and depth \(3\), observer receptive field \(7\), rank \(2\), and approximately \(13\)--\(16\)k parameters. We report the enlarged configuration explicitly rather than implying equal capacity across all methods.

The observer combines a finite-difference velocity estimate with a causal dilated-convolution correction whose output layer is initialized to zero. The canonicalizer is the identity, so the initial momentum-like input is an estimated velocity, not a supplied physical momentum. The mass matrix is constant positive definite with diagonal-plus-low-rank parameterization; damping is a state-dependent positive-semidefinite low-rank matrix. Integration uses a Strang-split leapfrog update with one substep per observation interval and a learnable step initialized to \(\Delta=0.05\).

Training minimizes a position-only, teacher-forced next-step squared error, with rollout-loss weight zero. AdamW uses learning rate \(10^{-3}\), weight decay \(10^{-5}\), cosine decay, batches of \(64\) trajectories, and gradient clipping at \(5\). The enlarged configuration trains for \(300\) epochs (\(50\) for the default), and checkpoint selection uses validation data.

Both forecast modes use the same selected checkpoint and \(21\) observed frames at the anchor. Self-conditioned forecasting re-applies the observer to the generated configuration history at each step. Takeover forecasting infers the initial phase state once and integrates without further observer updates. Neither receives observations after the anchor. Table~\ref{tab:phast} reports both modes; the main comparison uses self-conditioned forecasts. Negative scores indicate poor accuracy but do not by themselves establish numerical divergence.

For video, the configuration input is the shared frozen one-dimensional VAE code, and the observer differences that code. This feature is not a physical coordinate. Its reported linear-probe \(R^2\) for configuration recovery is \(0.99\) on pendulum and \(0.88\) on Duffing.

\begingroup
\newcommand{\phastscore}[2]{\ensuremath{#1{\scriptscriptstyle\,\pm\,#2}}}
\begin{table}[t]
\centering
\caption{\textbf{PHAST rollout modes under the evaluation protocol.} Each task stacks \(R^2_{\mathrm{Dyn}}\) at 30 and 50 steps and AUC over horizons 1--50. Entries are mean \(\pm\) standard deviation over five runs; \(<0\) denotes a negative score. CIPHER is included for reference. The configuration and rollout modes are described in Appendix~\ref{app:phast}.}
\label{tab:phast}
\small
\setlength{\tabcolsep}{5pt}
\renewcommand{\arraystretch}{1.06}
\begin{tabular}{@{}llccc@{}}
\toprule
\multirow{2}{*}{\textbf{Task}} & \multirow{2}{*}{\shortstack{\textbf{Metric}\\(\(R^2_{\mathrm{Dyn}}\))}} & \multicolumn{2}{c}{\textbf{PHAST}} & \multirow{2}{*}{\textbf{CIPHER}} \\
\cmidrule(lr){3-4}
 & & Self-conditioned & Takeover & \\
\midrule
\multirow{3}{*}{\shortstack[l]{Pendulum\\(numeric)}} & @30 & \phastscore{0.972}{0.002} & $<0$ & \phastscore{0.977}{0.012} \\
 & @50 & \phastscore{0.970}{0.004} & $<0$ & \phastscore{0.974}{0.002} \\
 & AUC & \phastscore{0.972}{0.002} & $<0$ & \phastscore{0.984}{0.006} \\
\midrule
\multirow{3}{*}{\shortstack[l]{Pendulum\\(video)}} & @30 & \phastscore{0.972}{0.001} & $<0$ & \phastscore{0.981}{0.011} \\
 & @50 & \phastscore{0.972}{0.001} & $<0$ & \phastscore{0.977}{0.003} \\
 & AUC & \phastscore{0.972}{0.001} & $<0$ & \phastscore{0.986}{0.006} \\
\midrule
\multirow{3}{*}{\shortstack[l]{Duffing\\(video)}} & @30 & \phastscore{0.856}{0.031} & $<0$ & \phastscore{0.971}{0.010} \\
 & @50 & \phastscore{0.826}{0.044} & $<0$ & \phastscore{0.939}{0.031} \\
 & AUC & \phastscore{0.864}{0.030} & $<0$ & \phastscore{0.975}{0.009} \\
\midrule
\multirow{3}{*}{\shortstack[l]{Three-mass\\chain}} & @30 & \phastscore{0.994}{0.000} & $<0$ & \phastscore{0.995}{0.001} \\
 & @50 & \phastscore{0.994}{0.000} & $<0$ & \phastscore{0.995}{0.001} \\
 & AUC & \phastscore{0.994}{0.000} & $<0$ & \phastscore{0.996}{0.000} \\
\midrule
\multirow{3}{*}{\shortstack[l]{Three-body\\(planar)}} & @30 & $<0$ & $<0$ & \phastscore{0.920}{0.067} \\
 & @50 & $<0$ & $<0$ & \phastscore{0.895}{0.092} \\
 & AUC & $<0$ & $<0$ & \phastscore{0.926}{0.061} \\
\midrule
\multirow{3}{*}{\shortstack[l]{Three-body\\(3-D)}} & @30 & $<0$ & $<0$ & \phastscore{0.900}{0.075} \\
 & @50 & $<0$ & $<0$ & \phastscore{0.871}{0.109} \\
 & AUC & $<0$ & $<0$ & \phastscore{0.907}{0.067} \\
\midrule
\multirow{3}{*}{\shortstack[l]{Three-body\\(planar)\\$\sigma=0.01$}} & @30 & \phastscore{0.943}{0.002} & $<0$ & \phastscore{0.927}{0.063} \\
 & @50 & \phastscore{0.892}{0.003} & $<0$ & \phastscore{0.899}{0.086} \\
 & AUC & \phastscore{0.953}{0.001} & $<0$ & \phastscore{0.934}{0.058} \\
\midrule
\multirow{3}{*}{\shortstack[l]{Three-body\\(planar)\\$\sigma=0.05$}} & @30 & \phastscore{0.808}{0.006} & $<0$ & \phastscore{0.879}{0.033} \\
 & @50 & \phastscore{0.605}{0.021} & $<0$ & \phastscore{0.813}{0.049} \\
 & AUC & \phastscore{0.814}{0.006} & $<0$ & \phastscore{0.895}{0.031} \\
\midrule
\multirow{3}{*}{\shortstack[l]{Three-body\\(3-D)\\$\sigma=0.01$}} & @30 & \phastscore{0.575}{0.449} & $<0$ & \phastscore{0.904}{0.060} \\
 & @50 & $<0$ & $<0$ & \phastscore{0.873}{0.090} \\
 & AUC & $<0$ & $<0$ & \phastscore{0.911}{0.054} \\
\midrule
\multirow{3}{*}{\shortstack[l]{Three-body\\(3-D)\\$\sigma=0.05$}} & @30 & \phastscore{0.737}{0.007} & $<0$ & \phastscore{0.865}{0.026} \\
 & @50 & \phastscore{0.460}{0.019} & $<0$ & \phastscore{0.801}{0.039} \\
 & AUC & \phastscore{0.750}{0.007} & $<0$ & \phastscore{0.878}{0.024} \\
\bottomrule
\end{tabular}
\end{table}
\endgroup

\subsection{Forecast horizons under observation noise}
\label{app:noise_horizons}

Figures~\ref{fig:horizon_planar} and~\ref{fig:horizon_3d} resolve the three-body forecasting scores by horizon. Each noisy variant perturbs only the observations of the corresponding clean trajectories; the physical evaluation targets, test anchors, and coordinate-wise reference variances are unchanged. Consequently, adding observation noise does not change the score denominator in these matched evaluations. On clean observations, Mamba-3 and self-conditioned PHAST start with accurate short-horizon predictions but degrade sharply as the rollout length increases. Noise-trained versions can sustain useful predictions for longer, although the benefit is not uniform: on three-body 3-D at \(\sigma_{\mathrm{obs}}=0.01\), their mean scores still become negative before the final horizon. DCL\(^{*}\) and PHAST takeover do not share the same improvement pattern.

This behavior is consistent with observation noise acting as a training regularizer for autoregressive prediction. Mamba-3 and self-conditioned PHAST repeatedly process generated inputs, so small forecast errors change the inputs used at subsequent steps. Training on perturbed histories may make this feedback less sensitive to those errors. The latent ODE models instead evolve from an encoded anchor without reapplying an observation encoder during the forecast, and their curves decay more gradually. They remain sensitive to observation noise through their learned representations and dynamics; their scores are not invariant across noise levels.

\paragraph{Checkpoint cross-evaluation.}
Additional checkpoint-level checks evaluate noise-trained models on clean observations and clean-trained models on noisy observations with the same scorer. Improved forecasts persist for noise-trained Mamba-3 and self-conditioned PHAST on the clean test set, while adding test noise to clean-trained checkpoints does not produce a comparable gain. These checks support a training-dependent effect consistent with the regularization explanation. They do not isolate that mechanism from other effects of training on noisy data.

\begin{figure}[t]
\centering
\includegraphics[width=\linewidth]{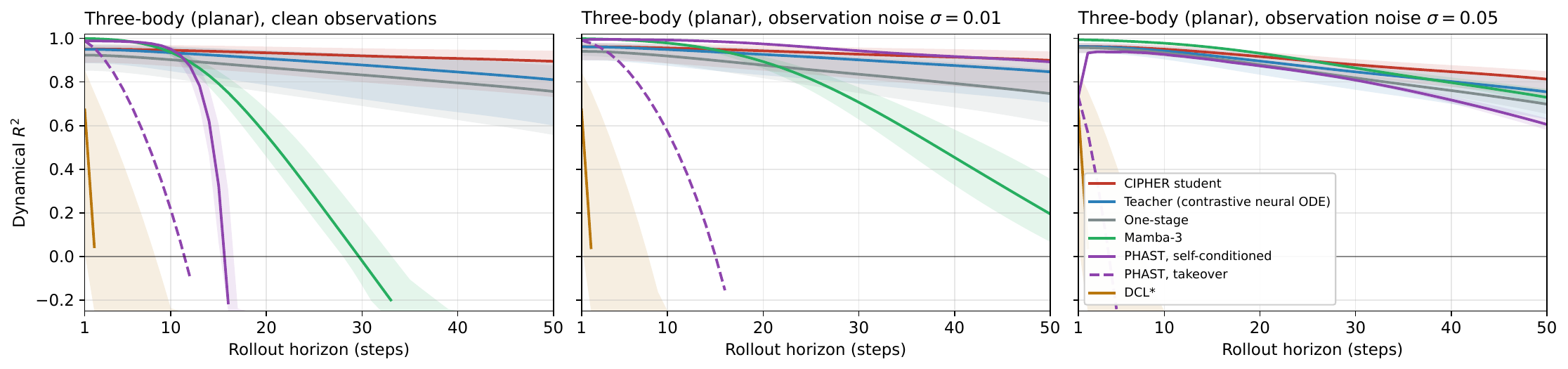}
\caption{\textbf{Planar three-body forecasts across horizons.} Dynamical \(R^2_{\mathrm{Dyn}}\) for selected models on clean observations and the two noisy variants of the same trajectories. 
The reference states and score normalization are identical across panels. Noise-trained Mamba-3 and self-conditioned PHAST degrade more gradually than their clean-trained counterparts. The vertical plotting range is truncated, and the horizontal line marks zero score.}
\label{fig:horizon_planar}
\end{figure}

\begin{figure}[t]
\centering
\includegraphics[width=\linewidth]{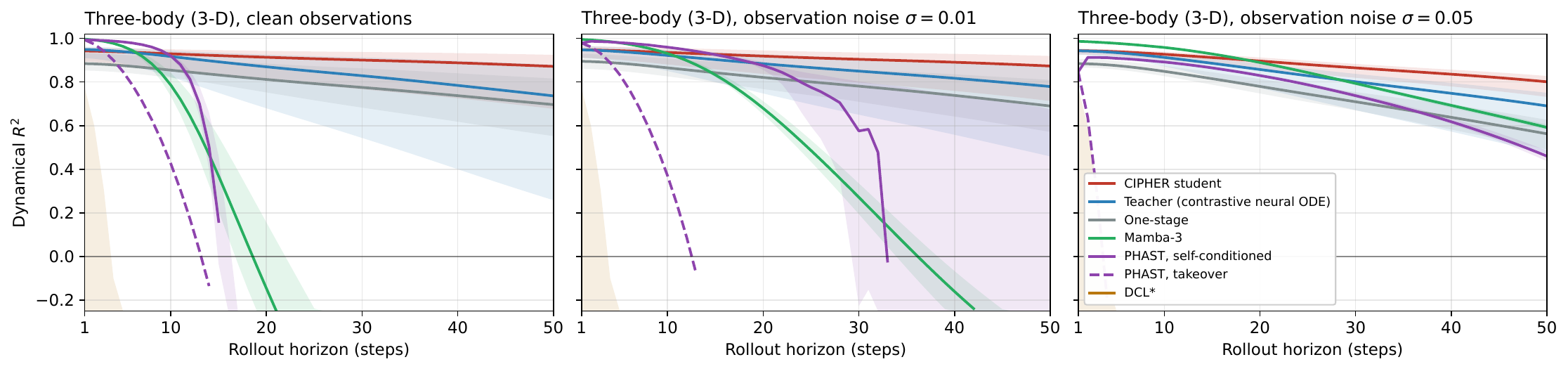}
\caption{\textbf{Three-body 3-D forecasts across horizons.} The same comparison as Figure~\ref{fig:horizon_planar}, with clean physical evaluation targets in every panel. 
The lower noise level delays the sharp degradation of Mamba-3 and self-conditioned PHAST, while their mean scores still become negative at long horizons; at the higher level their decline is more gradual. Negative scores extending below the plotting range remain negative results in Table~\ref{tab:results}.}
\label{fig:horizon_3d}
\end{figure}

\subsection{One-stage and two-stage training with constant matrices}
\label{app:staged_constant_comparison}

Table~\ref{tab:staged_constant_comparison} compares one-stage CL+pHNN with CIPHER (constant \(J,R\)) using the same datasets and evaluation protocol. Both use constant interconnection and dissipation matrices. One-stage learns the encoder and pH field jointly by contrastive prediction; CIPHER first learns the encoder with an unconstrained teacher, then fits a chart and pH field to frozen encoded observed futures. Matching the matrix class excludes state-dependent flexibility as an explanation for this comparison. The comparison evaluates the complete staged procedure, including its observed-future objective, learned chart, and additional optimization phase; it does not isolate encoder freezing alone. One-stage training with state-dependent matrices is not evaluated here.

\begingroup
\newcommand{\stagescore}[2]{\ensuremath{#1{\scriptscriptstyle\,\pm\,#2}}}
\newcommand{\stagebest}[2]{\ensuremath{\mathbf{#1}{\scriptscriptstyle\,\pm\,#2}}}
\begin{table}[t]
\centering
\caption{\textbf{One-stage versus two-stage training with constant \(J,R\).} Each task reports \(R^2_{\mathrm{Dyn}}(30)\), \(R^2_{\mathrm{Dyn}}(50)\), and \(\mathrm{AUC}_{1:50}\), as mean \(\pm\) sample standard deviation over five runs. Bold marks the higher reported mean within each row. \(<0\) abbreviates a negative mean. CIPHER (constant \(J,R\)) uses a learned chart and observed-future targets; its scores also appear in Tables~\ref{tab:results} and~\ref{tab:student_ablations}.}
\label{tab:staged_constant_comparison}
\small
\setlength{\tabcolsep}{9pt}
\renewcommand{\arraystretch}{1.1}
\begin{tabular}{@{}llcc@{}}
\toprule
\textbf{Task} & \shortstack{\textbf{Metric}\\($R^2_{\mathrm{Dyn}}$)} & \shortstack{\textbf{One-stage}\\CL+pHNN (constant $J,R$)} & \shortstack{\textbf{CIPHER}\\(constant $J,R$)} \\
\midrule
\multirow{3}{*}{\shortstack[l]{Pendulum\\(numeric)}} & @30 & \stagebest{0.996}{0.002} & \stagescore{0.977}{0.012} \\
 & @50 & \stagebest{0.994}{0.003} & \stagescore{0.974}{0.002} \\
 & AUC & \stagebest{0.997}{0.001} & \stagescore{0.984}{0.007} \\
\midrule
\multirow{3}{*}{\shortstack[l]{Pendulum\\(video)}} & @30 & $<0$ & \stagebest{0.979}{0.012} \\
 & @50 & $<0$ & \stagebest{0.977}{0.004} \\
 & AUC & $<0$ & \stagebest{0.985}{0.006} \\
\midrule
\multirow{3}{*}{\shortstack[l]{Duffing\\(video)}} & @30 & \stagescore{0.904}{0.110} & \stagebest{0.972}{0.009} \\
 & @50 & \stagescore{0.912}{0.085} & \stagebest{0.943}{0.029} \\
 & AUC & \stagescore{0.911}{0.102} & \stagebest{0.976}{0.008} \\
\midrule
\multirow{3}{*}{\shortstack[l]{Three-mass\\chain}} & @30 & \stagescore{0.984}{0.001} & \stagebest{0.992}{0.001} \\
 & @50 & \stagescore{0.974}{0.003} & \stagebest{0.991}{0.002} \\
 & AUC & \stagescore{0.986}{0.001} & \stagebest{0.994}{0.001} \\
\midrule
\multirow{3}{*}{\shortstack[l]{Three-body\\(planar)}} & @30 & \stagescore{0.832}{0.098} & \stagebest{0.899}{0.069} \\
 & @50 & \stagescore{0.756}{0.131} & \stagebest{0.857}{0.094} \\
 & AUC & \stagescore{0.846}{0.089} & \stagebest{0.908}{0.063} \\
\midrule
\multirow{3}{*}{\shortstack[l]{Three-body\\(3-D)}} & @30 & \stagescore{0.775}{0.071} & \stagebest{0.869}{0.108} \\
 & @50 & \stagescore{0.696}{0.103} & \stagebest{0.819}{0.163} \\
 & AUC & \stagescore{0.792}{0.062} & \stagebest{0.881}{0.095} \\
\midrule
\multirow{3}{*}{\shortstack[l]{Three-body\\(planar)\\$\sigma=0.01$}} & @30 & \stagescore{0.835}{0.065} & \stagebest{0.912}{0.070} \\
 & @50 & \stagescore{0.747}{0.101} & \stagebest{0.871}{0.094} \\
 & AUC & \stagescore{0.853}{0.058} & \stagebest{0.920}{0.063} \\
\midrule
\multirow{3}{*}{\shortstack[l]{Three-body\\(planar)\\$\sigma=0.05$}} & @30 & \stagescore{0.818}{0.013} & \stagebest{0.845}{0.060} \\
 & @50 & \stagescore{0.698}{0.030} & \stagebest{0.752}{0.089} \\
 & AUC & \stagescore{0.842}{0.012} & \stagebest{0.866}{0.053} \\
\midrule
\multirow{3}{*}{\shortstack[l]{Three-body\\(3-D)\\$\sigma=0.01$}} & @30 & \stagescore{0.781}{0.067} & \stagebest{0.881}{0.072} \\
 & @50 & \stagescore{0.690}{0.104} & \stagebest{0.832}{0.108} \\
 & AUC & \stagescore{0.798}{0.059} & \stagebest{0.891}{0.063} \\
\midrule
\multirow{3}{*}{\shortstack[l]{Three-body\\(3-D)\\$\sigma=0.05$}} & @30 & \stagescore{0.709}{0.038} & \stagebest{0.831}{0.047} \\
 & @50 & \stagescore{0.562}{0.056} & \stagebest{0.742}{0.075} \\
 & AUC & \stagescore{0.738}{0.034} & \stagebest{0.850}{0.041} \\
\bottomrule
\end{tabular}
\end{table}
\endgroup

\subsection{Student-configuration ablation}
\label{app:student_configuration_ablation}

Table~\ref{tab:student_ablations} reports forecasting results for CIPHER and the five ablation variants defined above on all ten datasets, with five runs per configuration. All scores use the same physical-coordinate evaluation as Table~\ref{tab:results}; shared training settings and the changes defining each variant are specified in Appendix~\ref{app:compared_methods}.
\paragraph{Realization and field parameterization.}
The Unconstrained Neural ODE retains the frozen representations, learned chart, observed-future supervision, and optimization settings of the realization stage.
Its improvements over the teacher demonstrate a benefit of this second-stage procedure that is shared across field choices.
Comparing CIPHER with this student assesses the additional benefit of the pH parameterization and its associated regularization.
CIPHER achieves higher reported mean \(@50\) scores on the three-mass chain and all six three-body conditions, while the oscillator comparisons are closely matched.
The component table and the parameter-matched dimension sweep therefore evaluate complementary aspects of the proposed realization stage.

\paragraph{Coordinate--matrix interaction.}
For each task, let \(r_{\mathrm{const,learn}}\) and \(r_{\mathrm{const,fixed}}\) be the mean \(@50\) scores of the two constant-matrix configurations, and define \(r_{\mathrm{state,learn}}\) and \(r_{\mathrm{state,fixed}}\) analogously. The chart benefits are \(\Delta_{\mathrm{const}}=r_{\mathrm{const,learn}}-r_{\mathrm{const,fixed}}\) and \(\Delta_{\mathrm{state}}=r_{\mathrm{state,learn}}-r_{\mathrm{state,fixed}}\). On all ten factorial comparisons, \(\Delta_{\mathrm{const}}-\Delta_{\mathrm{state}}\) is positive and ranges from \(0.050\) to \(0.795\), using the rounded means in the table. State-dependent chart benefits are at most \(0.003\) on nine datasets, but reach \(0.437\) on Duffing video, compared with \(0.597\) for constant matrices on that task. This descriptive interaction measures how the chart benefit changes with the matrix parameterization. Its uncertainty requires the corresponding within-seed differences; marginal standard deviations do not determine it. These results are consistent with Corollary~\ref{cor:chart_absorption}, without establishing exact equivalence of the trained fields or a unique latent coordinate transformation.

\paragraph{Observed futures versus teacher targets.}
Teacher-rollout targets yield stronger oscillator forecasts, where the teacher is already accurate. Encoded observed futures perform better on the chain and three-body systems, enabling improvements beyond the teacher's own forecasts.

\begingroup
\newcommand{\ablscore}[2]{\ensuremath{#1{\scriptscriptstyle\,\pm\,#2}}}
\newcommand{\ablbest}[2]{\ensuremath{\mathbf{#1}{\scriptscriptstyle\,\pm\,#2}}}
\newcommand{\ablsecond}[2]{\ensuremath{\underline{#1}{\scriptscriptstyle\,\pm\,#2}}}
\begin{table*}[t]
\centering
\caption{\textbf{Student ablations.} Each task stacks $R^2_{\mathrm{Dyn}}$ at 30 and 50 steps and its mean over horizons 1--50 (AUC; higher is better). Entries are mean $\pm$ sample standard deviation over five runs. The grouped headers pair each learned-chart model with its identity-chart counterpart: CIPHER with Fixed-coordinate CIPHER, and CIPHER (constant \(J,R\)) with Fixed-coordinate CIPHER (constant \(J,R\)). Together they form a $2\times2$ comparison of coordinate learning and matrix parameterization. Bold and underline mark the highest and second-highest distinct reported means in each row; ties at the displayed precision share a rank.}
\label{tab:student_configuration_ablation}
\label{tab:student_ablations}
\scriptsize
\setlength{\tabcolsep}{3pt}
\renewcommand{\arraystretch}{1.14}
\resizebox{\textwidth}{!}{%
\begin{tabular}{@{}llcccccc@{}}
\toprule
\multicolumn{2}{c}{} & \multicolumn{2}{c}{\textbf{State-dependent} $J(x),R(x)$} & \multicolumn{2}{c}{\textbf{Constant-matrix} $J,R$} & \multicolumn{2}{c}{} \\
\cmidrule(lr){3-4}\cmidrule(lr){5-6}
\textbf{Task} & \shortstack{\textbf{Metric}\\($R^2_{\mathrm{Dyn}}$)} & \textbf{CIPHER} & \shortstack{\textbf{Fixed-coordinate}\\\textbf{CIPHER}} & \shortstack{\textbf{CIPHER}\\(constant $J,R$)} & \shortstack{\textbf{Fixed-coordinate}\\\textbf{CIPHER}\\(constant $J,R$)} & \shortstack{\textbf{Unconstrained}\\\textbf{Neural ODE}} & \shortstack{\textbf{Teacher-rollout}\\\textbf{supervision}} \\
\midrule
\multirow{3}{*}{\shortstack[l]{Pendulum\\(numeric)}} & @30 & \ablsecond{0.977}{0.012} & \ablsecond{0.977}{0.012} & \ablsecond{0.977}{0.012} & \ablscore{0.823}{0.278} & \ablsecond{0.977}{0.012} & \ablbest{0.999}{0.000} \\
 & @50 & \ablsecond{0.974}{0.002} & \ablsecond{0.974}{0.002} & \ablsecond{0.974}{0.002} & \ablscore{0.744}{0.448} & \ablscore{0.973}{0.002} & \ablbest{0.999}{0.000} \\
 & AUC & \ablsecond{0.984}{0.006} & \ablsecond{0.984}{0.006} & \ablsecond{0.984}{0.007} & \ablscore{0.839}{0.276} & \ablscore{0.983}{0.006} & \ablbest{0.999}{0.000} \\
\midrule
\multirow{3}{*}{\shortstack[l]{Pendulum\\(video)}} & @30 & \ablsecond{0.981}{0.011} & \ablscore{0.980}{0.012} & \ablscore{0.979}{0.012} & \ablscore{0.077}{0.485} & \ablscore{0.980}{0.012} & \ablbest{0.998}{0.002} \\
 & @50 & \ablsecond{0.977}{0.003} & \ablsecond{0.977}{0.004} & \ablsecond{0.977}{0.004} & \ablscore{0.182}{0.438} & \ablsecond{0.977}{0.004} & \ablbest{0.997}{0.003} \\
 & AUC & \ablsecond{0.986}{0.006} & \ablsecond{0.986}{0.007} & \ablscore{0.985}{0.006} & \ablscore{0.295}{0.375} & \ablsecond{0.986}{0.007} & \ablbest{0.998}{0.001} \\
\midrule
\multirow{3}{*}{\shortstack[l]{Duffing\\(video)}} & @30 & \ablscore{0.971}{0.010} & \ablscore{0.825}{0.170} & \ablsecond{0.972}{0.009} & \ablscore{0.688}{0.017} & \ablscore{0.971}{0.009} & \ablbest{0.997}{0.001} \\
 & @50 & \ablscore{0.939}{0.031} & \ablscore{0.502}{0.482} & \ablsecond{0.943}{0.029} & \ablscore{0.346}{0.050} & \ablscore{0.939}{0.030} & \ablbest{0.996}{0.001} \\
 & AUC & \ablscore{0.975}{0.009} & \ablscore{0.831}{0.162} & \ablsecond{0.976}{0.008} & \ablscore{0.729}{0.013} & \ablscore{0.975}{0.008} & \ablbest{0.998}{0.000} \\
\midrule
\multirow{3}{*}{\shortstack[l]{Three-mass\\chain}} & @30 & \ablbest{0.995}{0.001} & \ablbest{0.995}{0.000} & \ablscore{0.992}{0.001} & \ablscore{0.918}{0.079} & \ablsecond{0.993}{0.001} & \ablscore{0.987}{0.001} \\
 & @50 & \ablbest{0.995}{0.001} & \ablbest{0.995}{0.001} & \ablscore{0.991}{0.002} & \ablscore{0.820}{0.189} & \ablsecond{0.993}{0.001} & \ablscore{0.977}{0.003} \\
 & AUC & \ablbest{0.996}{0.000} & \ablbest{0.996}{0.000} & \ablscore{0.994}{0.001} & \ablscore{0.926}{0.073} & \ablsecond{0.995}{0.000} & \ablscore{0.989}{0.001} \\
\midrule
\multirow{3}{*}{\shortstack[l]{Three-body\\(planar)}} & @30 & \ablbest{0.920}{0.067} & \ablsecond{0.919}{0.068} & \ablscore{0.899}{0.069} & \ablscore{0.861}{0.101} & \ablscore{0.915}{0.070} & \ablscore{0.882}{0.081} \\
 & @50 & \ablbest{0.895}{0.092} & \ablsecond{0.894}{0.094} & \ablscore{0.857}{0.094} & \ablscore{0.783}{0.140} & \ablscore{0.890}{0.092} & \ablscore{0.822}{0.105} \\
 & AUC & \ablbest{0.926}{0.061} & \ablsecond{0.925}{0.062} & \ablscore{0.908}{0.063} & \ablscore{0.875}{0.089} & \ablscore{0.922}{0.063} & \ablscore{0.893}{0.071} \\
\midrule
\multirow{3}{*}{\shortstack[l]{Three-body\\(3-D)}} & @30 & \ablbest{0.900}{0.075} & \ablsecond{0.899}{0.075} & \ablscore{0.869}{0.108} & \ablscore{0.793}{0.205} & \ablscore{0.887}{0.087} & \ablscore{0.826}{0.167} \\
 & @50 & \ablbest{0.871}{0.109} & \ablsecond{0.868}{0.110} & \ablscore{0.819}{0.163} & \ablscore{0.687}{0.324} & \ablscore{0.852}{0.126} & \ablscore{0.740}{0.271} \\
 & AUC & \ablbest{0.907}{0.067} & \ablsecond{0.906}{0.068} & \ablscore{0.881}{0.095} & \ablscore{0.818}{0.172} & \ablscore{0.896}{0.077} & \ablscore{0.844}{0.144} \\
\midrule
\multirow{3}{*}{\shortstack[l]{Three-body\\(planar)\\$\sigma=0.01$}} & @30 & \ablbest{0.927}{0.063} & \ablbest{0.927}{0.063} & \ablscore{0.912}{0.070} & \ablscore{0.884}{0.090} & \ablsecond{0.924}{0.065} & \ablscore{0.902}{0.061} \\
 & @50 & \ablbest{0.899}{0.086} & \ablsecond{0.898}{0.086} & \ablscore{0.871}{0.094} & \ablscore{0.820}{0.124} & \ablscore{0.896}{0.089} & \ablscore{0.851}{0.080} \\
 & AUC & \ablbest{0.934}{0.058} & \ablsecond{0.933}{0.058} & \ablscore{0.920}{0.063} & \ablscore{0.896}{0.080} & \ablscore{0.931}{0.060} & \ablscore{0.912}{0.057} \\
\midrule
\multirow{3}{*}{\shortstack[l]{Three-body\\(planar)\\$\sigma=0.05$}} & @30 & \ablsecond{0.879}{0.033} & \ablbest{0.880}{0.028} & \ablscore{0.845}{0.060} & \ablscore{0.779}{0.098} & \ablscore{0.859}{0.052} & \ablscore{0.846}{0.049} \\
 & @50 & \ablbest{0.813}{0.049} & \ablbest{0.813}{0.041} & \ablscore{0.752}{0.089} & \ablscore{0.641}{0.135} & \ablsecond{0.777}{0.078} & \ablscore{0.755}{0.074} \\
 & AUC & \ablsecond{0.895}{0.031} & \ablbest{0.896}{0.027} & \ablscore{0.866}{0.053} & \ablscore{0.813}{0.079} & \ablscore{0.878}{0.047} & \ablscore{0.868}{0.044} \\
\midrule
\multirow{3}{*}{\shortstack[l]{Three-body\\(3-D)\\$\sigma=0.01$}} & @30 & \ablbest{0.904}{0.060} & \ablbest{0.904}{0.059} & \ablscore{0.881}{0.072} & \ablscore{0.807}{0.166} & \ablsecond{0.892}{0.069} & \ablscore{0.850}{0.114} \\
 & @50 & \ablbest{0.873}{0.090} & \ablsecond{0.872}{0.087} & \ablscore{0.832}{0.108} & \ablscore{0.709}{0.260} & \ablscore{0.854}{0.105} & \ablscore{0.779}{0.184} \\
 & AUC & \ablbest{0.911}{0.054} & \ablbest{0.911}{0.052} & \ablscore{0.891}{0.063} & \ablscore{0.831}{0.139} & \ablsecond{0.901}{0.061} & \ablscore{0.866}{0.097} \\
\midrule
\multirow{3}{*}{\shortstack[l]{Three-body\\(3-D)\\$\sigma=0.05$}} & @30 & \ablbest{0.865}{0.026} & \ablsecond{0.864}{0.024} & \ablscore{0.831}{0.047} & \ablscore{0.751}{0.099} & \ablscore{0.851}{0.035} & \ablscore{0.802}{0.072} \\
 & @50 & \ablbest{0.801}{0.039} & \ablsecond{0.800}{0.035} & \ablscore{0.742}{0.075} & \ablscore{0.614}{0.129} & \ablscore{0.778}{0.057} & \ablscore{0.694}{0.121} \\
 & AUC & \ablbest{0.878}{0.024} & \ablbest{0.878}{0.022} & \ablscore{0.850}{0.041} & \ablscore{0.784}{0.081} & \ablsecond{0.867}{0.032} & \ablscore{0.827}{0.062} \\
\bottomrule
\end{tabular}%
}
\end{table*}
\endgroup

\subsection{Latent-dimension sensitivity}
\label{app:dimension_sensitivity}

Table~\ref{tab:dimsweep} gives the full forecasting scores underlying the dimension study in Section~\ref{subsec:results_ablations}. Each latent dimension is evaluated with three runs, separately from the five-run main comparison. CIPHER and the Unconstrained Neural ODE use the same teacher encodings, chart architecture, observed-future objective, and paired seeds. At each \(D\), the unconstrained field's hidden width is chosen to match the pH field's parameter count; the widths are \(231\) for \(D=4\) and \(800\) for \(D=50\). The comparison assesses the structured field and its regularization under this parameter-count matching. The ambient state-vector sizes serve as reference dimensions; invariable-frame constraints mean they need not equal the intrinsic dimension of the sampled support.

For Figure~\ref{fig:dimsweep_relative}, let \(a_{m,k}(D)\) denote the \(\mathrm{AUC}_{1:50}\) score of method \(m\) in run \(k\) at dimension \(D\), and let \(a_{\mathrm{free},k}(D)\) denote its Unconstrained Neural ODE reference. Each plotted reduction is \(g_{m,k}(D)=100[1-(1-a_{m,k}(D))/(1-a_{\mathrm{free},k}(D))]\), with positive reference error. The CIPHER and reference students in each pair share a teacher. The connected markers average these three paired reductions, rather than taking a ratio of the averaged scores; black bars show medians and whiskers the observed minimum and maximum. Thus the whiskers describe the range across runs, not a confidence interval. Table~\ref{tab:dimsweep} reports the mean of the three seed-paired percentage reductions for both AUC and \(R^2_{\mathrm{Dyn}}(50)\), using the corresponding score in the same formula. These reductions use unrounded scores and cannot be reconstructed from the rounded table entries. A positive mean paired reduction need not imply a higher mean score, because it normalizes each run by that run's reference error before averaging.

Both second-stage students retain accurate forecasts at many larger latent dimensions even as the teacher's forecasts deteriorate, supporting the observed-future realization stage.
CIPHER additionally achieves positive mean seed-paired AUC-error reductions relative to the parameter-matched unconstrained student at every tested \(D>d_s\) on all three systems.
This metric-specific finding supports an additional predictive benefit from the structured parameterization and its regularization; the corresponding \(@50\) ordering is not uniform across dimensions.

\begingroup
\begin{table}[t]
\centering
\caption{\textbf{Latent-dimension sensitivity with a parameter-matched unconstrained field.} The teacher recipe is fixed except for latent dimension \(D\). The Unconstrained Neural ODE uses a field parameter count matched to CIPHER at each \(D\). Scores are mean \(\pm\) sample standard deviation over three paired runs. Bold scores mark the physical state-vector size \(D=d_s\), not the best score or a certified minimal embedding dimension. AUC is the mean \(R^2_{\mathrm{Dyn}}\) over horizons \(1\)--\(50\). Matched unconstrained-field scores are reported for \(@50\) and AUC only. Error-reduction rows report mean seed-paired percentage reductions in \(1-\text{score}\) relative to the unconstrained field; positive values favor CIPHER. Signed zero percentages preserve the sign before rounding. Scores of \(1.000\pm0.000\) are rounded, so large relative reductions can accompany small absolute differences. Figure~\ref{fig:dimsweep_relative} displays the AUC reductions for the two three-body tasks.}
\label{tab:dimsweep}
\scriptsize
\setlength{\tabcolsep}{2pt}
\renewcommand{\arraystretch}{1.1}
\resizebox{\textwidth}{!}{%
\begin{tabular}{@{}llccccccccccc@{}}
\toprule
\textbf{Model} & \textbf{Metric} & \multicolumn{11}{c}{\textbf{Latent dimension \(D\)}} \\
\cmidrule(lr){3-13}
 & \(R^2_{\mathrm{Dyn}}\) & 4 & 6 & 8 & 12 & 16 & 20 & 24 & 28 & 36 & 40 & 50 \\
\midrule
\multicolumn{13}{l}{\textit{Three-body (planar), true $d_s = 8$}} \\
\multirow{3}{*}{Teacher} & @30 & $0.596_{\pm 0.080}$ & $0.636_{\pm 0.096}$ & $\mathbf{0.919}_{\pm 0.012}$ & $0.895_{\pm 0.029}$ & $0.729_{\pm 0.220}$ & $0.686_{\pm 0.158}$ & $0.586_{\pm 0.210}$ & $0.261_{\pm 0.144}$ & $-0.342_{\pm 0.677}$ & $-1.41_{\pm 1.46}$ & $-0.861_{\pm 1.278}$ \\
 & @50 & $0.519_{\pm 0.122}$ & $0.570_{\pm 0.113}$ & $\mathbf{0.868}_{\pm 0.028}$ & $0.790_{\pm 0.070}$ & $0.459_{\pm 0.459}$ & $0.342_{\pm 0.367}$ & $0.252_{\pm 0.355}$ & $-0.441_{\pm 0.222}$ & $-1.68_{\pm 1.23}$ & $-4.02_{\pm 2.36}$ & $-2.88_{\pm 2.52}$ \\
 & AUC & $0.612_{\pm 0.070}$ & $0.648_{\pm 0.092}$ & $\mathbf{0.924}_{\pm 0.011}$ & $0.905_{\pm 0.027}$ & $0.762_{\pm 0.194}$ & $0.725_{\pm 0.141}$ & $0.654_{\pm 0.172}$ & $0.350_{\pm 0.111}$ & $-0.157_{\pm 0.563}$ & $-1.10_{\pm 1.13}$ & $-0.652_{\pm 1.103}$ \\
\addlinespace[2pt]
\multirow{3}{*}{CIPHER} & @30 & $0.574_{\pm 0.091}$ & $0.661_{\pm 0.085}$ & $\mathbf{0.949}_{\pm 0.005}$ & $0.972_{\pm 0.006}$ & $0.963_{\pm 0.035}$ & $0.984_{\pm 0.004}$ & $0.982_{\pm 0.008}$ & $0.974_{\pm 0.024}$ & $0.960_{\pm 0.042}$ & $0.924_{\pm 0.067}$ & $0.935_{\pm 0.052}$ \\
 & @50 & $0.452_{\pm 0.156}$ & $0.609_{\pm 0.091}$ & $\mathbf{0.935}_{\pm 0.009}$ & $0.962_{\pm 0.008}$ & $0.945_{\pm 0.056}$ & $0.976_{\pm 0.004}$ & $0.974_{\pm 0.011}$ & $0.961_{\pm 0.036}$ & $0.940_{\pm 0.065}$ & $0.886_{\pm 0.093}$ & $0.896_{\pm 0.082}$ \\
 & AUC & $0.589_{\pm 0.083}$ & $0.667_{\pm 0.083}$ & $\mathbf{0.953}_{\pm 0.005}$ & $0.975_{\pm 0.005}$ & $0.969_{\pm 0.029}$ & $0.986_{\pm 0.003}$ & $0.984_{\pm 0.007}$ & $0.978_{\pm 0.019}$ & $0.966_{\pm 0.034}$ & $0.934_{\pm 0.059}$ & $0.944_{\pm 0.044}$ \\
\addlinespace[2pt]
\multirow{2}{*}{\shortstack[l]{Unconstrained\\Neural ODE}} & @50 & $0.496_{\pm 0.078}$ & $0.598_{\pm 0.098}$ & $\mathbf{0.938}_{\pm 0.009}$ & $0.962_{\pm 0.009}$ & $0.946_{\pm 0.053}$ & $0.976_{\pm 0.003}$ & $0.973_{\pm 0.011}$ & $0.961_{\pm 0.035}$ & $0.941_{\pm 0.062}$ & $0.883_{\pm 0.101}$ & $0.897_{\pm 0.083}$ \\
 & AUC & $0.606_{\pm 0.051}$ & $0.664_{\pm 0.086}$ & $\mathbf{0.954}_{\pm 0.005}$ & $0.974_{\pm 0.006}$ & $0.968_{\pm 0.028}$ & $0.985_{\pm 0.003}$ & $0.983_{\pm 0.007}$ & $0.977_{\pm 0.019}$ & $0.966_{\pm 0.034}$ & $0.928_{\pm 0.068}$ & $0.942_{\pm 0.047}$ \\
\addlinespace[2pt]
\multirow{2}{*}{\shortstack[l]{CIPHER error\\reduction}} & @50 & $-8\%$ & $+3\%$ & $-5\%$ & $+1\%$ & $+3\%$ & $+3\%$ & $+5\%$ & $+2\%$ & $-0\%$ & $+0\%$ & $-1\%$ \\
 & AUC & $-4\%$ & $+1\%$ & $-2\%$ & $+4\%$ & $+6\%$ & $+8\%$ & $+8\%$ & $+6\%$ & $+5\%$ & $+5\%$ & $+4\%$ \\
\midrule
\multicolumn{13}{l}{\textit{Three-mass chain, true $d_s = 6$}} \\
\multirow{3}{*}{Teacher} & @30 & $0.602_{\pm 0.028}$ & $\mathbf{0.986}_{\pm 0.002}$ & $0.974_{\pm 0.010}$ & $0.991_{\pm 0.003}$ & $0.953_{\pm 0.006}$ & $0.921_{\pm 0.021}$ & $0.841_{\pm 0.125}$ & $0.884_{\pm 0.039}$ & $0.901_{\pm 0.003}$ & $0.860_{\pm 0.014}$ & $0.249_{\pm 0.892}$ \\
 & @50 & $0.588_{\pm 0.048}$ & $\mathbf{0.975}_{\pm 0.005}$ & $0.946_{\pm 0.011}$ & $0.981_{\pm 0.003}$ & $0.866_{\pm 0.015}$ & $0.800_{\pm 0.054}$ & $0.574_{\pm 0.284}$ & $0.698_{\pm 0.163}$ & $0.746_{\pm 0.029}$ & $0.611_{\pm 0.082}$ & $-0.493_{\pm 1.556}$ \\
 & AUC & $0.603_{\pm 0.029}$ & $\mathbf{0.988}_{\pm 0.002}$ & $0.977_{\pm 0.007}$ & $0.992_{\pm 0.002}$ & $0.954_{\pm 0.006}$ & $0.927_{\pm 0.019}$ & $0.850_{\pm 0.112}$ & $0.891_{\pm 0.044}$ & $0.908_{\pm 0.004}$ & $0.867_{\pm 0.019}$ & $0.365_{\pm 0.725}$ \\
\addlinespace[2pt]
\multirow{3}{*}{CIPHER} & @30 & $0.595_{\pm 0.033}$ & $\mathbf{0.995}_{\pm 0.000}$ & $0.998_{\pm 0.003}$ & $1.000_{\pm 0.000}$ & $1.000_{\pm 0.000}$ & $1.000_{\pm 0.000}$ & $1.000_{\pm 0.000}$ & $1.000_{\pm 0.000}$ & $1.000_{\pm 0.000}$ & $1.000_{\pm 0.000}$ & $0.998_{\pm 0.003}$ \\
 & @50 & $0.581_{\pm 0.055}$ & $\mathbf{0.995}_{\pm 0.000}$ & $0.998_{\pm 0.003}$ & $1.000_{\pm 0.000}$ & $1.000_{\pm 0.000}$ & $1.000_{\pm 0.000}$ & $1.000_{\pm 0.000}$ & $1.000_{\pm 0.000}$ & $1.000_{\pm 0.000}$ & $1.000_{\pm 0.000}$ & $0.998_{\pm 0.003}$ \\
 & AUC & $0.598_{\pm 0.033}$ & $\mathbf{0.996}_{\pm 0.000}$ & $0.998_{\pm 0.003}$ & $1.000_{\pm 0.000}$ & $1.000_{\pm 0.000}$ & $1.000_{\pm 0.000}$ & $1.000_{\pm 0.000}$ & $1.000_{\pm 0.000}$ & $1.000_{\pm 0.000}$ & $1.000_{\pm 0.000}$ & $0.998_{\pm 0.002}$ \\
\addlinespace[2pt]
\multirow{2}{*}{\shortstack[l]{Unconstrained\\Neural ODE}} & @50 & $0.582_{\pm 0.054}$ & $\mathbf{0.994}_{\pm 0.000}$ & $0.997_{\pm 0.003}$ & $0.999_{\pm 0.000}$ & $0.999_{\pm 0.000}$ & $1.000_{\pm 0.000}$ & $0.999_{\pm 0.000}$ & $1.000_{\pm 0.000}$ & $1.000_{\pm 0.000}$ & $1.000_{\pm 0.000}$ & $0.997_{\pm 0.004}$ \\
 & AUC & $0.598_{\pm 0.033}$ & $\mathbf{0.996}_{\pm 0.000}$ & $0.998_{\pm 0.003}$ & $1.000_{\pm 0.000}$ & $1.000_{\pm 0.000}$ & $1.000_{\pm 0.000}$ & $1.000_{\pm 0.000}$ & $1.000_{\pm 0.000}$ & $1.000_{\pm 0.000}$ & $1.000_{\pm 0.000}$ & $0.998_{\pm 0.003}$ \\
\addlinespace[2pt]
\multirow{2}{*}{\shortstack[l]{CIPHER error\\reduction}} & @50 & $-0\%$ & $+12\%$ & $+22\%$ & $+14\%$ & $+13\%$ & $+5\%$ & $+12\%$ & $+13\%$ & $+7\%$ & $+9\%$ & $+16\%$ \\
 & AUC & $-0\%$ & $+10\%$ & $+19\%$ & $+9\%$ & $+7\%$ & $+1\%$ & $+5\%$ & $+6\%$ & $+3\%$ & $+4\%$ & $+12\%$ \\
\midrule
\multicolumn{13}{l}{\textit{Three-body (3-D), true $d_s = 12$}} \\
\multirow{3}{*}{Teacher} & @30 & $0.291_{\pm 0.114}$ & $0.618_{\pm 0.009}$ & $0.655_{\pm 0.045}$ & $\mathbf{0.899}_{\pm 0.019}$ & $0.836_{\pm 0.032}$ & $0.807_{\pm 0.041}$ & $0.766_{\pm 0.075}$ & $0.768_{\pm 0.068}$ & $0.749_{\pm 0.039}$ & $0.685_{\pm 0.064}$ & $0.572_{\pm 0.082}$ \\
 & @50 & $0.247_{\pm 0.161}$ & $0.571_{\pm 0.020}$ & $0.532_{\pm 0.047}$ & $\mathbf{0.856}_{\pm 0.037}$ & $0.741_{\pm 0.042}$ & $0.697_{\pm 0.080}$ & $0.579_{\pm 0.109}$ & $0.587_{\pm 0.158}$ & $0.554_{\pm 0.100}$ & $0.457_{\pm 0.092}$ & $0.218_{\pm 0.214}$ \\
 & AUC & $0.304_{\pm 0.096}$ & $0.625_{\pm 0.010}$ & $0.668_{\pm 0.048}$ & $\mathbf{0.908}_{\pm 0.017}$ & $0.858_{\pm 0.026}$ & $0.837_{\pm 0.036}$ & $0.796_{\pm 0.060}$ & $0.800_{\pm 0.062}$ & $0.785_{\pm 0.036}$ & $0.735_{\pm 0.049}$ & $0.638_{\pm 0.076}$ \\
\addlinespace[2pt]
\multirow{3}{*}{CIPHER} & @30 & $0.298_{\pm 0.103}$ & $0.647_{\pm 0.016}$ & $0.713_{\pm 0.061}$ & $\mathbf{0.933}_{\pm 0.005}$ & $0.942_{\pm 0.009}$ & $0.948_{\pm 0.008}$ & $0.951_{\pm 0.007}$ & $0.953_{\pm 0.003}$ & $0.957_{\pm 0.000}$ & $0.954_{\pm 0.003}$ & $0.953_{\pm 0.002}$ \\
 & @50 & $0.251_{\pm 0.152}$ & $0.630_{\pm 0.015}$ & $0.666_{\pm 0.070}$ & $\mathbf{0.920}_{\pm 0.006}$ & $0.928_{\pm 0.010}$ & $0.933_{\pm 0.009}$ & $0.938_{\pm 0.008}$ & $0.940_{\pm 0.003}$ & $0.944_{\pm 0.002}$ & $0.942_{\pm 0.004}$ & $0.941_{\pm 0.001}$ \\
 & AUC & $0.309_{\pm 0.089}$ & $0.651_{\pm 0.016}$ & $0.721_{\pm 0.061}$ & $\mathbf{0.936}_{\pm 0.005}$ & $0.945_{\pm 0.009}$ & $0.951_{\pm 0.007}$ & $0.954_{\pm 0.006}$ & $0.957_{\pm 0.002}$ & $0.961_{\pm 0.001}$ & $0.959_{\pm 0.002}$ & $0.959_{\pm 0.001}$ \\
\addlinespace[2pt]
\multirow{2}{*}{\shortstack[l]{Unconstrained\\Neural ODE}} & @50 & $0.258_{\pm 0.145}$ & $0.630_{\pm 0.015}$ & $0.676_{\pm 0.065}$ & $\mathbf{0.916}_{\pm 0.004}$ & $0.925_{\pm 0.011}$ & $0.932_{\pm 0.010}$ & $0.936_{\pm 0.008}$ & $0.940_{\pm 0.003}$ & $0.944_{\pm 0.002}$ & $0.942_{\pm 0.004}$ & $0.940_{\pm 0.002}$ \\
 & AUC & $0.310_{\pm 0.089}$ & $0.651_{\pm 0.017}$ & $0.725_{\pm 0.059}$ & $\mathbf{0.934}_{\pm 0.004}$ & $0.942_{\pm 0.010}$ & $0.949_{\pm 0.008}$ & $0.952_{\pm 0.007}$ & $0.956_{\pm 0.002}$ & $0.959_{\pm 0.001}$ & $0.958_{\pm 0.003}$ & $0.957_{\pm 0.002}$ \\
\addlinespace[2pt]
\multirow{2}{*}{\shortstack[l]{CIPHER error\\reduction}} & @50 & $-1\%$ & $+0\%$ & $-3\%$ & $+4\%$ & $+4\%$ & $+2\%$ & $+2\%$ & $-0\%$ & $+0\%$ & $-1\%$ & $+1\%$ \\
 & AUC & $-0\%$ & $+0\%$ & $-1\%$ & $+4\%$ & $+4\%$ & $+4\%$ & $+4\%$ & $+3\%$ & $+3\%$ & $+3\%$ & $+4\%$ \\
\bottomrule
\end{tabular}%
}
\end{table}
\endgroup

\subsection{Architecture and optimization}
\label{app:impl}
\label{app:arch}

\paragraph{Teacher encoder and flow.}
The causal input contains \(W=21\) observations, so \(w=20\) in \(o_{t-w:t}\). Each numeric frame passes through a linear map to \(128\) dimensions with SiLU activation. For video, the same front end follows the posterior mean of a frozen one-dimensional VAE encoder, pretrained on training-split frames resized to \(32\times32\) with \(\beta=0.5\). The window encoder is a two-layer bidirectional GRU of hidden width \(256\), with dropout \(0.1\), layer normalization, and a linear head to \(D\) dimensions. Bidirectionality is confined to the past window and uses no observations after \(t\). Latents with \(t<20\) are excluded from all training losses.

The teacher field is a three-hidden-layer SiLU MLP of width \(128\), with a zero-initialized output layer. One numerical transition is exactly one RK4 step of size \(\Delta=0.05\). Predictive InfoNCE averages uniformly over \(\mathcal H_T=\{1,2,4,8,16,32\}\). An EMA copy of the complete encoding path, with retention coefficient \(0.999\), supplies stop-gradient targets. Logits are negative squared distances divided by \(\tau=0.01\) and the detached mean per-coordinate variance of batch targets. Their values are invariant to a common nonzero rescaling of queries and keys when the variance is recomputed accordingly.

Negatives comprise batch target latents and \(2048\) draws from a FIFO bank of \(131{,}072\) past targets. Same-trajectory candidates are retained only for offsets \(5<|\Delta t|\le20\) from the positive, measured in sample indices. This excludes immediate neighbors and restricts distant recurrent states; it does not guarantee that every retained negative is a distinct physical state. The teacher has approximately \(1.82\)M trainable parameters, excluding the frozen video VAE.

\paragraph{CIPHER chart and scalar function.}
The learned-chart student variants share the same chart \(\phi\): an element-wise invertible affine layer (called ActNorm in the implementation) followed by four affine coupling blocks with alternating masks. Each coupling network has two width-\(128\) SiLU hidden layers and a zero-initialized output layer. The affine layer and coupling blocks initialize to the identity. A fixed orientation reflection is selected automatically from the teacher's circulation on pendulum video and Duffing video, and is disabled on higher-dimensional tasks; this policy is shared by the state-dependent and constant-matrix variants. Thus identity initialization refers to the trainable chart before any fixed reflection. Coupling log-scales use \(\tanh\) and are bounded in magnitude by \(1\); the inverse is analytic. This bound controls the multiplicative scale within each block, not the complete chart's Jacobian norm.

The scalar function is \(H_\theta(x)=\tfrac12x^\top Qx+\mathrm{MLP}_H(x)\), with \(Q=Q^\top\) initialized to \(\tfrac12 I\) and two width-\(128\) SiLU hidden layers in \(\mathrm{MLP}_H\). Its gradient is obtained by automatic differentiation. Symmetry of \(Q\) alone does not ensure positive definiteness or global lower boundedness of \(H_\theta\).

\paragraph{State-dependent interconnection and dissipation.}
The primary model uses residual matrix networks:
\begin{equation}
J(x)=J_0+\tfrac12\big(K_\theta(x)-K_\theta(x)^\top\big),
\qquad
R(x)=\big(L_0+L_\theta(x)\big)\big(L_0+L_\theta(x)\big)^\top.
\label{eq:impl_residual_jr}
\end{equation}
Here \(J_0\) is canonical for \(D=2\), and a learned constant skew matrix initialized to the canonical matrix for higher even dimensions. Each of \(K_\theta,L_\theta\) maps \(x\in\mathbb R^D\) through two width-\(128\) SiLU hidden layers to \(D^2\) outputs, reshaped to a \(D\times D\) matrix. Both output layers are zero-initialized, and \(L_0\) is initialized at scale \(10^{-2}\). Only the skew part of \(K_\theta\) enters \(J\), while \(L_\theta\) modifies the factor of \(R\). Thus the initial field is a constant-matrix pH field, while state dependence is learned through the residuals. This is not a zero-field initialization.

For every parameter value, \(J(x)^\top=-J(x)\) and \(v^\top R(x)v=\|(L_0+L_\theta(x))^\top v\|^2\ge0\). These are algebraic constraints, not penalty targets. They give the continuous-time balance law for \(H_\theta\); they do not identify physical energy or imply its exact preservation by RK4.

\paragraph{Student objective and integration.}
CIPHER freezes the encoder and fits cached observed-future latents with \(\mathcal H_S=\mathcal H_T=\{1,2,4,8,16,32\}\). Its trajectory and chart objective is \(\mathcal L_z+0.2\mathcal L_x+10^{-4}\mathcal L_\phi\), using raw squared errors. In addition, the implemented state-dependent model penalizes squared residual outputs of \(K_\theta\) and \(L_\theta\) with coefficient \(10^{-3}\). This regularizes departure from the constant-matrix initialization; it does not enforce the matrix constraints, which already hold by construction.

Each update uses \(512\) anchors from eight cached trajectories, with \(20\le t<T-32\). Rollouts take one RK4 step per sampling interval. The largest supervised horizon is \(32\) steps, or \(1.6\) time units; evaluation at \(50\) steps extrapolates eighteen steps beyond it. The \(30\)-step score lies within the training horizon range, and AUC over \(1\)--\(50\) mixes horizons within and beyond that range. No teacher vector-field or teacher-rollout target enters the primary objective.

Stabilization consists of the identity-initialized trainable chart, zero-initialized matrix residuals, bounded coupling log-scales, gradient clipping, and the warm-up/cosine schedule. Student training uses no EMA, Jacobian or spectral penalty, or curriculum. At \(D=8\), CIPHER has \(142{,}577\) trainable parameters: \(72{,}752\) in the chart and \(69{,}825\) in the field. These counts exclude the frozen encoder, which remains necessary to initialize a forecast from observations.

\paragraph{CIPHER (constant \(J,R\)) architecture.}
This variant shares the chart, reflection policy, and Hamiltonian architecture described above. It uses canonical \(J\) for \(D=2\), and a learned \(J=(A-A^\top)/2\) for larger \(D\). Dissipation is parameterized as \(R=L^\top L\succeq0\), with \(L\) initialized at scale \(10^{-2}\). The state-dependent matrix networks and their residual penalty are absent. Encoded observed-future targets, RK4 step size, horizons, trajectory/chart loss weights, anchor sampling, and optimization settings are the same as for CIPHER; horizon losses are summed as in \eqref{eq:Ltraj}.

\paragraph{Training and seed policy.}
\label{app:training}
Table~\ref{tab:train_ours} gives the primary optimizer settings; CIPHER (constant \(J,R\)) uses the same student optimizer and update budget. The teacher is trained first, after which the encoder is frozen and its latents cached. The main comparisons use five randomly selected seeds per dataset, with one student per configuration for each teacher and the same seed determining initialization and training randomness. All five runs, including poor-performing ones, enter the summaries. The teacher and student use final checkpoints; Mamba-3 and PHAST instead use validation-selected checkpoints as described above. DCL is evaluated after its fixed training budget. These are method-specific training policies, not identical optimization budgets or checkpoint selection.

The dimension sweep uses three runs per cell. Pairing students with their parent teachers measures complete-pipeline variability, rather than student optimization variability conditional on one fixed teacher.

\begin{table}[t]
\centering
\small
\begin{tabular}{lll}
\toprule
Setting & Teacher & CIPHER student \\
\midrule
Optimizer & AdamW & AdamW \\
Learning rate & \(3\times10^{-4}\) & \(5\times10^{-4}\) \\
Weight decay & \(10^{-5}\) & \(10^{-5}\) \\
Training steps & \(30{,}000\) & \(10{,}000\) \\
Warm-up steps & \(2{,}000\) & \(500\) \\
Schedule after warm-up & Cosine & Cosine \\
Gradient clipping norm & \(5\) & \(5\) \\
Trajectories per update & \(8\) numeric; \(4\) video & \(8\) cached \\
Anchors per update & -- & \(512\) \\
Checkpoint & Final & Final \\
\bottomrule
\end{tabular}
\caption{Primary optimization settings. Student trajectory losses use \(\lambda_x=0.2\), chart weight \(10^{-4}\), and a separate matrix-residual weight \(10^{-3}\). External-baseline configurations and their validation selection are specified separately.}
\label{tab:train_ours}
\end{table}

\subsection{Evaluation metrics}
\label{app:exp_metrics}

\paragraph{Dynamical determinant of coefficient}
For the common anchor set \(\mathcal A\), let \(s_{a,h}\) and \(\hat s_{a,h}\) be the reference and decoded states at horizon \(h\), and let \(\bar s_{h,j}\) be the reference mean over anchors. We report the coordinate-averaged dynamical coefficient of determination,
\(R^2_{\mathrm{Dyn}}(h)=\frac{1}{d_s}\sum_{j=1}^{d_s}\left[1-\frac{\sum_{a\in\mathcal A}(\hat s_{a,h,j}-s_{a,h,j})^2}{\sum_{a\in\mathcal A}(s_{a,h,j}-\bar s_{h,j})^2}\right]\), which ranges from \((-\infty, 1]\).
Each evaluated coordinate must have nonzero reference variance. 

\paragraph{Frozen physical-state probe.}
All methods use the same probe architecture and fitting procedure, with separate weights for different representation spaces. The nonlinear probe has two width-\(256\) ReLU hidden layers and standardized inputs and outputs. It is fitted on \(250\) training trajectories selected at evenly spaced file indices, avoiding an energy-ordered prefix, using only times \(t\ge20\). A \(10\%\) validation subset controls early stopping and learning-rate reduction. Adam starts at \(10^{-3}\), with batches of \(4096\), at most \(300\) epochs, early stopping after \(25\) epochs without improvement, and learning-rate halving on a plateau. The best validation state is restored. Standardization is fitted on training data and held fixed for test evaluation. Affine linear probes use least squares on the same selected training data.

Forecasts from a teacher and its associated student share the \emph{same fitted teacher probe} \(P_z\). Specifically,
\(\hat s^T_{i,t,h}=P_z(\Phi_T^h(z_{i,t}))\) and
\(\hat s^S_{i,t,h}=P_z(\phi^{-1}(\Phi_{pH}^h(\phi(z_{i,t}))))\).
The inverse chart therefore precedes student decoding. The probe is frozen before evaluation and never supplies gradients to a dynamics model. This shared readout differs from the independently fitted probes on \(z_t\) and \(x_t\) used to assess instantaneous state recovery below.

\paragraph{Anchors and rollout scores.}
Index each trajectory by \(t=0,\ldots,T-1\), and let
\(\mathcal A=\{(i,t):i\in\mathcal I_{\mathrm{test}},\ 20\le t<T-50\}\)
be the common test-anchor set, of size \(N=|\mathcal A|\). This yields \(20{,}000\) anchors for \(T=150\), \(57{,}500\) for \(T=300\), and \(35{,}000\) for the three-mass chain. The same anchors are used at every horizon \(h=1,\ldots,50\). Models roll out autonomously after the anchor, without conditioning on future observations. For \(a=(i,t)\), write \(s_{a,h}=s_{i,t+h}\) and \(\hat s_{a,h}\) for the corresponding reference and decoded prediction. With \(\bar s_{h,j}=N^{-1}\sum_{a\in\mathcal A}s_{a,h,j}\), the score is \(R^2_{\mathrm{Dyn}}(h)\), as defined in Section~\ref{sec:experiments}.
This definition requires nonzero reference variance in each evaluated coordinate. Coordinates are scored separately and then averaged; it is not a ratio of sums pooled across coordinates. Scores may be negative and are retained in aggregation. Table~\ref{tab:results} reports \(R^2_{\mathrm{Dyn}}(30)\), \(R^2_{\mathrm{Dyn}}(50)\), and the horizon mean \(\mathrm{AUC}_{1:50}=50^{-1}\sum_{h=1}^{50}R^2_{\mathrm{Dyn}}(h)\), labeled \(@30\), \(@50\), and AUC, respectively. Here AUC denotes a discrete mean, not an unnormalized or trapezoidal integral. These are decoded state-forecast scores; they do not measure vector-field agreement or instantaneous state recovery. In particular, this benchmark metric differs from the aligned transition-agreement metric used in the original DCL evaluation \citep{laiz2025dyncl}.

\paragraph{Instantaneous state recovery.}
Identification is measured by decoding encoded test states at \(t\ge20\), without a dynamics rollout. For a representation \(v_t\), both an affine linear probe and a two-hidden-layer MLP probe \(P_v\) are fitted on training-split pairs \((v_t,s_t)\) and then frozen. The representations are the teacher encoding \(z_t\), the student encoding \(x_t=\phi(z_t)\), and the one-stage encoding. The probe architecture and fitting policy are matched across these comparisons. For one trained run, let \(\mathcal T\) index the evaluated test states and let \(\hat s_n=P_v(v_n)\) be the frozen probe prediction for the corresponding physical state \(s_n\). For each coordinate \(j\), define \(\bar s_j=|\mathcal T|^{-1}\sum_{n\in\mathcal T}s_{n,j}\) and
\(R_j^2=1-\frac{\sum_{n\in\mathcal T}(s_{n,j}-\hat s_{n,j})^2}{\sum_{n\in\mathcal T}(s_{n,j}-\bar s_j)^2}\).
The denominator must be positive for each evaluated coordinate. For a nonempty coordinate subset \(I\), its retention score is \(R^2_{\mathrm{id}}(I)=|I|^{-1}\sum_{j\in I}R_j^2\). With \(I_q\) and \(I_p\) denoting the position and momentum coordinate sets, the All, \(q\), and \(p\) entries use \(I=\{1,\ldots,d_s\}\), \(I_q\), and \(I_p\), respectively. Thus the displayed score averages per-coordinate coefficients of determination, rather than pooling residual sums across coordinates. Coordinate averages are computed separately for each run before taking the mean and sample standard deviation over five seeds. These evaluation probes do not update the encoder, chart, or dynamics model.

An independently fitted probe on \(x_t\) measures how readily the chosen probe class can decode student coordinates. Information preservation itself follows from invertibility: for any teacher probe \(P_z\), the transported readout \(P_z\circ\phi^{-1}\) satisfies \((P_z\circ\phi^{-1})(x_t)=P_z(z_t)\), up to numerical inversion error. Thus, different fitted-probe scores do not imply that the student chart creates additional state information. High held-out identification scores support approximate state recovery on the evaluated distribution; they do not certify a global diffeomorphism.
Table~\ref{tab:identification} reports these probe scores, with five runs for every method and task. Figure~\ref{fig:state_retention} visualizes the full-state and momentum scores in the main text. The linear and nonlinear probe groups are ranked separately for each task and target subset, using distinct means at the displayed precision. A displayed value of \(1.000\) is rounded, not an assertion of zero reconstruction error.

\textbf{State information retention}
\label{app:state_retention}
We assess current-state recovery from teacher encodings \(z_t\), student states \(x_t=\phi(z_t)\), and one-stage encodings using independently fitted linear and nonlinear probes. The identification score \(R^2_{\mathrm{id}}\) averages test-set coefficients of determination across state coordinates, computed as in \(R^2_{\mathrm{Dyn}}\) but without a rollout. We report full-state and momentum-only averages. Figure~\ref{fig:state_retention} summarizes the results; Appendix Table~\ref{tab:identification} gives all scores, with probe details in Appendix~\ref{app:exp_metrics}.

\begin{figure}[!htbp]
\centering
\includegraphics[width=\linewidth]{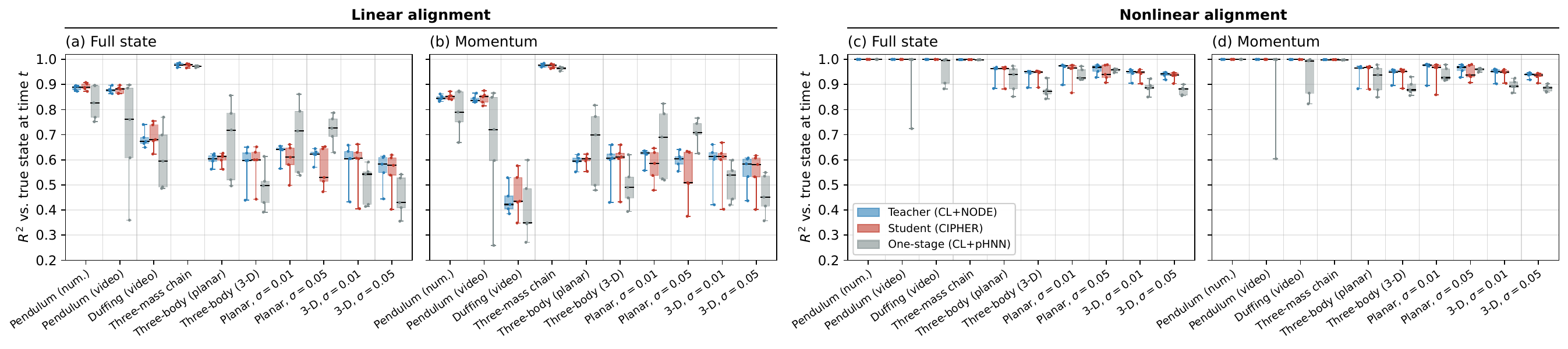}
\caption{\textbf{Physical-state recovery from learned representations.} Teacher, CIPHER, and one-stage models' identification \(R^2_{\mathrm{id}}\) of full state and momentum information after linear and nonlinear probes.  Both teacher and CIPHER retain position and momentum information in their state representations consistently, shown in (c, d), while one-stage method retain less. Position and momentum are generally nonlinearly related to physical states. Details are provided in Appendix Table~\ref{tab:identification}.}
\label{fig:state_retention}
\end{figure}

Teacher and student representations support accurate nonlinear recovery of both configuration and unobserved momentum. The gap between linear and nonlinear probes, particularly for Duffing momentum, indicates that poor linear alignment need not mean missing state information. One-stage recovery is weaker on the video oscillators and 3-D tasks, although the planar comparisons show no uniform ordering.

The invertible student chart preserves the information in the teacher encoding; differences between independently fitted probes reflect their function class and fitting. These results support approximate full-state recovery on the test distribution, consistent with the conditional theory, without establishing a global diffeomorphism.

\begingroup
\newcommand{\idscore}[2]{\ensuremath{#1{\scriptscriptstyle\,\pm\,#2}}}
\newcommand{\idbest}[2]{\ensuremath{\mathbf{#1}{\scriptscriptstyle\,\pm\,#2}}}
\newcommand{\idsecond}[2]{\ensuremath{\underline{#1}{\scriptscriptstyle\,\pm\,#2}}}
\begin{table*}[t]
\centering
\caption{\textbf{Physical-state recovery from learned representations.} Identification $R^2_{\mathrm{id}}$ from independently fitted linear and nonlinear probes, evaluated without rollout. Each task stacks full-state (All), configuration ($q$), and momentum ($p$) scores. Entries are mean $\pm$ sample standard deviation over five runs. Within each probe group and task/target row, bold and underline mark the highest and second-highest distinct reported means; ties share a rank. Values of $1.000$ and standard deviations of $0.000$ are rounded and do not imply exact recovery.}
\label{tab:identification}
\scriptsize
\setlength{\tabcolsep}{3pt}
\renewcommand{\arraystretch}{1.14}
\resizebox{\textwidth}{!}{%
\begin{tabular}{@{}llcccccc@{}}
\toprule
\multirow{2}{*}{\textbf{Task}} & \multirow{2}{*}{\shortstack{\textbf{Target}\\($R^2_{\mathrm{id}}$)}} & \multicolumn{3}{c}{\textbf{Linear probe}} & \multicolumn{3}{c}{\textbf{Nonlinear probe}} \\
\cmidrule(lr){3-5}\cmidrule(lr){6-8}
 & & \textbf{Teacher} & \shortstack{\textbf{One-stage}\\CL+pHNN} & \shortstack{\textbf{Student}\\(CIPHER)} & \textbf{Teacher} & \shortstack{\textbf{One-stage}\\CL+pHNN} & \shortstack{\textbf{Student}\\(CIPHER)} \\
\midrule
\multirow{3}{*}{\shortstack[l]{Pendulum\\(numeric)}} & All & \idsecond{0.886}{0.010} & \idscore{0.828}{0.068} & \idbest{0.891}{0.014} & \idbest{1.000}{0.000} & \idbest{1.000}{0.000} & \idbest{1.000}{0.000} \\
 & $q$ & \idsecond{0.926}{0.015} & \idscore{0.865}{0.057} & \idbest{0.929}{0.019} & \idbest{1.000}{0.000} & \idbest{1.000}{0.000} & \idbest{1.000}{0.000} \\
 & $p$ & \idsecond{0.846}{0.011} & \idscore{0.791}{0.086} & \idbest{0.852}{0.013} & \idbest{1.000}{0.000} & \idbest{1.000}{0.000} & \idbest{1.000}{0.000} \\
\midrule
\multirow{3}{*}{\shortstack[l]{Pendulum\\(video)}} & All & \idsecond{0.878}{0.012} & \idscore{0.702}{0.224} & \idbest{0.879}{0.014} & \idbest{1.000}{0.000} & \idsecond{0.945}{0.123} & \idbest{1.000}{0.000} \\
 & $q$ & \idbest{0.914}{0.014} & \idscore{0.746}{0.204} & \idsecond{0.912}{0.010} & \idbest{1.000}{0.000} & \idsecond{0.969}{0.070} & \idbest{1.000}{0.000} \\
 & $p$ & \idsecond{0.842}{0.014} & \idscore{0.658}{0.248} & \idbest{0.846}{0.024} & \idbest{1.000}{0.000} & \idsecond{0.921}{0.177} & \idbest{1.000}{0.000} \\
\midrule
\multirow{3}{*}{\shortstack[l]{Duffing\\(video)}} & All & \idsecond{0.683}{0.035} & \idscore{0.608}{0.125} & \idbest{0.694}{0.053} & \idbest{1.000}{0.000} & \idsecond{0.956}{0.058} & \idbest{1.000}{0.000} \\
 & $q$ & \idbest{0.927}{0.015} & \idscore{0.806}{0.132} & \idsecond{0.925}{0.019} & \idbest{1.000}{0.000} & \idsecond{0.977}{0.030} & \idbest{1.000}{0.000} \\
 & $p$ & \idsecond{0.439}{0.056} & \idscore{0.410}{0.131} & \idbest{0.464}{0.090} & \idbest{1.000}{0.000} & \idsecond{0.936}{0.085} & \idbest{1.000}{0.000} \\
\midrule
\multirow{3}{*}{\shortstack[l]{Three-mass\\chain}} & All & \idbest{0.978}{0.008} & \idscore{0.972}{0.003} & \idsecond{0.977}{0.007} & \idbest{0.999}{0.000} & \idsecond{0.998}{0.000} & \idbest{0.999}{0.000} \\
 & $q$ & \idbest{0.981}{0.010} & \idsecond{0.980}{0.003} & \idsecond{0.980}{0.007} & \idbest{0.999}{0.000} & \idbest{0.999}{0.000} & \idbest{0.999}{0.000} \\
 & $p$ & \idbest{0.976}{0.006} & \idscore{0.963}{0.006} & \idsecond{0.973}{0.008} & \idbest{0.999}{0.000} & \idsecond{0.998}{0.000} & \idbest{0.999}{0.000} \\
\midrule
\multirow{3}{*}{\shortstack[l]{Three-body\\(planar)}} & All & \idscore{0.600}{0.024} & \idbest{0.675}{0.160} & \idsecond{0.603}{0.025} & \idsecond{0.948}{0.036} & \idscore{0.922}{0.053} & \idbest{0.949}{0.037} \\
 & $q$ & \idscore{0.608}{0.022} & \idbest{0.697}{0.165} & \idsecond{0.610}{0.025} & \idbest{0.946}{0.034} & \idsecond{0.922}{0.050} & \idbest{0.946}{0.035} \\
 & $p$ & \idscore{0.593}{0.026} & \idbest{0.653}{0.156} & \idsecond{0.597}{0.026} & \idbest{0.951}{0.038} & \idsecond{0.922}{0.056} & \idbest{0.951}{0.039} \\
\midrule
\multirow{3}{*}{\shortstack[l]{Three-body\\(3-D)}} & All & \idsecond{0.582}{0.083} & \idscore{0.489}{0.085} & \idbest{0.585}{0.082} & \idsecond{0.937}{0.028} & \idscore{0.876}{0.031} & \idbest{0.938}{0.028} \\
 & $q$ & \idsecond{0.580}{0.078} & \idscore{0.482}{0.086} & \idbest{0.583}{0.076} & \idsecond{0.932}{0.030} & \idscore{0.866}{0.034} & \idbest{0.936}{0.025} \\
 & $p$ & \idsecond{0.584}{0.089} & \idscore{0.497}{0.086} & \idbest{0.587}{0.089} & \idbest{0.942}{0.026} & \idscore{0.887}{0.029} & \idsecond{0.941}{0.032} \\
\midrule
\multirow{3}{*}{\shortstack[l]{Three-body\\(planar)\\$\sigma=0.01$}} & All & \idsecond{0.630}{0.036} & \idbest{0.691}{0.144} & \idscore{0.600}{0.065} & \idbest{0.960}{0.034} & \idscore{0.939}{0.025} & \idsecond{0.950}{0.047} \\
 & $q$ & \idsecond{0.644}{0.040} & \idbest{0.715}{0.147} & \idscore{0.624}{0.063} & \idbest{0.959}{0.032} & \idscore{0.938}{0.021} & \idsecond{0.949}{0.043} \\
 & $p$ & \idsecond{0.615}{0.033} & \idbest{0.668}{0.142} & \idscore{0.575}{0.068} & \idbest{0.961}{0.037} & \idscore{0.941}{0.028} & \idsecond{0.950}{0.051} \\
\midrule
\multirow{3}{*}{\shortstack[l]{Three-body\\(planar)\\$\sigma=0.05$}} & All & \idsecond{0.618}{0.028} & \idbest{0.720}{0.062} & \idscore{0.563}{0.081} & \idbest{0.961}{0.021} & \idsecond{0.956}{0.006} & \idscore{0.946}{0.031} \\
 & $q$ & \idsecond{0.637}{0.040} & \idbest{0.730}{0.072} & \idscore{0.595}{0.067} & \idbest{0.960}{0.021} & \idsecond{0.954}{0.007} & \idscore{0.946}{0.032} \\
 & $p$ & \idsecond{0.599}{0.032} & \idbest{0.709}{0.054} & \idscore{0.530}{0.107} & \idbest{0.961}{0.021} & \idsecond{0.959}{0.007} & \idscore{0.946}{0.031} \\
\midrule
\multirow{3}{*}{\shortstack[l]{Three-body\\(3-D)\\$\sigma=0.01$}} & All & \idbest{0.587}{0.089} & \idscore{0.505}{0.081} & \idsecond{0.582}{0.101} & \idbest{0.943}{0.022} & \idscore{0.888}{0.026} & \idsecond{0.942}{0.022} \\
 & $q$ & \idbest{0.588}{0.084} & \idscore{0.497}{0.086} & \idsecond{0.581}{0.099} & \idbest{0.941}{0.020} & \idscore{0.879}{0.031} & \idsecond{0.940}{0.023} \\
 & $p$ & \idbest{0.585}{0.094} & \idscore{0.512}{0.076} & \idsecond{0.582}{0.103} & \idbest{0.945}{0.024} & \idscore{0.897}{0.022} & \idsecond{0.943}{0.022} \\
\midrule
\multirow{3}{*}{\shortstack[l]{Three-body\\(3-D)\\$\sigma=0.05$}} & All & \idbest{0.560}{0.069} & \idscore{0.453}{0.080} & \idsecond{0.549}{0.088} & \idbest{0.938}{0.012} & \idscore{0.876}{0.019} & \idsecond{0.934}{0.018} \\
 & $q$ & \idbest{0.567}{0.068} & \idscore{0.444}{0.079} & \idsecond{0.551}{0.088} & \idbest{0.939}{0.013} & \idscore{0.868}{0.024} & \idsecond{0.935}{0.019} \\
 & $p$ & \idbest{0.553}{0.071} & \idscore{0.462}{0.081} & \idsecond{0.548}{0.088} & \idbest{0.937}{0.011} & \idscore{0.885}{0.014} & \idsecond{0.932}{0.016} \\
\bottomrule
\end{tabular}%
}
\end{table*}
\endgroup

\paragraph{Aggregation and table display.}
Each metric is first computed over the complete test set for one trained run. For \(n\ge2\) available runs with scores \(u_1,\ldots,u_n\), tables report \(\bar u=n^{-1}\sum_{r=1}^n u_r\) and sample standard deviation \(\big[(n-1)^{-1}\sum_{r=1}^n(u_r-\bar u)^2\big]^{1/2}\). Tables~\ref{tab:identification} and~\ref{tab:results} use \(n=5\) for every entry. Any single-run result in additional analyses is labeled and has no across-run standard deviation. Overlapping anchors are not treated as independent training repetitions. Negative scores are retained in the numerical aggregation; Tables~\ref{tab:results} and~\ref{tab:phast} abbreviate negative means as \(<0\), omitting their magnitude and standard deviation. A negative finite score means that the coordinate-averaged normalized squared error exceeds that of the mean-state reference; it does not by itself imply numerical divergence. This abbreviated display therefore does not characterize failure severity. Bold and underline identify the highest and second-highest distinct means within each task and metric, ranked at the reported three-decimal precision; tied means share a rank. This highlighting describes the reported means and is not a statistical-significance test.

\subsection{Additional representation views}
\label{app:representation_geometry}

Figure~\ref{fig:representation_clouds_extended} supplements the oscillator overview in Figure~\ref{fig:representation_clouds}. Clouds use \(30\) test trajectories, with \(100\)-step rollouts for the video oscillators and \(200\)-step rollouts for the 3-D three-body system. In aligned views, teacher predictions and inverse-charted student predictions pass through the fitted teacher probe \(P_z\). The reference column also passes through that probe: it displays \(P_z(E_\psi(o_{t-w:t}))\) along observed sequences, rather than raw simulator states. Thus these panels compare model trajectories with a decoded observation reference in a common readout space; accurate physical-state recovery is assessed separately by the unsmoothed test scores against simulator states.

Raw views retain each model's coordinates. The twelve-dimensional latent view is reduced to three principal components fitted on encoded test latents for visualization. Aligned three-body panels show body positions, not the full position--momentum state. A centered five-frame moving average is applied to all cloud columns for display. This smoothing is not part of model inference or the forecasting and identification metrics. Neither a projected cloud nor a flexible nonlinear alignment certifies an injective or differentiable correspondence.

\begin{figure}[t]
    \centering
    {\small\bfseries (i) Raw coordinate views\par}
    \includegraphics[width=\linewidth]{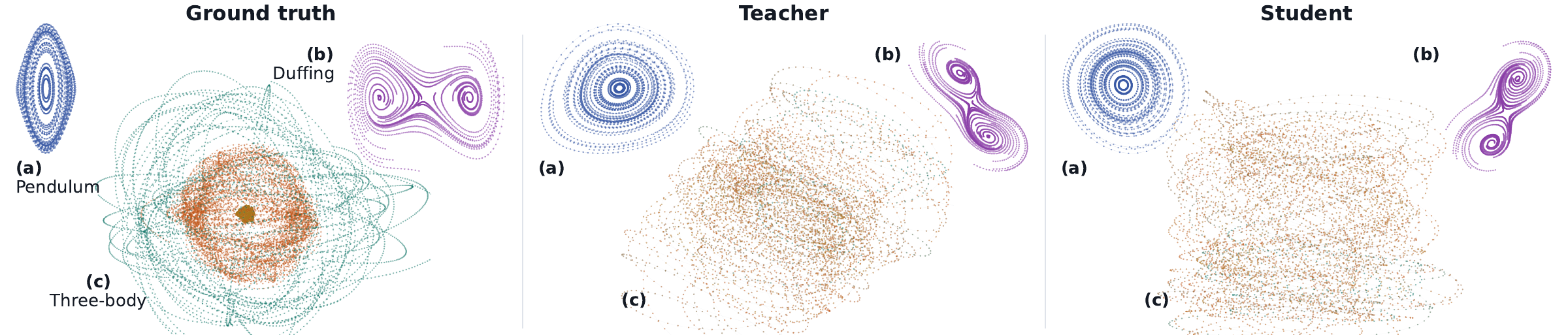}\par
    \vspace{2pt}
    {\small\bfseries (ii) Aligned coordinate views\par}
    \includegraphics[width=\linewidth]{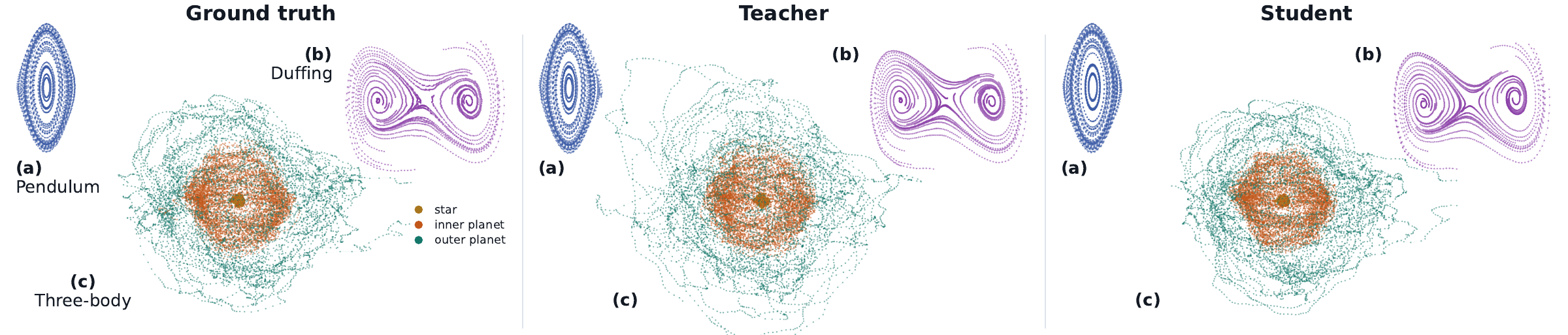}
    \caption{\textbf{Additional raw and aligned representation clouds.}
    Columns show the observation reference, teacher, and student. In the aligned panels, the column labeled ``Ground truth is the probe-decoded observation sequence, not raw physical states. Within each column, (a), (b), and (c) denote pendulum, Duffing, and three-body dynamics. Row (i) shows raw coordinate views; row (ii) shows views aligned to physical coordinates. In the three-body panels, the aligned positions distinguish the star and the two planetary orbital scales, but do not display the full phase space. Latent projections and aligned position clouds alone establish neither full-state recovery nor a diffeomorphism.}
    \label{fig:representation_clouds_extended}
\end{figure}

\paragraph{Decoded field diagnostics.}
Oscillator field panels use a \(22\times22\) grid over a percentile box of test physical states. At each grid state, backward simulation supplies a \(21\)-frame history ending there, which is encoded by the target encoder. This simulator access is for visualization only. Models advance by \(\pm\Delta\), using four RK4 substeps per interval; their endpoints are decoded through the probe, after inverse charting for CIPHER. Centered differences of these endpoints estimate local evolution in readout coordinates. Unlike benchmark rollouts, this diagnostic uses finer integration. The reference uses the same differencing stencil, which does not imply identical truncation errors. Cells with fewer than three test states are masked; stream colors share a speed scale set by the reference's \(98\)th percentile within the mask. Raw latent speeds are coordinate-dependent.

The span diagnostic repeats differencing at intervals \(k\Delta\), \(k\in\{1,2,4,8\}\). Endpoint errors bounded by \(\epsilon\) contribute at most \(\epsilon/(k\Delta)\) to the centered velocity estimate. Longer spans can reduce this contribution while changing truncation and dynamical errors; the diagnostic therefore checks sensitivity rather than exactly separating readout and field error.

\newpage
\section{Future Directions}
\label{app:limitations_future}

We foresee two directions for future work to expand the scope and application of CIPHER. While CIPHER learns representation and models both conservative and dissipative system through port-Hamiltonian structure, other properties are not currently used. Future work will investigate utilizing ports, inputs, and essential pH properties to learn dynamical representations for modeling and controlling of systems. 

In addition, we provide theoretical and empirical results for CIPHER's robustness to latent dimension. Hence, development of data-driven dimension selection is interesting to investigate further.

\end{document}